\documentclass[hidelinks,onefignum,onetabnum]{siamart251216}

\usepackage{lipsum}
\usepackage{amsfonts}
\usepackage{graphicx}
\usepackage{epstopdf}
\usepackage{algorithmic}
\ifpdf
  \DeclareGraphicsExtensions{.eps,.pdf,.png,.jpg}
\else
  \DeclareGraphicsExtensions{.eps}
\fi

\usepackage{dsfont}
\usepackage{enumitem}
\usepackage{booktabs,tabularx,ragged2e}
\newcolumntype{Y}{>{\RaggedRight\arraybackslash}X}
\usepackage{xfrac}
\usepackage{subcaption}
\usepackage{tikz}
\usepackage{multirow}
\usepackage{comment}
\usepackage{placeins}
\usepackage{microtype}

\newcommand{\bbR}{\mathbb{R}}
\def\eps{\varepsilon}
\newcommand{\one}{\mathds{1}}
\def\l{\left(}
\def\r{\right)}

\newcommand{\cE}{\mathcal{E}}
\def\TLp#1{\mathrm{TL}^{#1}}
\def\Wkp#1#2{\mathrm{W}^{#1,#2}}
\newcommand{\cP}{\mathcal{P}}
\newcommand{\cZ}{\mathcal{Z}}
\def\Lp#1{\mathrm{L}^{#1}}
\def\dd{\mathrm{d}}
\def\lb{\left\{}
\def\rb{\right\}}
\def\ls{\left[}
\def\rs{\right]}
\def\spaceBar{\, | \,}
\def\dTLp#1{d_{\TLp{#1}}}
\newcommand{\iid}{\stackrel{\mathrm{iid}}{\sim}}
\newcommand{\bbP}{\mathbb{P}}
\newcommand{\Id}{\mathrm{Id}}
\newcommand{\bbN}{\mathbb{N}}
\newcommand{\etaP}{\eta_{\mathrm{p}}}
\newcommand{\etaPTilde}{\tilde{\eta}_{\mathrm{p}}}
\newcommand{\Div}{\mathrm{div}}
\newcommand{\cF}{\mathcal{F}}
\newcommand{\cG}{\mathcal{G}}
\newcommand{\cS}{\mathcal{S}}
\newcommand{\cH}{\mathcal{H}}
\newcommand{\cI}{\mathcal{I}}
\newcommand{\cJ}{\mathcal{J}}
\newcommand{\cK}{\mathcal{K}}

\newcommand{\bbZ}{\mathbb{Z}}
\def\Ck#1{\mathrm{C}^{#1}}
\newcommand{\bbQ}{\mathbb{Q}}
\newcommand{\bbE}{\mathbb{E}}
\newcommand{\cN}{\mathcal{N}}
\newcommand{\E}{{\mathbb E}}
\newcommand{\dist}{\mathrm{dist}}
\def\Ckc#1{\Ck{#1}_{\mathrm{c}}}
\newcommand{\supp}{\mathrm{supp}}

\newsiamremark{remark}{Remark}
\newsiamremark{hypothesis}{Hypothesis}
\crefname{hypothesis}{Hypothesis}{Hypotheses}
\newsiamthm{claim}{Claim}
\newsiamremark{fact}{Fact}
\crefname{fact}{Fact}{Facts}

\newsiamthm{mydef}{Definition}
\newsiamthm{assumptions}{Assumptions}

\definecolor{darkgreen}{rgb}{0.1,0.6,0.1}

\headers{Analysis of SSL on Hypergraphs}{A. Weihs, A. L. Bertozzi, and M. Thorpe}

\title{Analysis of Semi-Supervised Learning on Hypergraphs\thanks{
\textbf{Funding}: AW and AB were supported in part by NSF grant DMS-2152717. 
MT acknowledges the support of the EPSRC Mathematical and Foundations of Artificial Intelligence Probabilistic AI Hub (grant agreement EP/Y007174/1), the Leverhulme Trust through the Project Award ``Robust Learning: Uncertainty Quantification, Sensitivity and Stability'' (grant agreement RPG-2024-051) and the NHSBT award 177PATH25 ``Harnessing Computational Genomics to Optimise Blood Transfusion Safety and Efficacy''.}}

\author{Adrien Weihs \thanks{Department of Mathematics, University of California Los Angeles, Los Angeles, CA 90095, USA
  (\email{weihs@math.ucla.edu}, \email{bertozzi@math.ucla.edu}).}
\and Andrea L. Bertozzi \footnotemark[2] 
\and Matthew Thorpe \thanks{Department of Statistics, University of Warwick, Coventry, CV4 7AL, UK
  (\email{Matthew.Thorpe@warwick.ac.uk}).}}

\usepackage{amsopn}

\ifpdf
\hypersetup{
  pdftitle={An Example Article},
  pdfauthor={D. Doe, P. T. Frank, and J. E. Smith}
}
\fi

\allowdisplaybreaks[4]

\begin{document}

\maketitle

\begin{abstract}
Hypergraphs provide a natural framework for modeling multiway interactions. We analyze a class of variational semi-supervised learning problems posed on random geometric hypergraphs and establish asymptotic consistency in the large-data limit. In particular, we identify scaling regimes that ensure well-posedness—yielding nontrivial label propagation rather than collapse to a constant labeling—and show that discrete minimizers converge, in the continuum, to solutions of a density-weighted p-Laplacian equation. We also propose Higher-Order Hypergraph Learning (HOHL), a multiscale regularization scheme based on powers of Laplacians associated with hypergraph-induced subgraphs. For geometric point clouds, we analyze an efficient multiscale Laplacian surrogate for HOHL and prove convergence to a higher-order Sobolev-type seminorm. 
Numerical experiments on standard benchmarks support the practical utility of the resulting higher-order regularization.
\end{abstract}



\section{Introduction}

This paper establishes continuum limits and well-posedness characterizations for variational semi-supervised learning on hypergraphs. Specifically, in the large-data regime we study the following model: given a point cloud with labels prescribed on a subset, we construct a data-driven hypergraph and recover the remaining labels as minimizers of an energy subject to interpolation constraints. In this formulation, the ambient geometry enters through the hypergraph structure, extending classical graph-based methods by allowing multiway couplings via hyperedges.

A central question is well-posedness: for which parameter configurations does the method yield nontrivial label propagation, rather than degenerating (or “collapsing”) to a constant labeling (see \cite{10.5555/2984093.2984243,elalaoui16} for illustrations of this phenomenon in graph learning)?

We address this through discrete-to-continuum limits as the number of data points tends to infinity. By identifying the continuum energies and Euler–Lagrange operators approximated by the discrete objectives, we characterize the parameter scalings---both in hypergraph construction and regularization strength---that prevent collapse, and we obtain principled guidance for hyperparameter selection. More broadly, the continuum viewpoint provides a common analytical language for comparing graph and hypergraph methods that may appear unrelated at finite sample size, revealing both shared behavior and meaningful differences in limiting regularization. This perspective motivates the taxonomy in Figure~\ref{fig:classificationAlgorithms}, which organizes several methods by their limiting Sobolev-type behavior. At the same time, belonging to the same limiting variational class should not be interpreted as strict finite-sample equivalence: as our numerical experiments show, discrete methods with related continuum limits may still differ substantially in their performance.



Our first contribution is a discrete-to-continuum study of Dirichlet-type
hypergraph energies \cite{scholkopfHyper2006}, obtained by aggregating pairwise
finite differences over hyperedges. We work with random geometric hypergraphs
endowed with product-type hyperedge weights that implement a soft
\(\varepsilon\)-clique rule: hyperedges are strongly weighted only when their
vertices are mutually \(\varepsilon\)-local, in contrast to
neighborhood-centered constructions. A technically substantial part of the analysis is the pointwise consistency of
the associated Euler--Lagrange operators. Unlike the graph case, the operator
contains products of many pairwise kernels and sums over several interacting
indices. Its continuum limit therefore requires controlling a genuinely
multi-index statistic, expanding the nonlinear \(p\)-Laplacian interaction, and
identifying the nontrivial constants produced by the geometry of the
product-type hyperedge kernel. This yields quantitative pointwise convergence
rates and a nondivergence-form representation of the limiting weighted
\(p\)-Laplacian operator induced by the hypergraph weights. We also establish \(\Gamma\)-convergence of the energies, obtaining a
sampling-density--weighted first-order Sobolev-type continuum limit and
delineating well-posed and ill-posed regimes in the semi-supervised setting as
functions of \(\varepsilon\), the regularization strength, and the ambient
dimension.


Our second contribution is Higher-Order Hypergraph Learning (HOHL), a multiscale regularization framework defined on general hypergraphs that leverages hyperedge structure beyond density reweighting of pairwise couplings. The abstract HOHL model penalizes powers of Laplacians on hypergraph-induced graphs. In the geometric point-cloud setting, where explicit hyperedge enumeration is computationally expensive, we study an efficient multiscale Laplacian surrogate that preserves the intended hierarchy of scales and regularity orders.
We then prove $\Gamma$-convergence to a higher-order Sobolev-type limit and characterize well-/ill-posedness thresholds as functions of the graph-construction parameters, the regularization order, and the dimension. Unlike the Dirichlet-type hypergraph energies, HOHL therefore allows control of how strongly the solution is smoothed locally across different regularity orders.


\begin{figure}
  \centering
  \includegraphics[width=\textwidth]{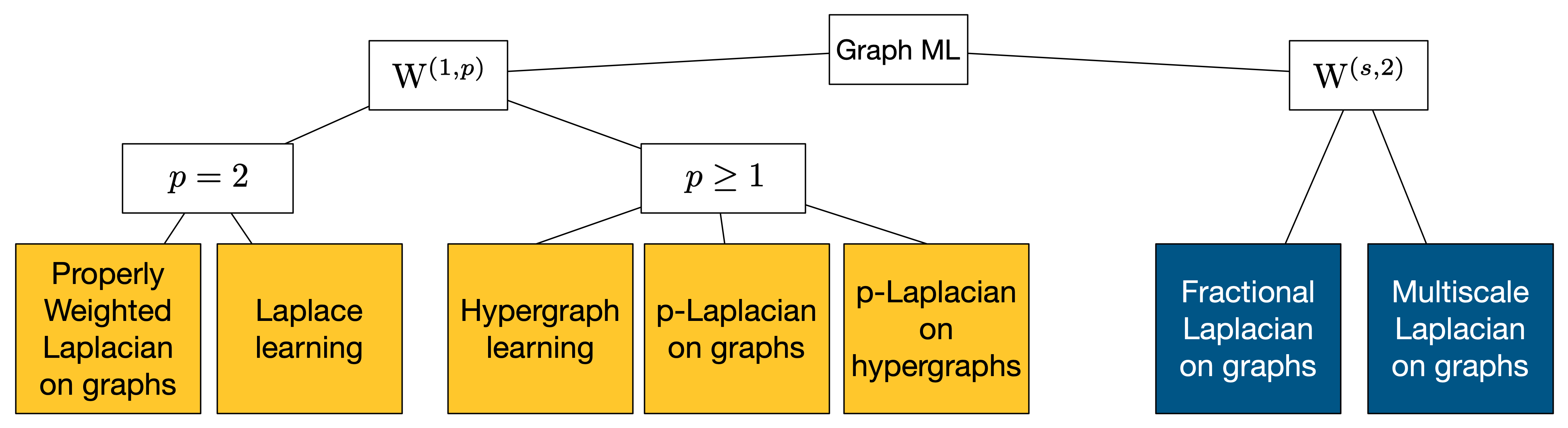}
  \caption{Classification of several algorithms based on their continuum limit. Edges indicate convergence to a Sobolev-type seminorm in the continuum limit, with each algorithm linked to its associated $\mathrm{W}^{(k,p)}$ space. Methods from left-to-right: \cite{calderSlepcev}, \cite{LapRef}, \cite{scholkopfHyper2006} (shown in this work), \cite{elalaoui16}, \cite{shi2025hypergraphplaplacianequationsdata}, \cite{Stuart}, \cite{Merkurjev} (shown in this work as the surrogate of HOHL) } \label{fig:classificationAlgorithms}
\end{figure}

The remainder of the paper is organized as follows.
Section~\ref{sec:background} reviews related work and the analytical framework used throughout, in particular $\Gamma$-convergence in $\TLp{p}$ spaces.
Section~\ref{sec:main} introduces the geometric hypergraph model, defines the various energies considered and states our main theorems.
Section~\ref{sec:proofs} provides detailed proofs of the pointwise and variational consistency results.
Section~\ref{sec:discussion} reports numerical experiments illustrating the practical implications of the theory.
Section~\ref{sec:conclusion} concludes and discusses directions for future research.

\section{Background} \label{sec:background}

In this section we review related works, introduce the $\TLp{p}$ topology, and briefly recall the notion of $\Gamma$-convergence. 
These tools underpin our variational convergence analysis, which concerns the convergence of minimizers of the discrete energies considered in this paper. 
This viewpoint is particularly natural in semi-supervised learning, where predictions are typically extracted from the minimizer (often real-valued) by a thresholding step. 

\subsection{Related works}

\paragraph{Discrete-to-continuum analysis for graph learning} Laplacian regularization on graphs~\cite{LapRef} is a foundational tool for label propagation and semi-supervised learning. 
A well-known limitation is degeneracy in low-label regimes when the number of unlabeled points is large~\cite{10.5555/2984093.2984243,elalaoui16}. 
Discrete-to-continuum analyses such as~\cite{Slepcev} provide a theoretical explanation and delineate well-posed versus degenerate regimes, yielding principled guidance for selecting graph-construction and regularization parameters. 

More broadly, asymptotic consistency analysis is a standard tool for understanding graph-based objectives in the large-data regime: one compares discrete energies $\cE_n$ on functions $v_n:\Omega_n\to\bbR$ over a point cloud $\Omega_n\subset\Omega\subset\mathbb{R}^d$ with a continuum energy $\cE_\infty$ acting on $v:\Omega\to\bbR$.
This program can be pursued in complementary ways:
\begin{itemize}
    \item \emph{Spectral convergence} studies convergence of eigenpairs of the discrete operators associated with $\cE_n$ to those of the corresponding continuum operator~\cite{NIPS2006_5848ad95,Trillos,CALDER2022123,10.1214/009053607000000640,JMLR:v12:pelletier11a,10.1093/imaiai/iaw016}.
    \item \emph{Pointwise convergence} establishes $\cE_n(v|_{\Omega_n})\to \cE_\infty(v)$ (equivalently, convergence of the associated Euler--Lagrange operators) as $n\to\infty$ for sufficiently smooth $v$~\cite{NIPS2006_5848ad95,COIFMAN20065,Gine,10.1007/11503415_32,10.1007/11776420_7,Singer,10.5555/3104322.3104459,weihs2023discreteToContinuum}.
    \item \emph{Variational convergence} (typically via $\Gamma$-convergence) concerns convergence of minimizers of $\cE_n$ to minimizers of $\cE_\infty$~\cite{calderGameTheoretic,cristoferi_thorpe_2020,Stuart,Trillos3,GARCIATRILLOS2018239,thorpe_theil_2019,Gennip}.
\end{itemize}
In the semi-supervised setting, continuum-limit results include \(p\)-Laplacian learning \cite{Slepcev}, fractional Laplacian regularization~\cite{weihs2023consistency}, Lipschitz- or \(\infty\)-Laplacian-type methods~\cite{pmlr-v40-Kyng15,doi:10.1137/18M1199241,Bungert,doi.org/10.48550/arxiv.2111.12370}, Poisson learning~\cite{98b487bb64994720ba648f45328e2135,bungert2024convergenceratespoissonlearning}, Ginzburg--Landau regularization~\cite{Gennip}, and reweighting strategies~\cite{shi2017weighted,shi2018generalization,shi2018error,calderSlepcev}.
Beyond semi-supervised learning, analogous limits have been established for other graph energies, including graph total variation~\cite{Trillos3}, graph cuts and Cheeger-type problems~\cite{JMLR:v17:14-490,trillos2017estimating,thorpeCheeger,doi:10.1137/16M1098309}, graph analogues of Mumford--Shah~\cite{Caroccia_2020}, and objectives arising in empirical risk minimization~\cite{garcia_trillos_murray_2017}.

\paragraph{Asymptotic consistency analysis for hypergraph regularization}

Hypergraphs have been advocated as a way to encode multiway relations beyond pairwise edges~\cite{zanette,neuhauser}, motivating a broad hypergraph learning literature~\cite{scholkopfHyper2006,fazeny,pmlr-v80-li18e,hgLearningPractice,hgPLaplacianGeometric}. 
Within variational regularization, one can distinguish two main strains. 
Let $G=(V,E)$ be a hypergraph, where $V$ is a vertex set and $E$ is a family of subsets $e\subset V$ (the hyperedges). 
Given a labeling function $v:V\to\bbR$ and hyperedge weight function $w:E \mapsto \bbR$, these strains can be represented schematically as follows:
\begin{enumerate}[label=(\roman*), leftmargin=*, itemsep=0.4em]

\item \emph{Pairwise-aggregation energies.}
Each hyperedge contributes through an \emph{aggregation of pairwise finite differences} over all vertex pairs it contains, thereby inducing weighted pairwise couplings~\cite{scholkopfHyper2006}:
\begin{equation}\label{eq:intro:clustering}
\sum_{e\in E}\;\sum_{\{x_i,x_j\}\subseteq e} \frac{w(e)}{\vert e \vert} \,\bigl(v(x_i)-v(x_j)\bigr)^2.
\end{equation}

\item \emph{Hyperedge-level couplings.}
Rather than aggregating over pairs, these methods assign a \emph{single} interaction term to each hyperedge, modifying the within-hyperedge coupling mechanism; a prototypical example uses a range/TV-type penalty motivated by Lov\'asz extensions of cut objectives and hypergraph total variation~\cite{TVHg,shi2025hypergraphplaplacianequationsdata}:
\begin{equation} \label{eq:TVHG}
    \sum_{e\in E} w(e)\,\max_{x_i,x_j\in e}\bigl|v(x_i)-v(x_j)\bigr|^{p}.
\end{equation}

\end{enumerate}

Much of the existing theory for hypergraph regularization is discrete (combinatorial or probabilistic), often comparing hypergraph constructions to graph analogues (e.g.,~\cite{hypergraphGraph,jostMulas,chitra,jost,mulas}). 
Compared with graphs, less is known about large-data limits and recovery-versus-collapse regimes for hypergraph regularizers.
Recent progress in this direction includes continuum-limit results for TV-type formulations in geometric settings~\cite{shi2025hypergraphplaplacianequationsdata}, which yield first-order (Sobolev-type) limiting regularization.

Our work complements these developments in two directions. First, we establish discrete-to-continuum limits and well/ill-posedness thresholds for the hypergraph energies of~\cite{scholkopfHyper2006} on random geometric hypergraphs (see Remark \ref{rem:weights} for a detailed comparison with the neighborhood hyperedge model used in \cite{shi2025hypergraphplaplacianequationsdata}). 
Second, we introduce and analyze in the large-data limit HOHL, an alternative operator-based hypergraph regularization framework. Instead of assigning each hyperedge a single Lovász/TV-type interaction term, HOHL builds graphs from hyperedges of different sizes and regularizes through powers of the
corresponding Laplacians. Thus the coupling induced by HOHL is multiscale and
operator-based; it also leads to higher-order Sobolev-type limits.

Finally, we note that asymptotic analysis of hypergraphs has also been studied in stochastic block model settings, where the hypergraph itself is random and sampled conditional on labels~\cite{NIPS2014_ca5fcdda,2da0e056-e747-371f-ab74-a1641e59236a}. 
This is fundamentally different from our setting, which studies regularization of label functions on a fixed (data-driven) hypergraph.

\paragraph{HOHL and connections to multiscale regularization}
Our HOHL model builds on higher-order Laplacian-based regularization \cite{Stuart,weihs2023consistency,pmlr-v15-zhou11b} and multiscale constructions~\cite{Merkurjev} on graphs.
This perspective is related to multi-hop and neighborhood-mixing operators in the graph neural network literature~\cite{NEURIPS2019_23c89427,AbuElHaija2019MixHop,monti}, but in contrast to (hyper)graph neural networks~\cite{besta2024demystifyinghigherordergraphneural} our approach is architecture-free and comes with continuum-limit guarantees.

\paragraph{Technical background}
Our proofs rely on nonlocal approximation results for \(\Wkp{1}{p}\) originating in~\cite{Bourgain01anotherlook} and developed via \(\Gamma\)-convergence in~\cite{ponce2004}, with discrete-to-continuum extensions in~\cite{Trillos3,Slepcev}. 
Related results for \(\Wkp{s}{2}\) approximations can be found in~\cite{Stuart,weihs2023consistency,trillos2022rates}, and quantitative rates in graph-based problems have been obtained in~\cite{calder2020rates,weihs2023discreteToContinuum,ElBouchairi}.

\subsection{The \texorpdfstring{$\TLp{p}$}{TLp} Space}

Let $\cP(\Omega)$ be the set of probability measures on $\Omega$ and $\cP_p(\Omega)$ be the set of probability measures on $\Omega$ with finite $p$th-moment.
We denote by $\Lp{p}(\mu)$ the set of functions $u$ that are measurable with respect to $\mu$ and such that $\int_{\Omega} \vert u(x)\vert^p \, \dd \mu(x) < +\infty$  .
The pushforward of a measure $\mu\in\cP(\Omega)$ by a map $T:\Omega\to \cZ$ is the measure $\nu\in\cP(\cZ)$ defined by
\[ \nu(A) = T_{\#} \mu(A) := \mu(T^{-1}(A)) = \mu\l \lb x\spaceBar T(x)\in A \rb\r \qquad \text{for all measurable sets } A. \]
For $\mu,\nu\in\cP_p(\Omega)$ we denote by $\Pi(\mu,\nu)$ the set of all probability measures on $\Omega\times\Omega$ such that the first marginal is $\mu$ and the second marginal is $\nu$, i.e. $(P_X)_{\#}\pi=\mu$ and $(P_Y)_{\#}\pi=\nu$ where $P_X: \Omega\times\Omega \ni (x,y) \mapsto x \in \Omega$ and $P_Y: \Omega\times\Omega \ni (x,y) \mapsto y \in \Omega$.
The following definition of the $\TLp{p}$ space and metric can be found in \cite{Trillos3}.

\begin{mydef}
For an underlying domain $\Omega$, define the set
\[ \TLp{p} = \lb (\mu,u) \spaceBar \mu \in \cP_p(\Omega), u \in \Lp{p}(\mu) \rb. \]
For $(\mu,u),(\nu,v) \in \TLp{p}$, we define the $\TLp{p}$ distance $\dTLp{p}$ as follows:
\[ \dTLp{p}((\mu,u),(\nu,v)) = \inf_{\pi \in \Pi(\mu,\nu)} \l \int_{\Omega \times \Omega} \vert x-y \vert^p + \vert u(x) - v(y) \vert^p \,\dd\pi(x,y) \r^{\frac{1}{p}}. \]
\end{mydef} 

The $\TLp{p}$ distance is related to the $p$-Wasserstein \cite{Santambrogio,villani2009} distance between the measures $\mu$ and $\nu$ and we refer to \cite{Trillos3} for more details. In particular, from the latter property, we can characterize convergence in the $\TLp{p}$ space as follows (\cite[Proposition 3.12]{Trillos3}).

\begin{proposition}
Let $(\mu,u) \in \TLp{p}$ where $\mu$ is absolutely continuous with respect to Lebesgue measure and let $\{(\mu_n,u_n)\}_{n=1}^\infty$ be a sequence in $\TLp{p}$. The following are equivalent:
\begin{enumerate}
    \item $(\mu_n,u_n)$ converges to $(\mu,u)$ in $\TLp{p}$;
    \item $\mu_n$ converges weakly to $\mu$ and there exists a sequence of transport maps $T_n$, $n=1\dots,\infty$, with $(T_{n})_\# \mu = \mu_n$ and $\int_\Omega \vert x - T_n(x) \vert \, \dd x \to 0$ such that
    \[
    \int_\Omega \vert u(x) - u_n(T_n(x))\vert^p \, \dd \mu(x) \to 0;
    \]
\end{enumerate}
\end{proposition}

To compare discrete functions with their continuum counterparts, we let
\(\mu_n := \frac1n\sum_{i=1}^n \delta_{x_i}\) denote the empirical measure associated with the samples \(\{x_i\}_{i=1}^n\), and let \(\mu\) denote the sampling measure on \(\Omega\).
To apply the above result, we require transport maps \(T_n\) pushing \(\mu\) forward to \(\mu_n\); the existence of such maps with suitable quantitative control is guaranteed by the following theorem~\cite[Theorem~1.1]{garcia_trillos_slepcev_2015}.

\begin{theorem}[Existence of transport maps]
\label{thm:Back:TLp:LinftyMapsRate}
Let $\Omega \subset \bbR^d$ be open, connected and bounded with Lipschitz boundary. Let $\mu$ be a probability measure on $\Omega$ with a density that is bounded above and below by positive constants. Let $x_i\iid\mu\in\cP(\Omega)$ and we denote the empirical measure of $\{x_i\}_{i=1}^n$ by $\mu_n$. 
Then, there exists a constant $C > 0$ such that $\bbP$-a.s., there exists a sequence of transport maps $\{T_n:\Omega \mapsto \Omega_n \}_{n=1}^\infty$ from $\mu$ to $\mu_n$ such that:
\begin{equation} \notag
\begin{cases}
\limsup_{n \to \infty} \frac{n^{1/2} \Vert \Id - T_n \Vert_{\Lp{\infty}} }{\log(\log(n))} \leq C & \text{if } d = 1; \\
\limsup_{n \to \infty} \frac{n^{1/2} \Vert \Id - T_n \Vert_{\Lp{\infty}} }{\log(n)^{3/4}} \leq C & \text{if } d = 2; \\
\limsup_{n \to \infty} \frac{n^{1/d} \Vert \Id - T_n \Vert_{\Lp{\infty}} }{\log(n)^{1/d}} \leq C &\text{if } d \geq 3.
\end{cases}
\end{equation}
\end{theorem}
The probability measure $\mathbb{P}$ is defined in Section \ref{sec:main}.
In terms of the assumptions we introduce later, the conditions in the above theorem are given by~\ref{ass:Main:Ass:S1}, \ref{ass:Main:Ass:M1}, \ref{ass:Main:Ass:M2} and~\ref{ass:Main:Ass:D1}.

\subsection{\texorpdfstring{$\Gamma$}{Gamma}-Convergence} 

The appropriate framework to describe the convergence of variational problems is $\Gamma$-convergence from the calculus of variations. We only recall the key properties used in this paper similarly to what be found in \cite{Trillos3,Slepcev,weihs2023consistency} and refer to \cite{gammaConvergence} for more details.

\begin{mydef} 
Let $(Z,d_Z)$ be a metric space and $F_n:Z \to \bbR$ a sequence of functionals.
We say that $F_n$ $\Gamma$-converges to $F$ with respect to $d_Z$ if:
\begin{enumerate}
\item For every $z \in Z$ and every sequence $\{z_n\}$ with $d_Z(z_n,z) \to 0$:
\[ \liminf_{n \to \infty} F_n(z_n) \geq F(z); \]
\item For every $z \in Z$, there exists a sequence $\{z_n\}$ with $d_Z(z_n,z) \to 0$ and
\[ \limsup_{n\to \infty} F_n(z_n) \leq F(z). \]
\end{enumerate}
\end{mydef}

The notion of $\Gamma$-convergence allows one to derive the convergence of minimizers from compactness.

\begin{mydef} 
We say that a sequence of functionals $F_n:Z \to \bbR$ has the compactness property if the following holds: if $\{n_k\}_{k \in \bbN}$ is an increasing sequence of integers and $\{z_k\}_{k \in \bbN}$ is a bounded sequence in $Z$ for which $\sup_{k\in \bbN} F_{n_k}(z_k) < \infty$, then the closure of $\{z_k\}$ has a convergent subsequence. 
\end{mydef}

\begin{proposition}
\label{prop:Back:Gamma:minimizers}
Convergence of minimizers.
Let $F_n:Z \mapsto [0,\infty]$ be a sequence of functionals which are not identically equal to $\infty$.
Suppose that the functionals satisfy the compactness property and that they $\Gamma$-converge to $F:Z \mapsto [0,\infty]$.
Then
\[ \lim_{n\to \infty} \inf_{z\in Z} F_n(z) = \min_{z \in Z} F(z). \]
Furthermore, the closure of every bounded sequence $\{z_n\}$ for which \begin{equation} \label{eq:Back:Gamma:MinConv}
\lim_{n \to \infty} \left(F_n(z_n) - \inf_{z \in Z} F_n(z) \right) = 0
\end{equation}
has a convergent subsequence and each of its cluster points is a minimizer of $F$.
In particular, if $F$ has a unique minimizer, then any sequence satisfying~\eqref{eq:Back:Gamma:MinConv} converges to the unique minimizer of $F$.
\end{proposition}

In this paper, we show that our discrete objectives $\Gamma$-converge (with respect to the $\TLp{p}$-topology) to the appropriate continuum objectives. Then, we prove that the sequence of discrete minimizers is precompact in $\TLp{p}$ and, using Proposition \ref{prop:Back:Gamma:minimizers}, deduce that the latter converge to the continuum minimizers.

\section{Main results} \label{sec:main}

In this section, we present our main results as well as the relevant notation and assumptions used for our proofs. 

\subsection{General notation and setting}

For $z \in \bbR^d$ and $A \in \bbR^{d \times d}$, we denote by $(z)_i$ the $i$-th coordinate of $z$ and by $(A)_{ij}$ the $ij$-th element of $A$. We denote the Sobolev space of functions with $k$-th order derivatives in $\Lp{p}$ as \( \mathrm{W}^{k,p} \) \cite{leoni2017first}.

We will use the same probabilistic setting detailed in \cite{weihs2023consistency}. In particular, the idea is to consider a probability space with measure $\bbP$ in which elements are sequences $\{x_i\}_{i=1}^\infty$. Our results will be formulated in terms of $\bbP$, showing that certain properties hold for a set $\Psi$ of sequences $\{x_i\}_{i=1}^\infty$ with $\bbP(\Psi) = 1$.

\subsection{Hypergraph setting} 

Given a set of $n$ feature vectors $\Omega_n = \{x_i\}_{i=1}^n\subset \Omega \subset \bbR^d$ where we assume that $x_i\iid \mu\in\cP(\Omega)$, $\mu$ has density $\rho$, a length-scale $\eps >0$ and a function $\eta:[0,\infty) \mapsto [0,\infty)$, we can define weights $w_{\eps,ij}$ between vertices $x_i$ and $x_j$ as follows:
\( \notag 
w_{\eps,ij} = \eta\l\frac{\vert x_i - x_j \vert}{\eps}\r.
\)
The graph $(\Omega_n,W_{n,\eps})$ where $W_{n,\eps} = \{w_{\eps,ij}\}_{i,j=1}^n$ is called a random geometric graph \cite{DBLP:books/ox/P2003}.

We now define essential matrices related to such graphs. Let $D_{n,\eps}$ be the diagonal matrix with entries $d_{n,\eps,ii} = \sum_{j = 1}^n w_{\eps,ij}$ and define $\sigma_\eta = \frac{1}{d}\int_{\bbR^d} \eta(\vert h \vert) \vert h \vert^2 \,\dd h < \infty$.
The graph Laplacian is defined as
\( \notag
\Delta_{n,\eps} := \frac{2}{\sigma_\eta n\eps^{d+2}} (D_{n,\eps} - W_{n,\eps}).
\)
 The latter can be interpreted as a matrix $\Delta_{n,\eps}\in\bbR^{n\times n}$ or as an operator $\Delta_{n,\eps}:\Lp{2}(\mu_n)\to\Lp{2}(\mu_n)$ where $\mu_n=\frac{1}{n}\sum_{i=1}^n \delta_{x_i}$ is the empirical measure.

Given functions $u_n,v_n: \Omega_n \to \bbR$, we also define the $\Lp{2}(\mu_n)$ inner product: 
\begin{equation} \notag
\langle u_n, v_n \rangle_{\Lp{2}(\mu_n)} = \frac{1}{n} \sum_{i=1}^n u_n(x_i)v_n(x_i).
\end{equation} 
Such functions can be considered vectors in $\bbR^n$ and we will understand $u_n$ as both a function $u_n:\Omega_n\to \bbR$ and a vector $\bbR^n$.


We can generalize the random geometric graph weight model to create random geometric hypergraphs. In particular, we define the weight of a hyperedge of size $k+1$ by aggregating pairwise interactions as \begin{equation} \label{eq:main:hypergraphWeights}
w_{\eps,i_{0}\cdots i_{k}} = \prod_{j=1}^{k} \prod_{r=0}^{j-1} w_{\eps,i_j i_r}.
\end{equation}
This construction biases the model toward hyperedges whose constituent nodes lie within a shared neighborhood, effectively encoding a finer notion of locality. For instance, choosing \( \eta = \mathbf{1}_{[0,1]} \) yields \( w_{\varepsilon, i_0 \cdots i_k} > 0 \) if and only if the entire tuple $(x_{i_0},\dots,x_{i_k})$ lies within a common ball of radius \( \varepsilon \). 
In this sense, $\eps$ should be thought of as the length-scale of interaction between vertices. We denote by $t(k)$ the number of terms in the product $\prod_{j=1}^{k} \prod_{r=0}^{j-1} w_{\eps,i_j i_r}$. We refer to \cite{hgLearningPractice} for a review of other hyperedge models used in practice.

\begin{remark}[Comparison with the neighborhood hypergraph model] \label{rem:weights} 
In \cite{shi2025hypergraphplaplacianequationsdata}, the authors study the large-data limit of \eqref{eq:TVHG} on the neighborhood hypergraph $(\Omega_n,E)$ with $E=\{e_k\}_{k=1}^n$, $e_k=\Omega_n\cap B(x_k,\varepsilon)$, endowed with homogeneous weights $w(e_k)=1$. We contrast this construction with ours below.

Geometrically, in the neighborhood-hyperedge model with homogeneous weights, a hyperedge contains all points within distance $\varepsilon$ of $x_k$; thus two vertices $x_i,x_j$ can lie in the same hyperedge even if $|x_i-x_j|>\varepsilon$ (indeed, $|x_i-x_j|$ can be as large as $2\varepsilon$). By contrast, our product-type hyperedge weights enforce geometric coherence: a hyperedge receives large weight only when all pairwise kernel affinities $\eta(|x_i-x_j|/\varepsilon)$ within it are simultaneously large, so that a single distant pair strongly downweights the entire hyperedge. This implements a soft $\varepsilon$-clique rule. In the compactly supported case $\eta=\mathbf{1}_{[0,1]}$, it reduces to the hard condition $|x_i-x_j|\leq \varepsilon$ for all pairs, i.e. an $\varepsilon$-clique.

This geometric distinction has modeling consequences. Local hypergraph constructions typically rely on a homophily assumption: proximity (or strong affinity) in feature space correlates with label agreement. Under this premise, a multiway interaction is most reliable when the participating vertices are mutually similar, since joint label agreement is then plausible across all pairs in the group. Our product-type weighting encodes this directly: it assigns high weight only to hyperedges that are internally coherent (no “outlier” pair), whereas homogeneous neighborhood hyperedges can pool vertices that are each close to a common center but not necessarily close to one another---thereby diluting the intended homophily effect.

Lastly, product-type weights induce effective couplings controlled by local clique statistics, hence by $\rho$ and $\varepsilon$, producing data-adaptive smoothing. In the continuum, the density prefactor scales as $\sum_{k=1}^q \lambda_k \sigma_\eta^{(k,p)} \rho^{k+1}$ (Theorem \ref{thm:main}), whereas neighbourhood-based models typically yield $\rho$ \cite{shi2025hypergraphplaplacianequationsdata} (or $\rho^2$ for random geometric graphs \cite{Slepcev}). Thus the density dependence is both higher-degree and tunable via $q$ and $\{\lambda_k\}_{k=1}^q$.

\end{remark}

The weight construction introduced in~\eqref{eq:main:hypergraphWeights} serves as the foundation for the theoretical analysis in Theorems \ref{thm:pointwiseConsistency} and \ref{thm:main}. In particular, we reformulate the hypergraph learning energy \eqref{eq:intro:clustering} using this weight model in \eqref{eq:main:hypergraphEnergy}. With $$\eta_{\mathrm{p}}(x_{i_0},\dots,x_{i_k}) = \prod_{j=1}^k \prod_{r=0}^{j-1} \eta \left(\frac{\vert x_{i_j}- x_{i_r} \vert}{\eps}\right) = w_{\eps,i_{0}\cdots i_{k}},$$
we define the discrete $(k,p)$-Laplacian operators which we relate to the hypergraph learning energy \eqref{eq:main:hypergraphEnergy} in Proposition \ref{prop:eulerLagrange} as  
\begin{align}
\Delta_{n,\eps}^{(k,p)}(u)(x_{i_0}) &= \frac{1}{n^k\eps^{p + kd}} \hspace{-3mm}\sum_{i_1,\dots,i_k = 1}^{n}   \etaP(x_{i_0},\dots,x_{i_k}) \vert u(x_{i_1}) - u(x_{i_0}) \vert^{p-2} (u(x_{i_1}) - u(x_{i_0})). \notag
\end{align}
We note that the $(1,2)$-Laplacian is just $\Delta_{n,\eps}$ (up to normalization).

In order to introduce the continuum counterpart of $\Delta^{(k,p)}_{n,\eps}$, we first define 
\[\etaPTilde(z_1,\dots,z_k) = \left[ \prod_{s=1}^k \eta(\vert z_s\vert) \right]  \left[ \prod_{j=2}^k \prod_{r=1}^{j-1} \eta(\vert z_j - z_r \vert) \right],\]
and the constant
$\sigma_{\eta}^{(k,p)}  = \int_{(\bbR^d)^k}  \etaPTilde(\tilde{z}_1,\dots,\tilde{z}_k)  \vert (\tilde{z}_1)_d \vert^p \, \dd \tilde{z}_k \cdots \dd \tilde{z}_1.$
In Theorem \ref{thm:pointwiseConsistency}, we will establish the precise asymptotic relationship between $\Delta^{(k,p)}_{n,\eps}$ and 
    $
    \Delta_\infty^{(k,p)}(u)(x) = \frac{\sigma_{\eta}^{(k,p)}}{2\rho(x)}\Div(\rho^{k+1}\Vert \nabla u \Vert_2^{p-2} \nabla u)(x).
    $

\subsection{Variational problems for hypergraph learning}

For some $p > 1$ and a fixed hyperedge size $k \geq 1$, the classical hypergraph energy can be written as
\begin{equation} \label{eq:main:hypergraphEnergy}
    \cE_{n,\eps}^{(k,p)}(v) = \frac{1}{n^{k+1}\eps^{p + kd}} \sum_{i_0,\cdots, i_k = 1}^{n} \ls \prod_{j=1}^{k} \prod_{r=0}^{j-1} \eta\l \frac{\vert x_{i_j} - x_{i_{r}}\vert}{\eps} \r \rs \left\vert v(x_{i_1}) - v(x_{i_0}) \right\vert^p  
\end{equation}
for $v:\Omega \to \bbR$ while the associated discrete semi-supervised learning objective is
\[
\cF_{n,\eps}^{(k,p)}( (\nu,v) ) = \begin{cases} \cE_{n,\eps}^{(k,p)}(v) &\text{if } \nu = \mu_n \text{ and for } i\leq N, v(x_i) = y_i \\ +\infty & \text{else} \end{cases}
\]
for $(\nu,v) \in \TLp{p}(\Omega)$ and where $\{y_i\}_{i=1}^N \subset \{0,1\}$ are binary labels. 

In the continuum, we define
\begin{align}
\cE_{\infty}^{(k,p)}(v) &= \int_{\Omega} \int_{(\bbR^d)^k}\etaPTilde(z_1,\dots,z_k)  \left\vert \nabla v(x_0) \cdot z_1 \right\vert^p \rho(x_0)^{k+1} \, \dd z_k \cdots \dd z_1 \dd x_0 \notag \\
&= \int_{(\bbR^d)^k}\etaPTilde(z_1,\dots,z_k) \left\vert e \cdot z_1 \right\vert^p \dd z_k \cdots \dd z_1 \int_{\Omega} \Vert \nabla v(x_0) \Vert_2^p \,  \rho(x_0)^{k+1} \, \dd x_0 \label{eq:intro:isotropy} \\
&=: \sigma_\eta^{(k,p)} \int_{\Omega} \Vert \nabla v(x_0) \Vert_2^p \,  \rho(x_0)^{k+1} \, \dd x_0 \notag
\end{align}
where $e \in \bbR^d$ is any vector with $\Vert e \Vert_2=1$ and \eqref{eq:intro:isotropy} follows by isotropy of the kernels. 
The corresponding semi-supervised learning objectives are:
\[
\cF_{\infty}^{(k,p)}( (\nu,v) ) = \begin{cases} \cE_{\infty}^{(k,p)}(v) &\text{if } \nu = \mu, \, v \in \Wkp{1}{p}(\Omega) \text{ and for } i\leq N, v(x_i) = y_i \\ +\infty & \text{else,} \end{cases}
\]
\[
\cG_{\infty}^{(k,p)}( (\nu,v) ) = \begin{cases} \cE_{\infty}^{(k,p)}(v) &\text{if } \nu = \mu \text{ and } v \in \Wkp{1}{p}(\Omega) \\
+\infty &\text{else.} \end{cases}
\]

Our final objective, for $q \geq 1$ and a positive sequence $\{\lambda_k\}_{k=1}^q \subseteq \bbR$, is to consider the sums 
\begin{equation} \label{eq:SSLObjectiveSum}
    (\cS\cF)_{n,\eps}^{(q,p)}((\nu,v)) = \sum_{k=1}^q \lambda_k \cF_{n,\eps}^{(k,p)}((\nu,v)), \quad (\cS\cF)_{\infty}^{(q,p)}((\nu,v)) = \sum_{k=1}^q \lambda_k \cF_{\infty}^{(k,p)}((\nu,v))
\end{equation}
and 
$
(\cS\cG)_{\infty}^{(q,p)}((\nu,v)) = \sum_{k=1}^q \lambda_k \cG_{\infty}^{(k,p)}((\nu,v)).
$

\subsubsection{Comparison between \(V\)-statistic and distinct-vertex formulations}
\label{subsubsec:VUcomparison}

We now make three remarks clarifying the formulation of the discrete
energy~\eqref{eq:main:hypergraphEnergy}. They address, respectively, the role
of degenerate tuples, the energy-level comparison with the distinct-index
formulation, and the relation to pairwise-aggregation hypergraph energies.

\begin{remark}[Degenerate tuples in the V-statistic formulation]\label{rem:degenerateTuples}
For a fixed \(k\ge 1\), corresponding to hyperedges of size \(k+1\), and point
cloud size \(n\), our energy
\eqref{eq:main:hypergraphEnergy} is written as a V-statistic, i.e. a sum over
all ordered tuples $(i_0,\dots,i_k)\in\{1,\dots,n\}^{k+1}$. In a literal hypergraph
interpretation, a hyperedge of size $k+1$ should contain distinct vertices, whereas
the V-statistic indexing also includes \emph{degenerate} tuples with repeated indices
(corresponding to repeated vertices inside the same multiway interaction). We adopt
the V-statistic form because it is algebraically convenient for discrete-to-continuum
arguments (it factorizes into products of empirical sums and passes directly to
iterated integrals).

To quantify the effect of degeneracies, define the total, non-degenerate, and
degenerate index sets
\[
\mathcal{T}_{n,k}:=\{1,\dots,n\}^{k+1},\qquad
\mathcal{ND}_{n,k}:=\{(i_0,\dots,i_k)\in\mathcal{T}_{n,k}: i_r\neq i_s\ \forall r\neq s\},
\]
\[
\mathcal{D}_{n,k}:=\mathcal{T}_{n,k}\setminus \mathcal{ND}_{n,k}
=\{(i_0,\dots,i_k)\in\mathcal{T}_{n,k}: \exists\, r\neq s \text{ with } i_r=i_s\}.
\]
Clearly $\#\mathcal{T}_{n,k}=n^{k+1}$. Moreover,
\(
\#\mathcal{ND}_{n,k}
= n(n-1)\cdots(n-k),
\)
since $i_0$ can be chosen in $n$ ways, then $i_1$ in $(n-1)$ ways, and so on until
$i_k$ in $(n-k)$ ways. Hence the exact degenerate fraction is
\begin{equation}\label{eq:degFracExact} \notag
\frac{\#\mathcal{D}_{n,k}}{n^{k+1}}
=1-\frac{n(n-1)\cdots(n-k)}{n^{k+1}}
=1-\prod_{j=0}^{k}\Bigl(1-\frac{j}{n}\Bigr).
\end{equation}

A simple uniform bound follows from a union bound. For each pair of positions
$0\le r<s\le k$, let $A_{r,s}:=\{(i_0,\dots,i_k)\in\mathcal{T}_{n,k}: i_r=i_s\}$.
Then $\mathcal{D}_{n,k}=\bigcup_{0\le r<s\le k}A_{r,s}$, and since fixing $i_r=i_s$
imposes one equality constraint, $\#A_{r,s}=n^k$ (choose the common value in $n$
ways and the remaining $k-1$ indices freely). Therefore
\[
\frac{\#\mathcal{D}_{n,k}}{n^{k+1}}
\le \sum_{0\le r<s\le k}\frac{\#A_{r,s}}{n^{k+1}}
= \binom{k+1}{2}\frac{n^k}{n^{k+1}}
= \frac{k(k+1)}{2n}.
\]
In particular, for fixed $k$ and $n\to\infty$, degenerate configurations form a vanishing fraction of all tuples. Thus, the V-statistic is a reasonable approximation of the corresponding \emph{distinct-vertex} (U-statistic) formulation at the level of index counting.
\end{remark}

\begin{remark}[Energy comparison between U- and V-statistics]\label{rem:VUenergy}
We now compare the V-statistic energy \eqref{eq:main:hypergraphEnergy}, in which one sums over all ordered tuples
$(i_0,\dots,i_k)\in\{1,\dots,n\}^{k+1}$ (allowing repeated indices), to its U-statistic analogue obtained by restricting
to tuples with distinct indices. Under standard assumptions on the kernel $\eta$ (e.g.\ Assumption~\ref{ass:Main:Ass:W2}
with $\supp(\eta)\subset[0,R]$ for some $R>0$), on the sampling density $\mu$ (e.g.\ Assumption~\ref{ass:Main:Ass:M2}),
and on the scale $\eps_n$ (e.g.\ Assumption~\ref{ass:Main:Ass:L1}), and for $v$ regular enough (e.g.\ $L$-Lipschitz),
the V- and U-statistic energies are asymptotically equivalent provided $n\eps_n^d\to\infty$.

Throughout, $C>0$ denotes a constant independent of $n$ and $\eps=\eps_n$, which may change from line to line.
Recall the index sets from Remark~\ref{rem:degenerateTuples}:
$\mathcal{T}_{n,k}=\{1,\dots,n\}^{k+1}$, $\mathcal{ND}_{n,k}\subset\mathcal{T}_{n,k}$ the set of tuples with all indices
distinct, and $\mathcal{D}_{n,k}=\mathcal{T}_{n,k}\setminus\mathcal{ND}_{n,k}$ the degenerate tuples.
Define
\[
\cE^{V}_{n,\eps}(v)
:=\frac{1}{n^{k+1}\eps^{p+kd}}
\sum_{(i_0,\dots,i_k)\in\mathcal{T}_{n,k}}
\eta_{\mathrm{p}}(x_{i_0},\dots,x_{i_k})\,|v(x_{i_1})-v(x_{i_0})|^p,
\]
\[
\cE^{U}_{n,\eps}(v)
:=\frac{1}{n^{k+1}\eps^{p+kd}}
\sum_{(i_0,\dots,i_k)\in\mathcal{ND}_{n,k}}
\eta_{\mathrm{p}}(x_{i_0},\dots,x_{i_k})\,|v(x_{i_1})-v(x_{i_0})|^p.
\]

Since the U-statistic sum is a restriction of the V-statistic sum, we have $\cE^{V}_{n,\eps}(v)\ge \cE^{U}_{n,\eps}(v)$ and
\[
0\le \cE^{V}_{n,\eps}(v)-\cE^{U}_{n,\eps}(v)
=\frac{1}{n^{k+1}\eps^{p+kd}}
\sum_{(i_0,\dots,i_k)\in\mathcal{D}_{n,k}}
\eta_{\mathrm{p}}(x_{i_0},\dots,x_{i_k})\,|v(x_{i_1})-v(x_{i_0})|^p.
\]
If $\eta_{\mathrm{p}}(x_{i_0},\dots,x_{i_k})\neq 0$ and $\supp(\eta)\subset[0,R]$, then in particular
$\eta(|x_{i_1}-x_{i_0}|/\eps)\neq 0$, hence $|x_{i_1}-x_{i_0}|\le R\eps$. If $v$ is $L$-Lipschitz, this implies
$|v(x_{i_1})-v(x_{i_0})|^p\le (LR\eps)^p$. Moreover, $0\le \eta_{\mathrm{p}}\le \|\eta\|_{L^\infty}^{t(k)}$.
Therefore
\begin{equation}\label{eq:VU_step1}
0\le \cE^{V}_{n,\eps}(v)-\cE^{U}_{n,\eps}(v)
\le \frac{\|\eta\|_{L^\infty}^{t(k)}(LR)^p}{n^{k+1}\eps^{kd}}
\sum_{(i_0,\dots,i_k)\in\mathcal{D}_{n,k}}
\mathbf{1}_{\{\eta_{\mathrm{p}}(x_{i_0},\dots,x_{i_k})\neq 0\}}.
\end{equation}

To bound the remaining count, for each $i\in\{1,\dots,n\}$ define the local occupancy number
\(
N_i := \#\{m:\ |x_m-x_i|\le R\eps\} = n\,\mu_n(B(x_i,R\eps))
\)
and
for fixed $i_0$, set
\(
S_{i_0}:=\{m\in\{1,\dots,n\}:\ |x_m-x_{i_0}|\le R\eps\}
\)
(note that $\#S_{i_0}=N_{i_0}$).
If $\eta_{\mathrm{p}}(x_{i_0},\dots,x_{i_k})\neq 0$, then necessarily $x_{i_1},\dots,x_{i_k}\in B(x_{i_0},R\eps)$, so for
fixed $i_0$ there are at most $N_{i_0}^k$ admissible choices of $(i_1,\dots,i_k)$. 
Among these admissible tuples, degeneracy occurs if $i_r=i_s$ for some $1\le r<s\le k$ (a repeated vertex among
$i_1,\dots,i_k$) or if $i_j=i_0$ for some $1\le j\le k$ (a repetition with the anchor index $i_0$).
A crude but convenient union bound yields
\[
\#\Bigl\{(i_1,\dots,i_k)\in S_{i_0}^k:\ \exists\, r<s \text{ with } i_r=i_s \ \text{or}\ \exists\, j \text{ with } i_j=i_0\Bigr\}
\;\le\; \binom{k+1}{2}\,N_{i_0}^{k-1}.
\]
Indeed, fix a pair of positions $0\le r<s\le k$ and consider the constraint $i_r=i_s$ (with $i_0$ fixed).
If $1\le r<s\le k$, then the set
$A_{r,s}(i_0):=\{(i_1,\dots,i_k)\in S_{i_0}^k:\ i_r=i_s\}$
has cardinality $\#A_{r,s}(i_0)=N_{i_0}^{k-1}$: one chooses the common value of $(i_r,i_s)$ in $N_{i_0}$ ways and
chooses the remaining $k-2$ indices freely in $S_{i_0}$.
If $r=0$ and $1\le s\le k$, then the constraint $i_s=i_0$ fixes that index, leaving $k-1$ free choices in $S_{i_0}$,
so again $\#A_{0,s}(i_0)=N_{i_0}^{k-1}$.
Since any degenerate admissible tuple lies in $\bigcup_{0\le r<s\le k}A_{r,s}(i_0)$, we obtain
\[
\#(\text{degenerate admissible tuples}) \le \sum_{0\le r<s\le k}\#A_{r,s}(i_0)
= \binom{k+1}{2}\,N_{i_0}^{k-1}.
\]
Consequently,
\(
\sum_{(i_0,\dots,i_k)\in\mathcal{D}_{n,k}}
\mathbf{1}_{\{\eta_{\mathrm{p}}(x_{i_0},\dots,x_{i_k})\neq 0\}}
\le
\sum_{i_0=1}^{n} \binom{k+1}{2}\,N_{i_0}^{k-1}.
\)
Plugging this into \eqref{eq:VU_step1} yields the bound
\begin{equation}\label{eq:VU_masterbound}
0\le \cE^{V}_{n,\eps}(v)-\cE^{U}_{n,\eps}(v)
\le  \frac{C}{n^{k+1}\eps^{kd}}\sum_{i_0=1}^{n} N_{i_0}^{k-1},
\end{equation}
for a constant $C=C(k,p,\eta,R,L)$ independent of $n$ and $\eps$.

Finally, as in the proof of
Proposition~\ref{prop:proofs:gammaConvergence:limsupSSL}, one has for $n$ large enough
\(
N_i = n\,\mu_n(B(x_i,R\eps)) \;\le\; C\,n\eps^d
\)
for all $1 \leq i \leq n$.
Then
\(
\sum_{i_0=1}^{n} N_{i_0}^{k-1}
\le
n\,(C\,n\eps^d)^{k-1}
=
C^{k-1}\,n^{k}\,\eps^{d(k-1)}.
\)
Substituting into \eqref{eq:VU_masterbound} gives
\(
0\le \cE^{V}_{n,\eps}(v)-\cE^{U}_{n,\eps}(v)
\le
\frac{C}{n\eps^{d}}.
\)
In particular, for fixed $k$ and any regime with $n\eps_n^{d}\to\infty$ (e.g.\ Assumption~\ref{ass:Main:Ass:L1}),
we have $\cE^{V}_{n,\eps_n}(v)-\cE^{U}_{n,\eps_n}(v)\to 0$. This further justifies the interchangeability of the V- and U-statistic formulations in the discrete-to-continuum analysis.
\end{remark}

\begin{remark}[Symmetrization and pairwise aggregation]\label{rem:symmetrize_pairwise}
The purpose of this remark is to show that \eqref{eq:main:hypergraphEnergy} can be viewed as an asymptotic pairwise-aggregation energy (of the form \eqref{eq:intro:clustering}): the nondegenerate U-statistic symmetrizes to an all-pairs sum on each $(k+1)$-subset, and the V-statistic differs only by degenerate tuples whose contribution vanishes under standard scaling (see Remark \ref{rem:VUenergy}).

First, note that the definition of $\eta_{\mathrm{p}}$ involves all pairwise interactions among its $(k+1)$ arguments; therefore it is
symmetric under permutations of these arguments:
for every permutation $\pi$ of $\{0,1,\dots,k\}$,
\[
\eta_{\mathrm{p}}(x_{i_{\pi(0)}},\dots,x_{i_{\pi(k)}})=\eta_{\mathrm{p}}(x_{i_0},\dots,x_{i_k}).
\]
In particular, when the indices are distinct, $\eta_{\mathrm{p}}(x_{i_0},\dots,x_{i_k})$ depends only on the
underlying vertex set $S=\{i_0,\dots,i_k\}$, and we may write $\eta_{\mathrm{p}}(S)$.

Consider the (nondegenerate) U-statistic version of \eqref{eq:main:hypergraphEnergy},
\[
\cE^{U}_{n,\eps}(v)
:=\frac{1}{n^{k+1}\eps^{p+kd}}
\sum_{(i_0,\dots,i_k)\in\mathcal{ND}_{n,k}}
\eta_{\mathrm{p}}(x_{i_0},\dots,x_{i_k})\,|v(x_{i_1})-v(x_{i_0})|^p,
\]
which singles out the pair $(i_0,i_1)$ in each ordered tuple (recall the notation in Remark \ref{rem:degenerateTuples}).  Although this expression
does not explicitly sum over all pairs in a hyperedge, it becomes an all-pairs
(pairwise-aggregation) energy after symmetrization over permutations of the same vertex set.

Indeed, fix a set $S=\{a_0,\dots,a_k\}$ of $k{+}1$ distinct indices.  Summing over all
$(k{+}1)!$ orderings of $S$ and using the permutation invariance of $\eta_{\mathrm{p}}$, we obtain
\begin{align*}
\sum_{\pi\in\mathfrak{S}_{k+1}}
\eta_{\mathrm{p}}(S)\,\bigl|v(x_{a_{\pi(1)}})-v(x_{a_{\pi(0)}})\bigr|^p
&=
\eta_{\mathrm{p}}(S)\sum_{\pi\in\mathfrak{S}_{k+1}}
\bigl|v(x_{a_{\pi(1)}})-v(x_{a_{\pi(0)}})\bigr|^p \\
&=
2\,(k-1)!\;\eta_{\mathrm{p}}(S)\sum_{\{a,b\}\subset S}\bigl|v(x_a)-v(x_b)\bigr|^p,
\end{align*}
since for each unordered pair $\{a,b\}\subset S$, there are exactly $2\,(k-1)!$ permutations
$\pi$ such that $(a_{\pi(0)},a_{\pi(1)})$ equals $(a,b)$ or $(b,a)$. Here $\mathfrak{S}_{k+1}$ denotes the symmetric group on $\{0,1,\dots,k\}$, i.e.,
the set of all permutations of $(k+1)$ elements.

Consequently, grouping the U-statistic sum by the underlying $(k{+}1)$-subsets $S$ yields
\[
\cE^{U}_{n,\eps}(v)
=
\frac{2\,(k-1)!}{n^{k+1}\eps^{p+kd}}
\sum_{\substack{S\subset\{1,\dots,n\}\\ |S|=k+1}}
\eta_{\mathrm{p}}(S)\sum_{\{a,b\}\subset S}\bigl|v(x_a)-v(x_b)\bigr|^p,
\]
which is a pairwise-aggregation (as in \cite{scholkopfHyper2006}) energy with hyperedge weight $\eta_{\mathrm{p}}(S)$
(up to the explicit combinatorial prefactor $2\,(k-1)!$).
\end{remark}

Taken together, these remarks show that the \(V\)-statistic formulation is an
analytically convenient representation of an asymptotic pairwise-aggregation
hypergraph energy: degenerate tuples are asymptotically negligible, and once
one restricts to distinct indices, summing over all permutations of a fixed
hyperedge recovers the usual all-pairs interaction within that hyperedge.

\subsection{Higher-order hypergraph learning} \label{sec:main:HOHL}

In this section, we introduce the Higher-Order Hypergraph Learning (HOHL) model. 
Let $(V,E)$ be a hypergraph (independently of its weight model), and define \( q = \max_{e \in E} |e| - 1 \) as the maximum hyperedge size minus one. For each \( k \in \{1, \dots, q\} \), we define a corresponding skeleton graph \( G^{(k)} = (V, E^{(k)}) \) by
\[
E^{(k)} = \left\{ \{v_i, v_j\} \,\middle|\, \exists e \in E \text{ with } |e| = k+1 \text{ and } \{v_i, v_j\} \subset e \right\},
\]
that is, \( G^{(k)} \) includes all pairwise edges induced by hyperedges of size \( k+1 \). Let \( L^{(k)} \) denote the graph Laplacian associated with \( G^{(k)} \). 

We define the HOHL energy as
\begin{equation} \label{eq:discussion:higherOrder}
v^\top \left[ \sum_{k=1}^q \lambda_k (L^{(k)})^{p_k} \right] v 
\end{equation}
for \( v \in \bbR^n \), 
where \( 0 < p_1, \ldots, p_q \) are powers and \( 0 < \lambda_1, \ldots, \lambda_q \) are tuning parameters. In this paper, we restrict our attention to the HOHL energy~\eqref{eq:discussion:higherOrder} applied to the 
random hypergraph model where the vertex set is \( \Omega_n \subset \bbR^d \) and hyperedges are constructed via~\eqref{eq:main:hypergraphWeights}. For a generalization of~\eqref{eq:discussion:higherOrder} to non-geometric datasets and arbitrary hypergraphs, as well as an analysis of the computational properties of HOHL, we refer the reader to~\cite{weihs2025HOHL}.

\begin{figure}[htbp]
  \centering
  \begin{minipage}[b]{0.48\linewidth}
    \centering
    \includegraphics[width=\linewidth]{PenalizeSkeletonGraphs.png}
  \end{minipage}
  \hfill
  \begin{minipage}[b]{0.48\linewidth}
    \centering
    \includegraphics[width=\linewidth]{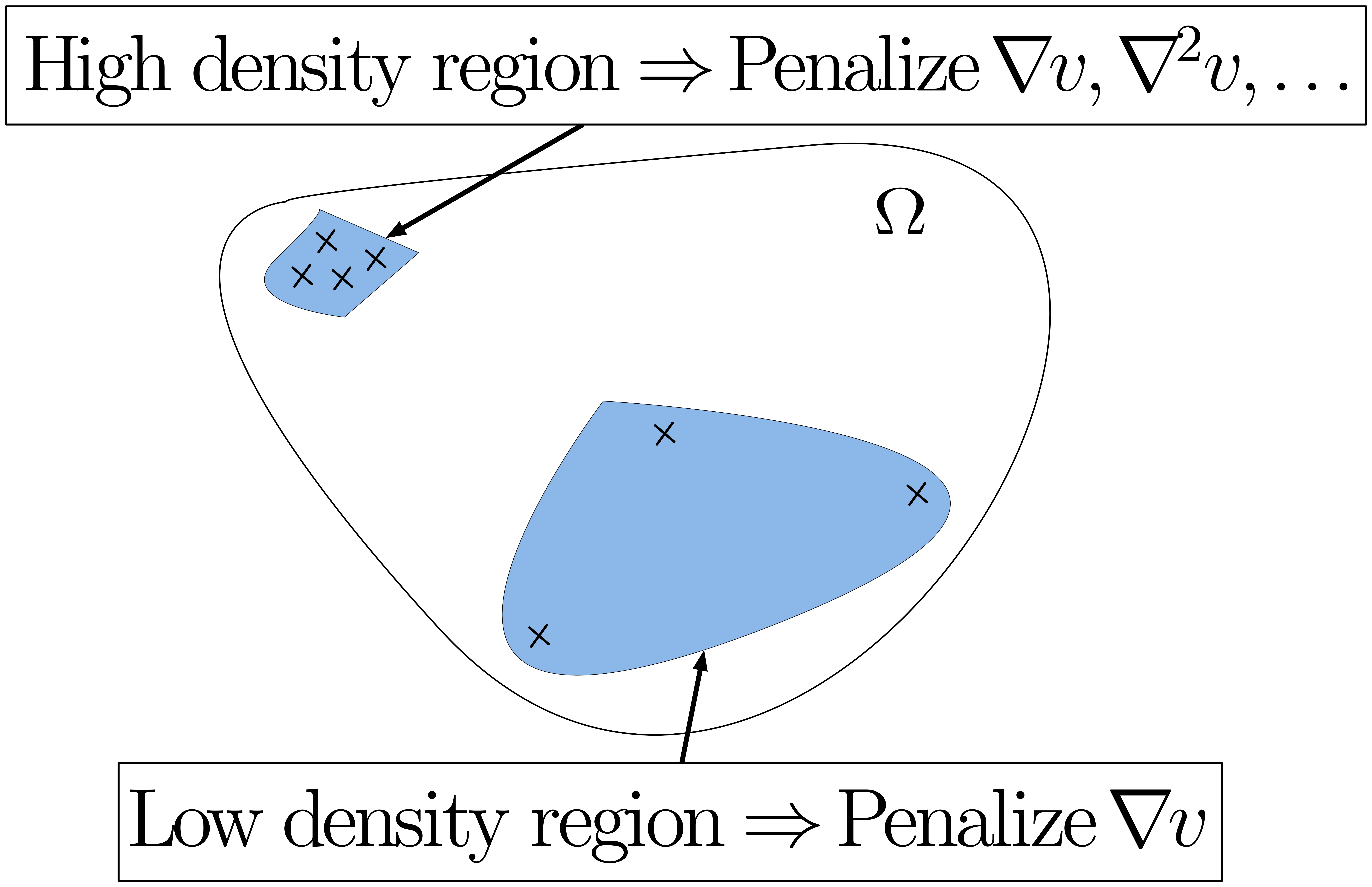}
  \end{minipage}
  \caption{
Illustration of the HOHL energy with $p_k = k$. Left: For \(q=2\), the energy imposes hierarchical regularization by penalizing \( v^\top L^{(1)} v \) on skeleton edges \(E^{(1)}\) and \( v^\top (L^{(2)})^2 v \) on \(E^{(2)}\). Right: With the random hypergraph model of \eqref{eq:main:hypergraphWeights}, in high-density regions, hyperedges of large size capture finer structural details, and HOHL imposes stronger smoothness to exploit this local structure.
  }
  \label{fig:HOHl}
\end{figure}

For hypergraph models where larger hyperedges connect increasingly closer points, given an increasing sequence $\{p_k\}_{k=1}^q$ (in practice we often set $p_k = k$ for simplicity), the HOHL energy~\eqref{eq:discussion:higherOrder}
induces a hierarchy of interaction scales and hence scale-aware regularization: small \(k\) enforces global smoothness,
while large \(k\) targets increasingly local structure and imposes finer regularity. Figure~\ref{fig:HOHl} illustrates this mechanism.
This multiscale viewpoint is closely related to the multiscale Laplacian regularizer of~\cite{Merkurjev},
\begin{equation}\label{eq:multiscale}
v^\top \Bigl(\sum_{k=1}^q \lambda_k \Delta_{n,\eps^{(k)}}^{p_k}\Bigr)v,
\end{equation}
where \( \varepsilon^{(1)}>\cdots>\varepsilon^{(q)} \) and \(p_k>0\) prescribes the order of regularization at scale \(\varepsilon^{(k)}\).
In our random hypergraph model, the skeleton graphs become increasingly local as \(k\) increases, and the Laplacians in~\eqref{eq:discussion:higherOrder}
are generated implicitly from this hypergraph-induced structure; thus~\eqref{eq:discussion:higherOrder} can be viewed as~\eqref{eq:multiscale}
without explicitly introducing multiple kernel bandwidths.

A further modeling choice in~\cite{Merkurjev} is how the exponents \(p_k\) should vary across scales.
The following continuum correspondence provides a structural guideline: under suitable conditions,
\(v^\top \Delta_{n,\varepsilon}^s v\) converges to a $\Wkp{s}{2}$ seminorm \cite{Stuart,weihs2023consistency}, so taking \(p_1<\cdots<p_q\)
amounts to enforcing higher-order regularity at progressively finer scales.

While the hypergraph construction in~\eqref{eq:main:hypergraphWeights} is
theoretically appealing, explicitly constructing the corresponding hyperedge
weights is computationally prohibitive for moderate or large point clouds. In
the V-statistic formulation used in the analysis, the sum runs over all
\((k+1)\)-tuples \((i_0,\dots,i_k)\in\{1,\dots,n\}^{k+1}\), allowing repeated
indices, so the number of candidate interactions scales as \(n^{k+1}\) (even if one instead restricts to distinct vertices, the number of ordered
tuples is \(n(n-1)\cdots(n-k)\)). 
Thus, direct enumeration of product-weighted hyperedges quickly becomes
infeasible beyond small datasets or very small hyperedge sizes. 

For this reason, in the geometric point-cloud setting we use the multiscale
Laplacian model~\eqref{eq:multiscale} as a computationally tractable surrogate
for HOHL. This surrogate preserves the central modeling principle of HOHL:
regularization is imposed across a hierarchy of interaction scales, with higher
powers of Laplacians enforcing stronger regularity at finer scales. At the same
time, it avoids explicit hyperedge enumeration and leads to an analytically
tractable continuum-limit theory. 

Although~\eqref{eq:multiscale} does not exactly approximate~\eqref{eq:discussion:higherOrder} with weights from~\eqref{eq:main:hypergraphWeights} --- since the corresponding limiting Laplacians may differ in density scaling --- it provides a practical surrogate for point clouds embedded in a metric space, where hypergraphs are constructed based on proximity and we do not have access to the (skeleton) Laplacians. In contrast, for general hypergraphs where the Laplacians \(L^{(k)}\) are
available, one can apply~\eqref{eq:discussion:higherOrder} directly. Numerical experiments in~\cite{weihs2025HOHL} indicate that this direct hypergraph implementation performs well and can effectively exploit hypergraph structure in non-geometric settings.

\subsection{Variational problems for higher-order hypergraph learning}

For HOHL and $p > 0$, we define the discrete energies $
\cI_{n,\Delta_{n,\eps}}^{(p)}(v) =  \langle v, \Delta_{n,\eps}^p v\rangle_{\Lp{2}(\mu_n)}
$
for $v:\Omega \to \bbR$ and their associated semi-supervised learning objectives
\[
\cJ_{n,\Delta_{n,\eps}}^{(p)}( (\nu,v) ) = \begin{cases} \cI_{n,\Delta_{n,\eps}}^{(p)}(v) &\text{if } \nu = \mu_n \text{ and for } i\leq N, v(x_i) = y_i \\ +\infty & \text{else} \end{cases}
\]
for $(\nu,v) \in \TLp{2}(\Omega)$.

The latter have continuum analogues. Indeed, let $\Delta_\rho$ be the continuum weighted Laplacian operator defined by
\begin{equation} \notag
\Delta_\rho u(x) = -\frac{1}{\rho(x)}\Div(\rho^2\nabla u)(x), \, x \in \Omega \quad \quad \quad \frac{\partial u}{\partial n} = 0, \, x \in \partial \Omega
\end{equation}
and let $\{(\beta_i,\psi_i)\}_{i=1}^\infty$ be its associated eigenpairs where $\beta_1 = 0 < \beta_2 \leq \beta_3 \leq \hdots$. We note that $\{\psi_i\}_{i=1}^\infty$ form a basis of $\Lp{2}(\mu)$.
The continuum energy is then defined as
$
\cI_{\infty}^{(p)}(v) = \langle v, \Delta^p_\rho v \rangle_{\Lp{2}(\mu)}
$
for $v:\Omega \to \bbR$ and we consider the following well-posed and ill-posed semi-supervised learning objectives:

\[
\cJ_{\infty}^{(p)}( (\nu,v) ) = \begin{cases} \cI_{\infty}^{(p)}(v) &\text{if } \nu = \mu, \, v \in \cH^{p}(\Omega) \text{ and for } i\leq N, v(x_i) = y_i \\ +\infty & \text{else,} \end{cases}
\]

\[
\cK_{\infty}^{(p)}( (\nu,v) ) = \begin{cases} \cI_{\infty}^{(p)}(v) &\text{if } \nu = \mu \text{ and } v \in \cH^{p}(\Omega) \\
+\infty &\text{else} \end{cases}
\]
for $(\nu,v) \in \TLp{2}(\Omega)$ and where \( \notag 
\cH^p(\Omega) = \{ h \in \Lp{2}(\mu) \spaceBar \cI_\infty^{(p)}(h) < +\infty \}. 
\)
The latter set can be shown to be very closely related to the Sobolev space $\Wkp{p}{2}(\Omega)$ \cite[Lemma 17]{Stuart}. Finally, for $q \geq 1$ and positive sequences $\{\lambda_k\}_{k=1}^q \subseteq \bbR$, $P := \{p_k\}_{k=1}^q \subseteq \bbR$ and $ E := \{\eps^{(k)}\}_{k=1}^q$ with $\eps^{(1)} > \cdots > \eps^{(q)}$, we consider the sums 
\[
(\cS\cJ)_{n,E}^{(q,P)}((\nu,v)) = \sum_{k=1}^q \lambda_k \cJ_{n,\Delta_{n,\eps^{(k)}}}^{(p_k)}((\nu,v)), \quad (\cS\cJ)_{\infty}^{(q,P)}((\nu,v)) = \sum_{k=1}^q \lambda_k \cJ_{\infty}^{(p_k)}((\nu,v))
\]
and 
$
(\cS\cK)_{\infty}^{(q,P)}((\nu,v)) = \sum_{k=1}^q \lambda_k \cK_{\infty}^{(p_k)}((\nu,v)).
$
We will also index our length-scales by the number of vertices, i.e. $\eps^{(k)} = \eps_n^{(k)}$, and in this case, we write $E_n := \{\eps_n^{(k)}\}_{k=1}^q$. The above sums correspond to the multiscale model for 
HOHL as detailed in Section~\ref{sec:main:HOHL}.

\subsection{Assumptions}

In this section, we list the assumptions used throughout the paper. 

\begin{assumptions}
Assumption on the space.
We assume either~\ref{ass:Main:Ass:S1} or~\ref{ass:Main:Ass:S2}.
\begin{enumerate}[label=\textbf{S.\arabic*}]
\item The feature vector space $\Omega$ is an open, connected and bounded subset of $\bbR^d$ with Lipschitz boundary. \label{ass:Main:Ass:S1}
\item The feature vector space $\Omega$ is the unit torus $\sfrac{\bbR^d}{\bbZ^d}$. \label{ass:Main:Ass:S2} 
\end{enumerate}
\end{assumptions}

We will use Assumption \ref{ass:Main:Ass:S2} in Theorem \ref{thm:main:HOHL}: assuming $\Omega$ is a torus simplifies the analysis by removing both boundary effects (which alter pointwise graph-to-continuum Laplacian convergence near $\partial\Omega$) and geometric curvature issues that would arise on a manifold.

\begin{assumptions}
Assumptions on the measure.
In most cases we need both~\ref{ass:Main:Ass:M1} and~\ref{ass:Main:Ass:M2}.
\begin{enumerate}[label=\textbf{M.\arabic*}]
\item The measure $\mu$ is a probability measure on $\Omega$. \label{ass:Main:Ass:M1} 
\item There is a continuous Lebesgue density $\rho$ of $\mu$ which is bounded from above and below by strictly positive constants, i.e. $0< \min_{x\in\Omega} \rho(x) \leq \max_{x\in\Omega} \rho(x) < +\infty.$ \label{ass:Main:Ass:M2}
\end{enumerate}
\end{assumptions}

The data consists of feature vectors $\{x_i\}_{i=1}^n$ and labels $\{y_i\}_{i=1}^N$ and we make the following assumptions.

\begin{assumptions}
Assumptions on the data.
Assumption~\ref{ass:Main:Ass:D1} is needed for consistency results and~\ref{ass:Main:Ass:D2} is needed in the semi-supervised setting.
\begin{enumerate}[label=\textbf{D.\arabic*}]
\item Feature vectors $\Omega_n = \{x_i\}_{i=1}^n$ are iid samples from a measure $\mu$ satisfying \ref{ass:Main:Ass:M1}.
\label{ass:Main:Ass:D1}
\item There are $N$ labels $\{y_i\}_{i=1}^N\subset \bbR$ corresponding to the first $N$ feature vectors $\{x_i\}_{i=1}^N$. \label{ass:Main:Ass:D2}
\end{enumerate}
\end{assumptions}

The weight function $\eta$ is assumed to satisfy the following assumptions.

\begin{assumptions}
Assumptions on the weight function or kernel.
\begin{enumerate}[label=\textbf{W.\arabic*}]
\item The function $\eta:[0,\infty) \to [0,\infty)$ is non-increasing, has compact support, is continuous and positive at $x=0$.
\label{ass:Main:Ass:W2}
\end{enumerate}
\end{assumptions}

The compactness of the support of $\eta$ corresponds to the setting in most applications where, for computational purposes, one wants to restrict the range of interactions between vertices in our hypergraph. Theoretically however, the compact support assumption is not strictly necessary and we can extend our results to the non-compactly supported case as in done in \cite{Trillos3,Slepcev}. 

Finally, we make the following assumption on the length scale $\eps=\eps_n$ which we scale with $n$.

\begin{assumptions}
Assumptions on the length-scale.
For our consistency results we will need one of~\ref{ass:Main:Ass:L0}, \ref{ass:Main:Ass:L1} or~\ref{ass:Main:Ass:L2}. 
\begin{enumerate}[label=\textbf{L.\arabic*}]
\item The length scale $\eps=\eps_n$ is positive, converges to 0, i.e. $0<\eps_n \to 0$. 
\label{ass:Main:Ass:L0}
\item The length scale $\eps=\eps_n$ is positive, converges to 0, i.e. $0<\eps_n \to 0$, and satisfies the following lower bound: 
\begin{equation*}
\begin{split} 
\lim_{n \to \infty} \frac{\log(n)}{n \eps_n^{d}} & = 0 \qquad \text{if } d\geq 3; \\
\lim_{n \to \infty} \frac{(\log(n))^{3/2}}{n \eps_n^2} & = 0 \qquad \text{if } d=2; \\
\lim_{n \to \infty} \frac{\log(\log(n))}{n \eps_n^2} & = 0 \qquad \text{if } d=1.
\end{split}
\end{equation*}
\label{ass:Main:Ass:L1}
\item The length scale $\eps=\eps_n$ is positive, converges to 0, i.e. $0<\eps_n \to 0$ and satisfies the following lower bound: 
\begin{equation*} 
\lim_{n \to \infty} \frac{\log(n)}{n \eps_n^{d+4}} = 0.
\end{equation*}
\label{ass:Main:Ass:L2}
\end{enumerate}

\end{assumptions}

Assumption~\ref{ass:Main:Ass:L1} guarantees that (with probability one -- measured with $\bbP$) that there exists $N_1$ such that for all $n\geq N_1$ the graph $G_{n,\eps_n}=(\Omega_n,W_{n,\eps_n})$ is connected (see \cite{goel} or \cite{DBLP:books/ox/P2003}). We also note that the condition in the $d=2$ case in Assumption \ref{ass:Main:Ass:L1} can be tightened by removing the $\log$-term (using the techniques from \cite{Caroccia_2020,CALDER2022123}), i.e. $\lim_{n\to\infty} \frac{\log(n)}{n\eps_n^d} = 0$ for $d\geq 2$, so that $\eps_n$ can be chosen to be any sequence asymptotically greater than the connectivity radius for all $d\geq 2$.

\subsection{Theorem statements} 

We give our results for (classical) hypergraph learning~\eqref{eq:intro:clustering}/\eqref{eq:SSLObjectiveSum} in Subsection~\ref{subsec:main:HL}, and our results for higher order hypergraph learning~\eqref{eq:multiscale} in Subsection~\ref{subsec:main:HOHL}.

\subsubsection{Hypergraph learning} \label{subsec:main:HL}

We start by determining the Euler–Lagrange equations corresponding to \eqref{eq:main:hypergraphEnergy}. 
The result implies that the gradient of the energy decomposes into a sum of discrete operators, each tied to hyperedges of a given size.
The proof is given in Section~\ref{subsubsec:Proofs:HL:EL}.

\begin{proposition}[Discrete Euler-Lagrange equations of hypergraph learning] \label{prop:eulerLagrange}
    The energy $v\mapsto \sum_{k=1}^q \lambda_k \cE_{n,\eps}^{(k,p)}(v)$ 
    is minimized by $u$ if and only if $u$ satisfies
    $
    \sum_{k=1}^q \lambda_k \Delta_{n,\eps_n}^{(k,p)}(u) = 0.
    $
\end{proposition}

We now study the asymptotic behavior as \( n \to \infty \). The next result shows quantitative pointwise convergence of the discrete hypergraph operator to its continuum analogue.

\begin{theorem}[Pointwise consistency]\label{thm:pointwiseConsistency}
Assume that Assumptions \ref{ass:Main:Ass:S1}, \ref{ass:Main:Ass:M1}, \ref{ass:Main:Ass:M2}, \ref{ass:Main:Ass:D1} and \ref{ass:Main:Ass:L0} hold. Furthermore, assume that $\rho \in \Ck{2}(\Omega)$. Let $\Omega'$ be compactly contained in $\Omega$, $q \geq 1$, $\{\lambda_k\}_{k=1}^q \subset (0,\infty)$, $p \in \{2\}\cup [3,\infty)$, $\eps_n \leq \delta$ and $u\in \mathrm{C}^3$. Then, for $n$ large enough, we have that 
\begin{align}
    &\biggl\vert \left(\sum_{k=1}^q \lambda_k \Delta_{n,\eps_n}^{(k,p)}\right) (u)(x_{i_0}) -  \rho(x_{i_0}) \left( \sum_{k=1}^q  \lambda_k \Delta_\infty^{(k,p)}\right)(u)(x_{i_0}) \biggr\vert = \mathcal{O} \left( \delta \Vert u \Vert_{\mathrm{C}^3(\mathbb{R}^d)}^{p-1} \right) \notag
\end{align}
for $x_{i_0} \in \Omega_n \cap \Omega'$, with probability 
$1-Cn \exp \left( -C n\eps_n^{2(1+qd)}\delta^2 \right)$ where $C > 0$ is a constant independent of $n$ and $\delta$. 
\end{theorem}


The proof differs from the corresponding graph calculation because the
hypergraph operator is a multi-index statistic: for each anchor vertex, it sums over \(k\) additional vertices and contains a product kernel encoding all
pairwise interactions within the resulting \((k+1)\)-tuple. Consequently, the continuum limit cannot be obtained by applying the standard graph \(p\)-Laplacian expansion directly.

The argument begins by expanding the nonlinear term \(|t|^{p-2}t\). For
\(p=2\), this expansion is exact, while for \(p\geq 3\) one must control the
nonlinear Taylor remainder using the assumed smoothness of the test function.
After this expansion, the main task is to identify which terms vanish by
symmetry and to compute the constants generated by the product-type hyperedge
kernel. The details are given in
Section~\ref{subsubsec:Proofs:HL:Pointwise}.

\begin{remark}[Nondivergence-form representation of the limiting \(p\)-Laplacian]
\label{rem:divfree_pLap}
Corollary~\ref{cor:alternativeLaplacian} provides a nondivergence-form representation of the limiting weighted \(p\)-Laplacian induced by the product-type hypergraph weights. This representation is useful in the pointwise analysis because the discrete operator naturally expands into nondivergence-form terms involving \(\nabla u\), \(\nabla^2 u\), \(\rho\), and
\(\nabla \rho\). The identities in Lemmas~\ref{lem:identity1}
and~\ref{lem:identity2} then show that the constants produced by the multi-index product kernel recombine into the divergence-form operator
\[
    \Delta_\infty^{(k,p)}u
    =
    \frac{\sigma_{\eta}^{(k,p)}}{2\rho}
    \operatorname{div}\!\left(
    \rho^{k+1}\|\nabla u\|_2^{p-2}\nabla u
    \right).
\]
Thus, the pointwise limit extends the graph-based nondivergence-form calculation of~\cite{calderGameTheoretic} to the hypergraph product-kernel setting.
\end{remark}

Next, we precisely characterize the asymptotic consistency of hypergraph learning as a function of the length-scale $\eps_n$. We refer to Figure~\ref{fig:hypergraphWellIllPosed} for a visual summary of the result.

\begin{figure}[htbp]
  \centering
  \begin{subfigure}[b]{0.48\linewidth}
    \centering
    \includegraphics[width=\linewidth]{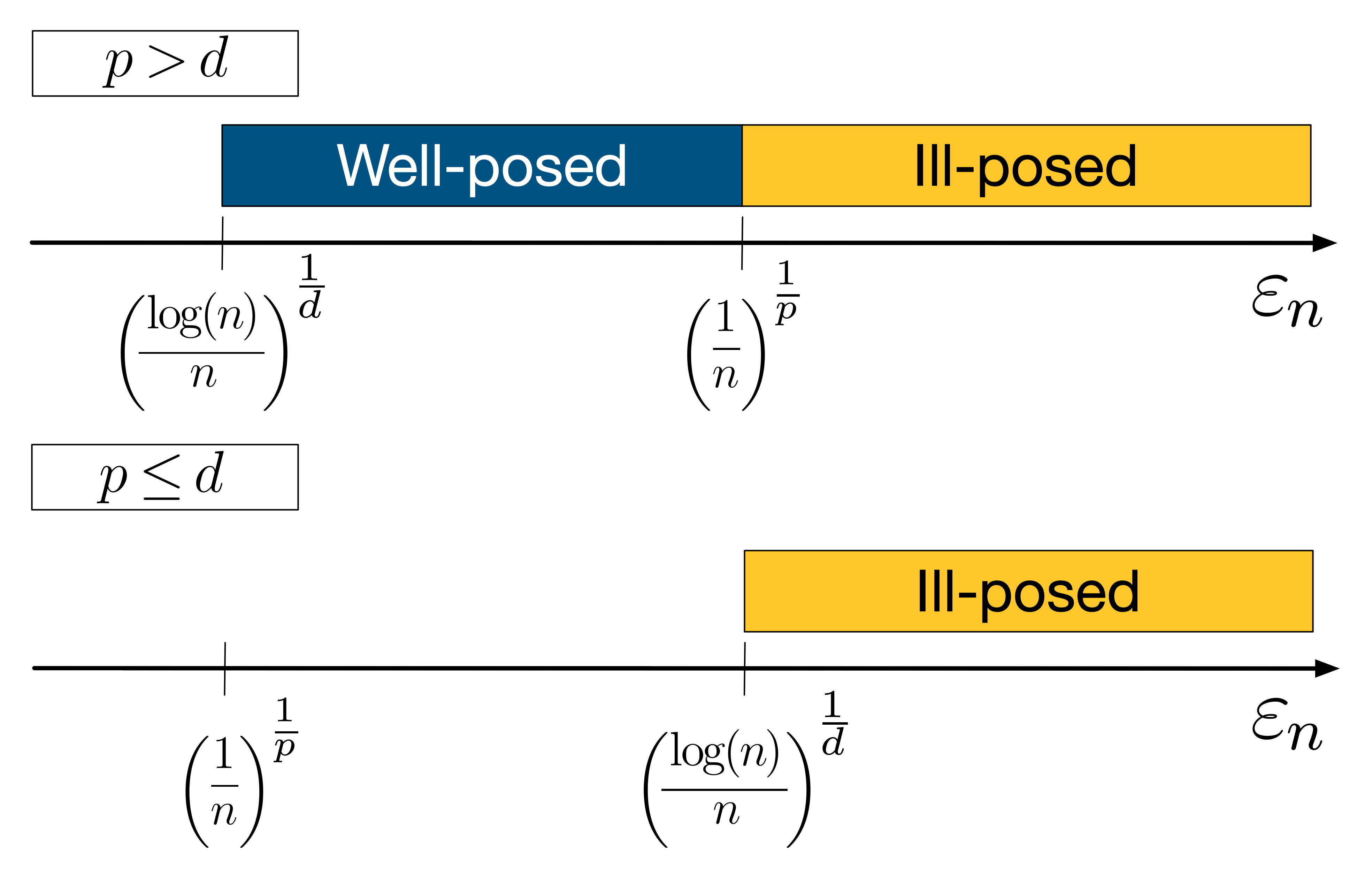}
    \caption{Hypergraph learning}
    \label{fig:hypergraphWellIllPosed}
  \end{subfigure}
  \hfill
  \begin{subfigure}[b]{0.48\linewidth}
    \centering
    \includegraphics[width=\linewidth]{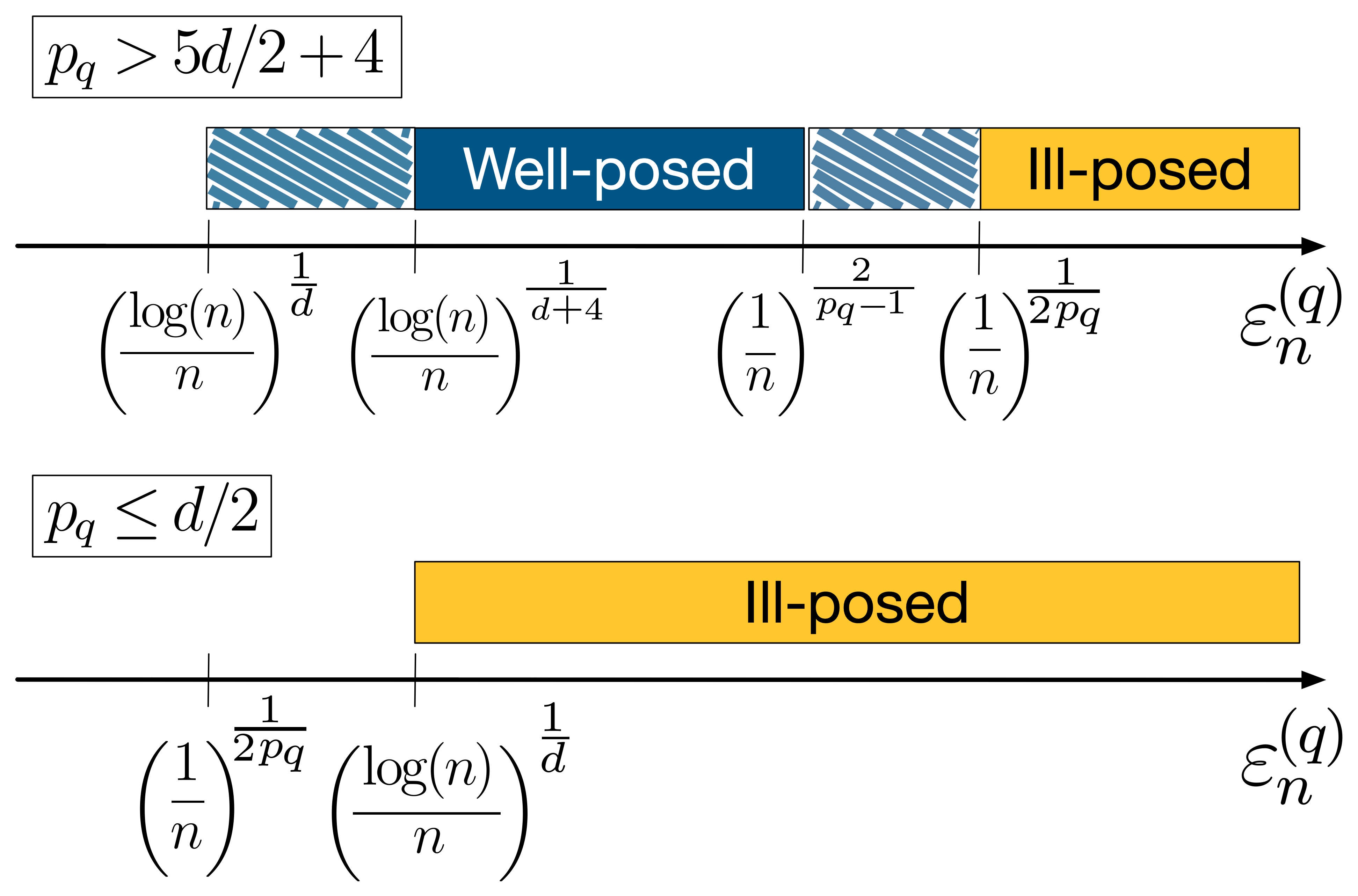}
    \caption{HOHL surrogate}
    \label{fig:wellIllPosedHOHL}
  \end{subfigure}
  \caption{Well- and Ill-posedness characterizations as a function of the length-scale. The striped regions are conjectured results (see \cite{weihs2023consistency}).}
\end{figure}

\begin{theorem}[Asymptotic consistency analysis of hypergraph learning]\label{thm:main}
    Assume that \ref{ass:Main:Ass:S1}, \ref{ass:Main:Ass:M1}, \ref{ass:Main:Ass:M2},  \ref{ass:Main:Ass:D1}, \ref{ass:Main:Ass:D2}, \ref{ass:Main:Ass:W2}, and \ref{ass:Main:Ass:L1} hold.
    Let $(\mu_n,u_n)$ be minimizers of $(\cS\cF)_{n,\eps_n}^{(q,p)}$.
    \begin{enumerate}
        \item (Well-posed case) Assume that $n\eps_n^{p} \to 0$. Then, $\bbP$-a.s., there exists a continuous function $u$ such that $(\mu_n,u_n) \to (\mu,u)$ in $\TLp{p}(\Omega)$ and for any $\Omega' \subset \subset \Omega$, $\max_{\{r \leq n \spaceBar x_r \in \Omega'\}} \vert u(x_r) - u_n(x_r) \vert \to 0$. In particular, $(\mu,u)$ is a minimizer of $(\cS\cF)_{\infty}^{(q,p)}$.
        \item (Ill-posed case) Assume that $n \eps_n^{p} \to \infty$. Then, $\bbP$-a.s., there exists $u \in \Wkp{1}{p}(\Omega)$ and a subsequence $\{n_r\}_{r=1}^\infty$ such that $(\mu_{n_r},u_{n_r}) \to (\mu,u)$ in $\TLp{p}(\Omega)$ and $(\mu,u)$ is a minimizer of $(\cS\cG)_{\infty}^{(q,p)}$. 
    \end{enumerate}
\end{theorem}

The limiting energy identified in Theorem \ref{thm:main} is 
\begin{equation} \label{eq:main:limitingEnergy}
\begin{aligned}
(\cS\cF)_{\infty}^{(q,p)}((\nu,v)) & = \sum_{k=1}^q \lambda_k \sigma_\eta^{(k,p)} \int_{\Omega} \Vert \nabla v(x_0) \Vert_2^p \,  \rho(x_0)^{k+1} \, \dd x_0 \\
 & =  \int_{\Omega} \Vert \nabla v(x_0) \Vert_2^p \, \l \sum_{k=1}^q \lambda_k \sigma_\eta^{(k,p)} \rho(x_0)^{k+1} \r \, \dd x_0.
\end{aligned}
\end{equation}
In particular, it only differs from the limiting energy of $p$-Laplacian learning 
\begin{equation*} 
\cE_\infty^{(1,p)}(v) = \int \Vert \nabla v(x_0) \Vert_2^p \rho(x_0)^2 \, \dd x_0
\end{equation*}
by a density reweighting (see Remark~\ref{rem:weights} for how the hyperedge weights encode locality and hence determine the density-dependent smoothing profile). 
Moreover, the resulting well-/ill-posedness condition on \(\varepsilon_n\) is exactly the same as in~\cite[Theorem~2.1]{Slepcev}, now in the hypergraph context.

Similarly to \cite[Remark 3.1]{weihs2023consistency}, the thresholds in Theorem \ref{thm:main} also imply that we recover an intuition stemming from Sobolev spaces. Indeed, in the continuum, our functions in $\Wkp{1}{p}(\Omega)$ must be at least continuous in order to satisfy pointwise constraints, i.e. be in the well-posed case: by Sobolev inequalities, this can only the case whenever $p > d$. Our results show that this condition is necessary but not sufficient as $\eps_n$ also has to satisfy an upper bound. We also note that in practice, the condition $p > d$ often leads to $p \geq 3$, which satisfies the requirements for pointwise convergence in Theorem~\ref{thm:pointwiseConsistency}.

For the ill-posed case, we note that minimizers of $(\cS\cG)_{\infty}^{(q,p)}$ are constants and therefore, for large $n$, we expect our discrete minimizers to be almost constant with spikes at the known labels (this is observed for $q=1$ in \cite{10.5555/2984093.2984243,elalaoui16}). The labelling problem relying on the thresholding of our minimizers is therefore rendered nonsensical, hence our denomination of ill-posed. The case  $p \leq d$ is also covered by our characterization of our ill-posed case (see \cite[Remark 3.3]{weihs2023consistency}), linking our results back to the Sobolev embedding intuition.

\begin{remark}[Graph-energy representation via induced pairwise weights]\label{rem:graph_representation}
The energy $\cE^{(k,p)}_{n,\eps}$ admits an exact representation as a (pairwise) graph energy.
Indeed, regrouping the sum by the pair $(i_0,i_1)$ yields
\begin{equation}\label{eq:graph_form_V}
\cE^{(k,p)}_{n,\eps}(v)
=\frac{1}{n^{2}\eps^{p}}
\sum_{i,j=1}^n \widetilde W^{(k)}_{ij}\,|v(x_j)-v(x_i)|^p,
\end{equation}
with induced pairwise weights
\begin{equation}\label{eq:Wtilde_V_def}
\widetilde W^{(k)}_{ij}
:=\frac{1}{n^{k-1}\eps^{kd}}
\sum_{i_2,\dots,i_k=1}^n
\eta_{\mathrm{p}}(x_i,x_j,x_{i_2},\dots,x_{i_k}).
\end{equation}
In this sense, the multiway interaction is encoded through an effective coupling matrix
$\widetilde W^{(k)}$ on the vertex set $\{x_1,\dots,x_n\}$.

However, despite the formal resemblance of \eqref{eq:graph_form_V} to standard graph
$p$-Dirichlet energies, the weights \eqref{eq:Wtilde_V_def} are not of the usual
kernel form $\eps^{-d}\eta(|x_i-x_j|/\eps)$ \cite{Slepcev} (nor a reweighting thereof as in \cite{calderSlepcev}): they are local statistics obtained by summing over $(k-1)$ additional indices, and thus
encode local clique/configuration information (and, in geometric sampling models, depend
nontrivially on the sampling density and on $\eps$). Consequently, existing graph
discrete-to-continuum results for kernel-type weights do not directly determine the
continuum behavior of $\cE^{(k,p)}_{n,\eps}$, and additional analysis as in Theorems \ref{thm:pointwiseConsistency} and \ref{thm:main} is required for
this density- and configuration-dependent weight structure.
\end{remark}

\begin{remark}[Clique expansion versus \(V\)-statistic graph representations]
\label{rem:clique_vs_V}
For a hypergraph in the strict combinatorial sense, hyperedges are unordered
subsets \(e\subset V\) with distinct vertices. In that setting, a
pairwise-aggregation energy as in~\cite{scholkopfHyper2006} has the form
\[
    \sum_{e\in E} w(e)
    \sum_{\{i,j\}\subset e}\phi\bigl(v(x_i)-v(x_j)\bigr),
\]
and it can be written exactly as a graph energy
\[
    \frac12\sum_{i,j}\widetilde W_{ij}
    \phi\bigl(v(x_i)-v(x_j)\bigr),
    \qquad
    \widetilde W_{ij}
    :=
    \sum_{e\in E:\ \{i,j\}\subset e} w(e),
\]
up to convention-dependent normalizations of the clique weights.

Our hypergraph energies are written instead in \(V\)-statistic form, which
sums over ordered tuples and includes degenerate tuples with repeated indices;
see Remark~\ref{rem:degenerateTuples}. Thus, strictly speaking, the energy is
not initially written as a sum over unordered \((k+1)\)-vertex hyperedges, and
the classical clique expansion does not apply verbatim to the indexing set.

Nevertheless, the two viewpoints are asymptotically consistent. If one
restricts the \(V\)-statistic to distinct indices and then symmetrizes over all
permutations of each fixed \((k+1)\)-vertex set \(S\), the energy becomes an
exact pairwise-aggregation hypergraph energy; see Remark~\ref{rem:symmetrize_pairwise}. Moreover,
Remark~\ref{rem:VUenergy} shows that the contribution of the degenerate tuples
vanishes under the standard scaling \(n\eps_n^d\to\infty\). Therefore, the
graph representation obtained by regrouping the \(V\)-statistic in
Remark~\ref{rem:graph_representation} may be interpreted as an asymptotic
clique-expansion representation of the corresponding distinct-vertex
hypergraph energy.
\end{remark}

\subsubsection{Higher Order Hypergraph Learning} \label{subsec:main:HOHL}


In contrast to the $p$-Dirichlet form of \eqref{eq:main:hypergraphEnergy}, the HOHL framework fundamentally alters the nature of interactions between vertices by introducing higher-order terms. For instance, the term \( v^\top (L^{(2)})^2 v \) incorporates nested finite differences of \( L^{(2)}(x_i) \) (see \cite{TutSpec}), which already represent aggregated information from multiple neighbors. Such terms approximate second-order derivatives (see Section \ref{sec:main:HOHL}) 
and in this way, HOHL leverages the hypergraph structure by simultaneously modifying the support of interactions (via multiscale decompositions) and the mechanism of interaction (through higher-order regularization).       

We obtain the following asymptotic consistency result for the HOHL surrogate \eqref{eq:multiscale}. Figure~\ref{fig:wellIllPosedHOHL} provides a visual summary of the result.  

\begin{theorem}[Asymptotic consistency analysis of HOHL surrogate]\label{thm:main:HOHL}
    Assume that \ref{ass:Main:Ass:S2}, \ref{ass:Main:Ass:M1}, \ref{ass:Main:Ass:M2}, \ref{ass:Main:Ass:D1}, \ref{ass:Main:Ass:D2} and \ref{ass:Main:Ass:W2} hold. Let $q \geq 1$, $P = \{p_k\}_{k=1}^q \subseteq (0,\infty)$ with $p_1 \leq \cdots \leq p_q$ and $ E_n = \{\eps_n^{(k)}\}_{k=1}^q$ with $\eps_n^{(1)} > \cdots > \eps_n^{(q)}$. Let $(\mu_n,u_n)$ be minimizers of $(\cS\cJ)_{n,E_n}^{(q,P)}$. Assume that $\rho \in \Ck{\infty}$ and that $\eps_n^{(1)} \to 0$.
    \begin{enumerate}
        \item (Well-posed case) Assume that $\eps_n^{(q)}$ satisfies \ref{ass:Main:Ass:L2}, that $n \cdot (\eps_n^{(q)})^{p_q/2 - 1/2}$ is bounded and that $p_q > \frac{5}{2}d + 4$. Then, $\bbP$-a.s., there exists a continuous function $u$ such that $(\mu_n,u_n) \to (\mu,u)$ in $\TLp{2}(\Omega)$ and $\max_{\{r \leq n\}} \vert u(x_r) - u_n(x_r) \vert \to 0$. In particular, $(\mu,u)$ is a minimizer of $(\cS\cJ)_{\infty}^{(q,P)}$.
        \item (Ill-posed case) Assume that $\eps_n^{(q)}$ satisfies \ref{ass:Main:Ass:L1} as well as $n(\eps_n^{(q)})^{2p_q} \to \infty$. Furthermore, assume that $\sup_{n \geq 1} \Vert u_n \Vert_{\Lp{2}(\mu_n)}$ is bounded. Then, $\bbP$-a.s., there exists $u$ and a subsequence $\{n_r\}_{r=1}^\infty$ such that $(\mu_{n_r},u_{n_r}) \to (\mu,u)$ in $\TLp{2}(\Omega)$ and $(\mu,u)$ is a minimizer of $(\cS\cK)_{\infty}^{(q,P)}$.
    \end{enumerate}    
\end{theorem}

Similarly to Theorem \ref{thm:main}, this result shows how the choice of scale and regularity governs the transition between expressive interpolation and trivial smoothing. 
Notably, we remark that the characterization mostly depends on the parameters of the finest scale, i.e. $\eps_n^{(q)}$ and $p_q$.

In contrast to standard hypergraph learning, which converges to a $\Wkp{1}{p}$ seminorm, as explained in Section \ref{sec:main:HOHL}, the continuum limiting energy identified through Theorem \ref{thm:main:HOHL} indicates that the HOHL surrogate converges to a $\Wkp{p_q}{2}$ seminorm. 
This underscores the distinct regularity structure induced by our higher-order formulation.

Furthermore, the same Sobolev intuition developed for hypergraph learning prevails for the HOHL surrogate. In particular, our result implies that $p_q > d/2$ --- or equivalently that $\Wkp{p_q}{2}$ is embedded in $\Ck{0}$ --- is necessary for well-posedness. Similarly $p_q \leq d/2$ also partly characterizes the ill-posed case.

Finally, these scalings are not believed to be sharp, since they inherit the non-sharpness of the underlying fractional Laplacian regularization. We refer to \cite{weihs2023consistency} for a broader discussion of potential improvements and related numerical experiments.

\section{Proofs} \label{sec:proofs}

\subsection{Euler-Lagrange equations of hypergraph learning} \label{subsubsec:Proofs:HL:EL}

In this section, we present the proof for the derivation of the Euler-Lagrange equations of hypergraph learning.

\begin{proof}[Proof of Proposition \ref{prop:eulerLagrange}]
We proceed as follows:
\begin{align}
    \frac{d}{dt} \cE_{n,\eps}^{(k,p)}(u + tv) \lfloor_{t=0} & = \frac{p}{n^{k+1}\eps^{p + kd}} \hspace{-3mm} \sum_{i_0,\cdots, i_k = 1}^{n} \hspace{-3mm} \etaP(x_{i_0},\dots,x_{i_k})\left\vert u(x_{i_1}) - u(x_{i_0}) \right\vert^{p-2} \notag \\
    & \qquad \qquad \times (u(x_{i_1}) - u(x_{i_0})) (v(x_{i_1}) - v(x_{i_0})) \notag \\
    &=: \sum_{i_0,i_1=1}^n g(x_{i_0},x_{i_1}) (v(x_{i_1}) - v(x_{i_0})) \notag \\
    &= \sum_{i_0,i_1=1}^n g(x_{i_0},x_{i_1}) v(x_{i_1}) - \sum_{i_0,i_1=1}^n g(x_{i_0},x_{i_1}) v(x_{i_0}) \notag \\
    &= \sum_{i_0,i_1=1}^n g(x_{i_1},x_{i_0}) v(x_{i_0}) - \sum_{i_0,i_1=1}^n g(x_{i_0},x_{i_1}) v(x_{i_0}) \notag \\
    &= -2 \sum_{i_0,i_1=1}^n g(x_{i_0},x_{i_1}) v(x_{i_0}) \label{eq:eulerLagrange:sign} \\
    &=  \langle -2p \Delta_{n,\eps}^{(k,p)}(u), v \rangle_{\Lp{2}(\mu_n)} \notag 
\end{align}
where we used the fact that the function $f(x,y) = \etaP(x,y,x_{i_2},\dots,x_{i_k})$ satisfies $f(x,y) = f(y,x)$ for all fixed $x_{i_2},\dots,x_{i_k}$ implying that $g(x,y) = - g(y,x)$ for \eqref{eq:eulerLagrange:sign}. We deduce that $u$ minimizing $\sum_{k=1}^q \lambda_k \cE_{n,\eps}^{(k,p)}$ must satisfy 
\(
\sum_{k=1}^q \lambda_k \Delta_{n,\eps_n}^{(k,p)}(u) = 0. 
\) 
Conversely, by convexity any $u$ satisfying $\sum_{k=1}^q \lambda_k \Delta_{n,\eps_n}^{(k,p)}(u) = 0$ must be a minimizer.
\end{proof}

\subsection{Pointwise convergence of hypergraph learning} \label{subsubsec:Proofs:HL:Pointwise}

In this section, we present the proofs related to Theorem \ref{thm:pointwiseConsistency}.

\subsubsection{Equivalent representation of the continuum Laplacian}

First, we prove an equivalent representation of the continuum Laplacian $\Delta_{\infty}^{(k,p)}$. The latter will appear as the continuum limit in Theorem \ref{thm:pointwiseConsistency}. We start by introducing the following constants: 
\[
\sigma_{\eta}^{(k,p,1)} = \int_{(\bbR^d)^k} \hspace{-3mm}  \etaPTilde(\tilde{z}_1,\dots,\tilde{z}_k)  \vert (\tilde{z}_1)_d \vert^{p-2} (\tilde{z}_1)_1^2\, \dd \tilde{z}_k \cdots \dd \tilde{z}_1,\\
\]
and 
\[
\sigma_{\eta}^{(k,p,2)} = \int_{(\bbR^d)^k} \etaPTilde(\tilde{z}_1,\dots,\tilde{z}_k) \vert(\tilde{z}_1)_d\vert^{p-2}(\tilde{z}_1)_d (\tilde{z}_2)_d \, \dd \tilde{z}_k \cdots \dd \tilde{z}_1.
\]
The key idea for the following computations is to consider integrals of the form \[
\int_{(\bbR^d)^k} \widetilde\eta_p(z_1,\dots,z_k) g(z_1,\dots,z_k) \, \dd z_1 \cdots \dd z_k 
\]
as an expectation with respect to the measure $\bbQ$ defined through the density
\[
f(z_1,\dots,z_k)
= \frac{1}{\mathcal Z}\,
 \widetilde\eta_p(z_1,\dots,z_k),
\qquad
\mathcal Z := \int_{(\mathbb R^d)^k} \widetilde\eta_p(z)\, \dd z.
\]
By considering a random vector $(Z_1,\dots,Z_k) \sim \bbQ$, we obtain that
\[
\mathcal{Z} \int_{(\mathbb R^d)^k} \frac{\widetilde\eta_p(z_1,\dots,z_k)}{\mathcal{Z}} g(z_1,\dots,z_k)\,dz_1\cdots dz_k
= \mathcal Z\,\mathbb E_{\mathbb Q}[g(Z_1,\dots,Z_k)].
\]
The structure of $\widetilde\eta_p$ implies strong symmetry properties of the
law of $(Z_1,\dots,Z_k)$, specifically invariance under simultaneous
rotations of all coordinates and under affine reflections fixing $Z_1$. Exploiting these
invariances via conditional expectations and multivariate symmetry arguments, we obtain the identities between $\sigma_{\eta}^{(k,p)}$, $\sigma_{\eta}^{(k,p,1)}$ and $\sigma_{\eta}^{(k,p,2)}$ stated below.

\begin{lemma}[Constant identity I] \label{lem:identity1}
Assume that Assumption \eqref{ass:Main:Ass:W2} holds. Let $d\ge 2$. 
Then,
\[
\sigma_{\eta}^{(k,p)} = (p-1)\,\sigma_{\eta}^{(k,p,1)}.
\]
\end{lemma}

\begin{proof}
For $(Z_1,\dots,Z_k) \sim \bbQ$, we write $Z_1=(X_1,\dots,X_d)^\top$. Then, 
\[
\sigma_{\eta}^{(k,p)} = \mathcal Z\,\mathbb E[|X_d|^p],
\qquad \text{and} \qquad
\sigma_{\eta}^{(k,p,1)} = \mathcal Z\,\mathbb E[|X_d|^{p-2}X_1^2]
\]
where the expectation is taken with respect to $\bbQ$.

By Lemma \ref{lem:radialMarginal}, the marginal law of $Z_1$ is rotation-invariant.
By \cite[Theorem 4.1.2]{Bryc1995}, since the distribution of $Z_1$ is
rotation invariant, we may write
\(
Z_1 \overset{d}{=} R U,
\)
where $R \overset{d}{=} \|Z_1\| \ge 0$, $U \in \mathbb S^{d-1}$ is uniformly distributed
on the unit sphere and $R$ and $U$ are independent. 
Writing $U=(U_1,\dots,U_d)$, we have
\[
|X_d|^p \overset{d}{=} R^p|U_d|^p
\qquad \text{and} \qquad
|X_d|^{p-2}X_1^2 \overset{d}{=} R^p |U_d|^{p-2}U_1^2.
\]
Therefore, using the independence of $R$ and $U$, we obtain that
\[
\frac{\mathbb{E}[|X_d|^p]}{\mathbb{E}[|X_d|^{p-2}X_1^2]}
= \frac{\bbE(R^p) \mathbb{E}[|U_d|^p]}{\bbE(R^p) \mathbb{E}[|U_d|^{p-2}U_1^2]} = \frac{\mathbb{E}[|U_d|^p]}{\mathbb{E}[|U_d|^{p-2}U_1^2]}.
\]
Let $Y_i := U_i^2$.  For $U$ uniform on $\mathbb{S}^{d-1}$, by the proof of \cite[Theorem 3.3]{FangKotzNg1990} (which shows that $(U_1,\dots,U_d) \overset{d}{=} x/\Vert x \Vert$ where $x \sim \cN(0,\Id_{d \times d})$) and \cite[Section 1.4]{FangKotzNg1990} (which shows that $(x_1^2/\Vert x \Vert^2,\dots,x_d^2/\Vert x \Vert^2)$ is Dirichlet-distributed with parameters $(\frac{1}{2},\dots,\frac{1}{2})$), the vector
$(Y_1,\dots,Y_d)$ is Dirichlet-distributed with parameters $(\frac{1}{2},\dots,\frac{1}{2})$. We can then apply the moment formula for Dirichlet distributions \cite[Section 27.6]{KotzBalakrishnanJohnson2000}: for $Y \sim \mathrm{Dirichlet}(\alpha_1,\dots,\alpha_n)$ and $\beta_i >0$, \[
\bbE\ls \prod_{i=1}^n Y_i^{\beta_i} \rs = \frac{\Gamma\l \sum_{i=1}^n \alpha_i \r}{\Gamma\l \sum_{i=1}^n \alpha_i + \beta_i \r} \prod_{i=1}^n \frac{\Gamma(\alpha_i + \beta_i)}{\Gamma(\alpha_i)}.
\]
This yields:
\[
\mathbb E[|U_d|^p] = \bbE[Y_d^{p/2}]
= \frac{\Gamma\l \frac{d}{2} \r}{\Gamma\l \frac{d+p}{2}\r} \frac{\Gamma\l \frac{1}{2}\r^{d-1}\Gamma\l \frac{p+1}{2} \r}{\Gamma\l \frac{1}{2}\r^{d}}= 
\frac{\Gamma\l\frac{p+1}{2}\r \Gamma(\frac{d}{2})}
       {\Gamma\l\frac{1}{2}\r\Gamma\l\frac{d+p}{2}\r}
\]
and
\[
\mathbb E[|U_d|^{p-2}U_1^2] = \bbE[Y_d^{(p-2)/2} Y_1] 
= \frac{\Gamma\l \frac{d}{2} \r}{\Gamma\l \frac{d+p}{2}\r} \frac{\Gamma\l \frac{p-1}{2} \r\Gamma\l \frac{3}{2} \r}{\Gamma\l \frac{1}{2}\r^{2}}.
\]
Taking the ratio, we obtain 
\begin{align*}
   \frac{\mathbb E[|U_d|^p]}{\mathbb E[|U_d|^{p-2}U_1^2]} =  \frac{\Gamma\l \frac{p+1}{2} \r \Gamma\l \frac{1}{2} \r}{\Gamma\l \frac{p-1}{2} \r \Gamma\l \frac{3}{2} \r} = 
   \frac{\frac{p-1}{2} \Gamma\l\frac{p-1}{2}\r \Gamma\l\frac12\r}{\Gamma\l\frac{p-1}{2}\r \frac12 \Gamma\l\frac12\r} = p-1
\end{align*} 
where we used the identity $\Gamma(t+1)=t\Gamma(t)$ for the middle equality. We conclude that
\(
\sigma_{\eta}^{(k,p)} = (p-1)\,\sigma_{\eta}^{(k,p,1)}.
\)
\end{proof}

The next lemma is proven in Section \ref{sec:equivalentSupplement} and is essential for Lemma \ref{lem:identity2}.

\begin{lemma}[Reflections] \label{lem:reflections}
Let $z_1 \in \bbR^d$ be non-zero. Define $v:=\frac{z_1}{\|z_1\|}$, $m:=\frac{z_1}{2}$, the function $R_{z_1}:\mathbb R^d\to\mathbb R^d$ by
\(
R_{z_1}(y) := y - 2(y\cdot v)\,v,
\)
and the function $S_{z_1}:\mathbb R^d\to\mathbb R^d$ by
\(
S_{z_1}(y) := m + R_{z_1}(y-m).
\)
Then:
    \begin{enumerate}
        \item $R_{z_1}$ is the reflection across the hyperplane $\{y : y\cdot v = 0\}$ and an isometry;
        \item $S_{z_1}$ is the reflection across the affine hyperplane $H_{z_1} := \{y\in\mathbb R^d : (y-m)\cdot v = 0\}$; 
        \item $S_{z_1}$ is an isometry;
        \item $S_{z_1}(0) = z_1$ and $S_{z_1}(z_1) = 0$;
        \item $\Vert S_{z_1}(y) \Vert = \Vert z_1 - y \Vert$ and $\Vert S_{z_1}(y) - z_1 \Vert = \Vert y \Vert$.
    \end{enumerate}
\end{lemma}

\begin{lemma}[Constant identity II] \label{lem:identity2}
Assume that Assumption \ref{ass:Main:Ass:W2} holds. Then,
\[
\sigma_\eta^{(k,p,2)} 
= \frac12\,\sigma_\eta^{(k,p)}.
\]
\end{lemma}

\begin{proof}

For $(Z_1,\dots,Z_k) \sim \bbQ$, we write $Z_1=(X_1,\dots,X_d)^\top$ and $Y:=(Z_2)_d$. Then, 
\(
\sigma_{\eta}^{(k,p)} = \mathcal Z\,\mathbb E[|X_d|^p]\)
and \(\sigma_\eta^{(k,p,2)}=\mathcal Z\,\mathbb E[|X_d|^{p-2}X_dY]
\)
where the expectation is taken with respect to $\bbQ$.

As shown in the proof of Lemma \ref{lem:radialMarginal}, since $\widetilde\eta_p$ depends only on norms and pairwise distances,
\(
\widetilde\eta_p(Qz_1,\dots,Qz_k)=\widetilde\eta_p(z_1,\dots,z_k)
\)
for all $Q \in O(d)$, where $O(d)$ denotes the group of orthogonal matrices in $\bbR^d$.
We now fix $z_1\neq 0$ and let
\(
G_{z_1}:=\{Q\in O(d): Qz_1=z_1\}
\)
be the subgroup of orthogonal transformations fixing $z_1$.
For $Q\in G_{z_1}$, we therefore have
\begin{equation} \label{eq:rotationInvarianceSubgroup}
    \widetilde\eta_p(z_1,Qz_2,\dots,Qz_k)
=\widetilde\eta_p(z_1,z_2,\dots,z_k).
\end{equation}

The marginal density of $Z_1$ at $z_1$ is
\(
f_{Z_1}(z_1)
= \int_{(\mathbb R^d)^{k-1}} f(z_1,z_2,\dots,z_k)\,\dd z_2\cdots \dd z_k.
\)
For $f_{Z_1}(z_1)>0$, the conditional density of
$(Z_2,\dots,Z_k)$ given $Z_1=z_1$ is
\[
f_{Z_2,\dots,Z_k\mid Z_1}(z_2,\dots,z_k\mid z_1)
= \frac{f(z_1,z_2,\dots,z_k)}{f_{Z_1}(z_1)}.
\]
For a fixed $Q\in G_{z_1}$, we consider the conditional density at the point
$(Qz_2,\dots,Qz_k)$:
\begin{align}
f_{Z_2,\dots,Z_k\mid Z_1}(Qz_2,\dots,Qz_k\mid z_1)
&= \frac{f(z_1,Qz_2,\dots,Qz_k)}{f_{Z_1}(z_1)} \notag \\
&= \frac{\widetilde\eta_p(z_1,Qz_2,\dots,Qz_k)}{\mathcal Z f_{Z_1}(z_1)} \notag \\
&= f_{Z_2,\dots,Z_k\mid Z_1}(z_2,\dots,z_k\mid z_1) \label{eq:rotationInvarianceSubgroup:application}
\end{align}
where we used \eqref{eq:rotationInvarianceSubgroup} for \eqref{eq:rotationInvarianceSubgroup:application}. Thus, for any measurable set $A\subseteq(\mathbb R^d)^{k-1}$,
\begin{align}
\mathbb P\bigl((Z_2,\dots,Z_k)\in A \mid Z_1=z_1\bigr)
&= \int_A f_{Z_2,\dots,Z_k\mid Z_1}(z_2,\dots,z_k\mid z_1)\,\dd z_2\cdots \dd z_k \notag \\
&= \int_A f_{Z_2,\dots,Z_k\mid Z_1}(Qz_2,\dots,Qz_k\mid z_1)\,\dd z_2\cdots  \dd z_k \notag \\
&= \int_{QA} f_{Z_2,\dots,Z_k\mid Z_1}(w_2,\dots,w_k\mid z_1)\,\dd w_2\cdots \dd w_k \label{eq:changeOfVariables} \\
&= \mathbb P\bigl((Z_2,\dots,Z_k)\in QA \mid Z_1=z_1\bigr) \label{eq:rotationInvariance}
\end{align}
where we used the change of variables $w_j = Qz_j$ for $2 \leq j \leq k$ and the fact that $\vert \det Q \vert = 1$ for \eqref{eq:changeOfVariables}.
Equivalently, we obtain that
\[
(Z_2,\dots,Z_k)\mid(Z_1=z_1)
\;\overset{d}{=}\;
(QZ_2,\dots,QZ_k)\mid(Z_1=z_1).
\]
By picking $A = B \times \bbR^d \times \dots \times \bbR^d$ for $B \subseteq \bbR^d$ a measurable set, we obtain
\begin{align}
     \mathbb P\bigl(Z_2 \in B \mid Z_1=z_1\bigr) &= \mathbb P\bigl((Z_2,\dots,Z_k)\in A \mid Z_1=z_1\bigr) \notag \\
     &= \mathbb P\bigl((Z_2,\dots,Z_k)\in QA \mid Z_1=z_1\bigr) \label{eq:rotationInvariance:application} \\
     &= \mathbb P\bigl(Z_2 \in QB \mid Z_1=z_1\bigr) \notag 
\end{align}
where we used \eqref{eq:rotationInvariance} for \eqref{eq:rotationInvariance:application}, which implies that
\begin{equation} \label{eq:marginalInvariance}
Z_2\mid(Z_1=z_1)\;\overset{d}{=}\;QZ_2\mid(Z_1=z_1)
\end{equation}
for all $Q\in G_{z_1}.$
Taking expectations gives
\begin{equation} \label{eq:rotationExpectation}
    \mathbb E[Z_2\mid Z_1=z_1]
= Q\,\mathbb E[Z_2\mid Z_1=z_1].
\end{equation}
Specifically, this means that $\mathbb E[Z_2\mid Z_1=z_1]$ is a vector which is fixed by all $Q\in G_{z_1}.$ We decompose $w := \mathbb E[Z_2\mid Z_1=z_1] = \l w \cdot \frac{z_1}{\Vert z_1 \Vert}\r \frac{z_1}{\Vert z_1 \Vert} + \l w -\l w \cdot \frac{z_1}{\Vert z_1 \Vert}\r\frac{z_1}{\Vert z_1 \Vert} \r :=  \l w \cdot \frac{z_1}{\Vert z_1 \Vert}\r \frac{z_1}{\Vert z_1 \Vert} + w_\perp$. Let $Q \in G_{z_1}$, then we have
\begin{align}
w & = Qw 
= \l w \cdot \frac{z_1}{\Vert z_1 \Vert}\r \frac{Qz_1}{\Vert z_1 \Vert} + Q w_\perp 
= \l w \cdot \frac{z_1}{\Vert z_1 \Vert}\r \frac{z_1}{\Vert z_1 \Vert} + Q w_\perp  
 = w - w_\perp + Qw_\perp \notag
\end{align}
where we used~\eqref{eq:rotationExpectation} and the fact that $Q z_1 = z_1$ by assumption. This implies that $w_\perp = Q w_\perp$ and, since the action of $G_{z_1}$ on the orthogonal complement $z_1^\perp$ is the
full orthogonal group $O(d-1)$ \cite[Example 21.19]{Lee2012}, this implies that $w_\perp = 0$. We therefore conclude that 
\begin{equation} \label{eq:scalar}
    \mathbb E[Z_2\mid Z_1=z_1] = \alpha(z_1)\,z_1
\end{equation}
for some scalar $\alpha(z_1)$.

We recall that
\(
\etaPTilde(z_1,\dots,z_k) = \left[ \prod_{s=1}^k \eta(\vert z_s\vert) \right]  \left[ \prod_{j=2}^k \prod_{r=1}^{j-1} \eta(\vert z_j - z_r \vert) \right]
\)
and let $S_{z_1}$ be the map defined in Lemma \ref{lem:reflections}.
Now, we estimate as follows: 
\begin{align}
    &\etaPTilde(z_1,S_{z_1}(z_2),\dots,S_{z_1}(z_k)) \notag \\
    &= \eta(\vert z_1 \vert) \left[ \prod_{s=2}^k \eta(\vert S_{z_1}(z_s)\vert) \right] \ls \prod_{j=2}^k \eta(\vert S_{z_1}(z_j) - z_1 \vert) \rs \notag \\
    & \qquad \qquad \qquad \times \left[ \prod_{j=2}^k \prod_{r=2}^{j-1} \eta(\vert S_{z_1}(z_j) - S_{z_1}(z_r) \vert) \right] \notag \\
    &=\eta(\vert z_1 \vert) \left[ \prod_{s=2}^k \eta(\vert S_{z_1}(z_s)\vert) \right] \ls \prod_{j=2}^k \eta(\vert S_{z_1}(z_j) - z_1 \vert) \rs \left[ \prod_{j=2}^k \prod_{r=2}^{j-1} \eta(\vert z_j - z_r \vert) \right] \label{eq:isometryS:application} \\
    &=\eta(\vert z_1 \vert) \left[ \prod_{s=2}^k \eta(\vert z_s - z_1 \vert) \right] \ls \prod_{j=2}^k \eta(\vert z_j \vert) \rs \left[ \prod_{j=2}^k \prod_{r=2}^{j-1} \eta(\vert z_j - z_r \vert) \right] \label{eq:metricIdentities} \\
    &= \etaPTilde(z_1,z_2,\dots,z_k) \label{eq:metricIdentities2}
\end{align}
where we used part 3 of Lemma \ref{lem:reflections} for \eqref{eq:isometryS:application} and part 4 of Lemma \ref{lem:reflections} for \eqref{eq:metricIdentities}. This directly implies:
\begin{align}
f_{Z_2,\dots,Z_k\mid Z_1}(S_{z_1}(z_2),\dots,S_{z_1}(z_k)\mid z_1)
&= \frac{f(z_1,S_{z_1}(z_2),\dots,S_{z_1}(z_k))}{f_{Z_1}(z_1)} \notag \\
&= \frac{\etaPTilde(z_1,S_{z_1}(z_2),\dots,S_{z_1}(z_k))}{\mathcal Z f_{Z_1}(z_1)} \notag \\
&= f_{Z_2,\dots,Z_k\mid Z_1}(z_2,\dots,z_k\mid z_1) \label{eq:metricIdentites:application}
\end{align}
where we used \eqref{eq:metricIdentities2} for \eqref{eq:metricIdentites:application}.
Therefore, for a measurable set $A\subseteq(\mathbb R^d)^{k-1}$,
\begin{align}
\mathbb P\!\left( (Z_2,\dots,Z_k)\in A \,\middle|\, Z_1=z_1\right)
&= \int_A f_{Z_2,\dots,Z_k\mid Z_1}(z_2,\dots,z_k\mid z_1)\,
    \dd z_2\cdots \dd z_k \notag \\
&= \int_A f_{Z_2,\dots,Z_k\mid Z_1}(S_{z_1}(z_2),\dots,S_{z_1}(z_k)\mid z_1)\,
    \dd z_2\cdots \dd z_k \label{eq:invarianceS1} \\
&= \int_{S_{z_1}(A)}
    f_{Z_2,\dots,Z_k\mid Z_1}(w_2,\dots,w_k\mid z_1)\,
    \dd w_2\cdots \dd w_k \label{eq:invarianceS1:2} \\
    &= \mathbb P\!\left( (Z_2,\dots,Z_k)\in S_{z_1}(A)
\,\middle|\, Z_1=z_1\right) \notag
\end{align}
where we used \eqref{eq:metricIdentites:application} for \eqref{eq:invarianceS1}, the change of variables $w_j = S_{z_1}(z_j)$ and the fact
that $|\det DS_{z_1}| = 1$ (since $S_{z_1}$ is an isometry by part 3 of Lemma \ref{lem:reflections}) for \eqref{eq:invarianceS1:2}.
We conclude that 
\(
(Z_2,\dots,Z_k)\mid(Z_1=z_1)
\overset{d}{=}
(S_{z_1}(Z_2),\dots,S_{z_1}(Z_k))\mid(Z_1=z_1),
\)
and taking marginals, analogously to how we derived \eqref{eq:marginalInvariance},
\begin{equation} \label{eq:marginalInvariance:2}
    Z_2\mid(Z_1=z_1)
\overset{d}{=}
S_{z_1}(Z_2)\mid(Z_1=z_1).
\end{equation}

We now estimate as follows (and using the notation of Lemma \ref{lem:reflections}): \begin{align}
    \alpha(z_1) z_1 &=\mathbb E[Z_2\mid Z_1=z_1] \label{eq:0}\\
    &= \bbE[S_{z_1}(Z_2)\mid Z_1=z_1] \label{eq:1} \\
    &=\bbE[m + R_{z_1}(Z_2 - m)\mid Z_1=z_1] \notag \\
    &= m + R_{z_1}(\bbE[Z_2 \mid Z_1 = z_1] - m) \label{eq:2} \\
    &= m + R_{z_1}\l \alpha(z_1)z_1 - \frac{z_1}{2} \r \label{eq:3} \\
    &= m + \l \alpha(z_1) - \frac{1}{2} \r R_{z_1}(z_1) \label{eq:4} \\
    &= \frac{z_1}{2} + \l \frac{1}{2} - \alpha(z_1) \r z_1 \label{eq:5} \\
    &= (1 - \alpha(z_1)) z_1 \notag. 
\end{align}
where we used \eqref{eq:scalar} for \eqref{eq:0}, \eqref{eq:marginalInvariance:2} for \eqref{eq:1}, part 1 of Lemma \ref{lem:reflections} for \eqref{eq:2}, \eqref{eq:scalar} for \eqref{eq:3}, part 1 of Lemma \ref{lem:reflections} for \eqref{eq:4} and the fact that $R_{z_1}(z_1) = -z_1$ for \eqref{eq:5}. Since $z_1 \neq 0$, we deduce that $\alpha(z_1) = \frac{1}{2}$ and \(
\mathbb E[Z_2\mid Z_1=z_1]=\frac{z_1}{2}
\)
or equivalently \begin{equation} \label{eq:conditional_Y}
    \bbE[(Z_2)_d \mid Z_1]= \bbE[ Y\mid Z_1] = \frac{(Z_1)_d}{2} = \frac{X_d}{2}.
\end{equation}

We conclude with the following computation: 
\begin{align}
    \mathbb E[|X_d|^{p-2}X_d Y] &= \bbE[\bbE[|X_d|^{p-2}X_dY \mid Z_1 ]] \label{eq:6} \\
    &= \bbE[|X_d|^{p-2}X_d \bbE[Y \mid Z_1 ]] \label{eq:7}\\
    &= \frac{1}{2} \bbE[|X_d|^{p}] \notag
\end{align}
where we use the tower property of conditional expectation for \eqref{eq:6}, the fact that $X_d$ is measurable with respect to the $\sigma$-algebra induced by $Z_1$ and \eqref{eq:conditional_Y} for \eqref{eq:7}.
From this, we directly obtain
\(
\sigma_\eta^{(k,p,2)}
=\mathcal Z\,\mathbb E[|X_d|^{p-2}X_dY]
=\frac12\,\mathcal Z\,\mathbb E[|X_d|^p]
=\frac12\,\sigma_\eta^{(k,p)}. 
\)
\end{proof}

\begin{corollary}[$p$-Laplacian representation] \label{cor:alternativeLaplacian}
    Assume that assumption \ref{ass:Main:Ass:W2} holds let \(d\ge 2\). Then, 
    \begin{align}
    &\Delta_\infty^{(k,p)}(u)(x) = \biggl( \Vert \nabla u(x) \Vert_2^{p-2} \rho(x)^{k} \nabla \rho(x)\cdot \nabla u(x) \times \frac{2(\sigma_{\eta}^{(k,p)} + (k-1) \sigma_\eta^{(k,p,2)})}{(p-1)\sigma_\eta^{(k,p,1)}} \notag \\
    &+ \rho(x)^{k+1} \Vert \nabla u(x) \Vert_2^{p-2} \biggl[ \Delta u(x) + \l\frac{\sigma_{\eta}^{(k,p)}}{\sigma_{\eta}^{(k,p,1)}} - 1 \r \frac{\nabla u(x)^\top \nabla^2 u(x) \nabla u(x)}{\Vert \nabla u(x)\Vert_2^2}  \biggr] \biggr) \notag\\
    &\times \frac{\sigma_\eta^{(k,p,1)}(p-1)}{2\rho(x)} \notag 
\end{align}
where $\Delta$ denotes the regular continuum Laplacian operator. 
\end{corollary}

\begin{proof}
By Lemma \ref{lem:identity1}, we have that \begin{equation} \label{eq:computation1}
\frac{\sigma_{\eta}^{(k,p)}}{\sigma_{\eta}^{(k,p,1)}} - 1 = p-2.
\end{equation}
Similarly, we have \begin{align}
    \frac{2(\sigma_{\eta}^{(k,p)} + (k-1) \sigma_\eta^{(k,p,2)})}{(p-1)\sigma_\eta^{(k,p,1)}} &= \frac{2}{p-1} \ls \frac{\sigma_{\eta}^{(k,p)}}{\sigma_\eta^{(k,p,1)}} + (k-1)\frac{\sigma_\eta^{(k,p,2)}}{\sigma_\eta^{(k,p,1)}} \rs \notag \\
    &= \frac{2}{p-1} \ls (p-1) + (k-1)(p-1) \frac{\sigma_\eta^{(k,p,2)}}{\sigma_{\eta}^{(k,p)}} \rs \label{eq:8} \\
    &= \frac{2}{p-1} \ls (p-1) + \frac{(k-1)(p-1)}{2} \rs \label{eq:9} \\
    &= k+1 \label{eq:computation2}
\end{align}
where we used Lemma \ref{lem:identity1} for \eqref{eq:8} and Lemma \ref{lem:identity2} for \eqref{eq:9}. We then have:
\begin{align}
    &\biggl( \Vert \nabla u(x) \Vert_2^{p-2} \rho(x)^{k} \nabla \rho(x)\cdot \nabla u(x) \times \frac{2(\sigma_{\eta}^{(k,p)} + (k-1) \sigma_\eta^{(k,p,2)})}{(p-1)\sigma_\eta^{(k,p,1)}} \notag \\
    &+ \rho(x)^{k+1} \Vert \nabla u(x) \Vert_2^{p-2} \biggl[ \Delta u(x) + \l\frac{\sigma_{\eta}^{(k,p)}}{\sigma_{\eta}^{(k,p,1)}} - 1 \r \frac{\nabla u(x)^\top \nabla^2 u(x) \nabla u(x)}{\Vert \nabla u(x)\Vert_2^2}  \biggr] \biggr)\notag \\
     &\times \frac{\sigma_\eta^{(k,p,1)}(p-1)}{2\rho(x)} \notag \\
    &= \biggl( \Vert \nabla u(x) \Vert_2^{p-2} \rho(x)^{k} \nabla \rho(x)\cdot \nabla u(x) (k+1) \notag \\
    &+ \rho(x)^{k+1} \Vert \nabla u(x) \Vert_2^{p-2} \biggl[ \Delta u(x) + (p-2) \frac{\nabla u(x)^\top \nabla^2 u(x) \nabla u(x)}{\Vert \nabla u(x)\Vert_2^2}  \biggr] \biggr) \frac{\sigma_\eta^{(k,p,1)}(p-1)}{2\rho(x)} \label{eq:10} \\
    &= \frac{\sigma_{\eta}^{(k,p)}}{2\rho(x)}\Div(\rho(x)^{k+1}\Vert \nabla u \Vert_2^{p-2} \nabla u(x)) \label{eq:11}
\end{align}
where we used \eqref{eq:computation1} and \eqref{eq:computation2} for \eqref{eq:10} and Lemma \ref{lem:identity1} for \eqref{eq:11}.
\end{proof}

\subsubsection{Proof of Theorem \ref{thm:pointwiseConsistency}}

We recall that $f:(\bbR^d)^k \mapsto \bbR$ is odd symmetric if $f(-x_1,\dots,-x_k)= -f(x_1,\dots,x_k)$ and that for such a function as well as symmetric $A$, we have $\int_A f(x_1,\dots,x_k)\, \dd x_k \cdots \dd x_1 = 0$.

\begin{proof}[Proof of Theorem \ref{thm:pointwiseConsistency}]
In the proof $C>0$ will denote a constant that can be arbitrarily large, independent of $n$ and that may change from line to line. We will roughly follow the strategy in \cite{calderGameTheoretic}.

Let us start by assuming that $p \geq 3$. By Taylor's theorem, if $\psi(t) = \vert t\vert^{p-2}t$, then, we have that $\psi(t) = \psi(a) + \psi'(a)(t-a) + \mathcal{O}(C_b^{p-3}\vert t-a\vert^2)$ for $a,t\in [-C_b,C_b]$. For $u \in \mathrm{C}^3(\mathbb{R}^d)$, let $t = u(x+z) - u(x)$ and $a = \nabla u(x)\cdot z$. Then, using the previous identity, we obtain that
\begin{align}
    &\psi(u(x+z) - u(x)) \notag \\
    &= \vert \nabla u(x)\cdot z\vert^{p-2} \nabla u(x)\cdot z + (p-1) \vert \nabla u(x) \cdot z\vert^{p-2} (u(x+z)-u(x) - \nabla u(x) \cdot z) \notag \\
    &+ \mathcal{O}(C_b^{p-3} \vert u(x+z)-u(x) - \nabla u(x) \cdot z \vert^2). \notag
\end{align}
Noting that $u(x+z)-u(x) - \nabla u(x) \cdot z = z^\top\nabla^2 u(x)z/2 + \mathcal{O}( \Vert u \Vert_{\mathrm{C}^3(\mathbb{R}^d)} \vert z \vert^{3})$ and also $u(x+z)-u(x) - \nabla u(x) \cdot z = \mathcal{O}(\Vert u \Vert_{\mathrm{C}^3(\mathbb{R}^d)} \vert z \vert^{2})$, we continue the above computation:
\begin{align}
    \psi(u(x+z) - u(x)) &= \vert \nabla u(x)\cdot z\vert^{p-2} \nabla u(x)\cdot z + \frac{(p-1)}{2} \vert \nabla u(x) \cdot z\vert^{p-2} z^\top\nabla^2 u(x)z  \notag \\
    &+ \mathcal{O}(\Vert u \Vert_{\mathrm{C}^3(\mathbb{R}^d)}^{p-1} \vert z \vert^{p+1}) + \mathcal{O}(C_b^{p-3} \Vert u \Vert_{\mathrm{C}^3(\mathbb{R}^d)}^{2} \vert z \vert^4). \notag
\end{align}
Finally, we note that $\max\{\vert u(x+z) - u(x)\vert,\vert \nabla u(x) \cdot z \vert\} \leq C_b$ means that we can pick $C_b = \Vert u \Vert_{\mathrm{C}^3(\mathbb{R}^d)} \vert z \vert$ which allows us to conclude that 
\begin{align}
    \psi(u(x+z) - u(x)) &= \vert \nabla u(x)\cdot z\vert^{p-2} \nabla u(x)\cdot z + \frac{(p-1)}{2} \vert \nabla u(x) \cdot z\vert^{p-2} z^\top\nabla^2 u(x)z \notag \\
    &+ \mathcal{O}(\Vert u \Vert_{\mathrm{C}^3(\mathbb{R}^d)}^{p-1} \vert z \vert^{p+1}). \label{eq:pointwise:finalTaylor}
\end{align}

We first start by assuming that $x_{i_0} \in \Omega_n\cap\Omega^\prime$ is fixed (and hence non-random) and let $1 \leq k \leq q$ be fixed. Let us estimate as follows: 
\begin{align}
    &n^k \eps_n^{p+kd} \Delta_{n,\eps_n}^{(k,p)}(u)(x_{i_0}) \notag \\
    &= \hspace{-0.4cm}\sum_{i_1,\dots,i_k = 1}^n \etaP(x_{i_0},\dots,x_{i_k}) \vert (x_{i_1}-x_{i_0})\cdot \nabla u(x_{i_0}) \vert^{p-2}  (x_{i_1}-x_{i_0})\cdot \nabla u(x_{i_0}) \notag \\
    &+ \frac{(p-1)}{2}\sum_{i_1,\dots,i_k = 1}^n \etaP(x_{i_0},\dots,x_{i_k}) \vert (x_{i_1}-x_{i_0})\cdot \nabla u(x_{i_0}) \vert^{p-2} \notag \\
    &\times (x_{i_1}-x_{i_0})^\top \nabla^2 u(x_{i_0}) (x_{i_1}-x_{i_0}) \label{eq:pointwise:taylor} + \hspace{-4mm} \sum_{i_1,\dots,i_k = 1}^n  \hspace{-3mm} \etaP(x_{i_0},\dots,x_{i_k}) \mathcal{O}\left( \Vert u \Vert_{\mathrm{C}^3(\mathbb{R}^d)}^{p-1} \vert x_{i_1}-x_{i_0} \vert^{p+1} \right)  \\
    &=: T_1(x_1,\dots,x_{n}) + T_2(x_1,\dots,x_{n}) + T_3(x_{1},\dots,x_{n}) \notag \\
    &= \E(T_1) + \E(T_2) + \E(T_3) + \sum_{i=1}^3 \left( T_i - \E(T_i) \right) \notag
\end{align}
where we used \eqref{eq:pointwise:finalTaylor} with $x = x_{i_0}$ and $z = x_{i_1} - x_{i_0}$ for \eqref{eq:pointwise:taylor}.

We now want to estimate $\vert T_i - \E(T_i) \vert$ for $1 \leq i \leq 3$ using Theorem \ref{thm:mcdiarmid}. For the purpose of the next few equations, for a general function $f(x_{i_0},\dots,x_{i_k})$, we will write $f(x_{i_0},\dots,x_{i_k})\vert_{\{x_1,\dots,x_n\}}$ where the extra subscript $\{x_1,\dots,x_n\}$ indicates that $x_{i_\ell} \in \{x_1,\dots,x_n\}$ for $1\leq \ell \leq k$.
Let us start by considering
\begin{align}
    &\vert T_3(x_1,\dots,x_i,\dots,x_n) - T_3(x_1,\dots,\tilde{x_i},\dots,x_n) \vert \notag \\ 
    &\leq \sum_{i_1,\dots,i_k=1}^n \Biggl\vert \left[\etaP(x_{i_0},\dots,x_{i_k}) \mathcal{O}\left( \Vert u \Vert_{\mathrm{C}^3(\mathbb{R}^d)}^{p-1} \vert x_{i_1}-x_{i_0} \vert^{p+1} \right) \right]\vert_{\{x_1,\dots,x_i,\dots,x_n\}} \notag \\
    &- \left[\etaP(x_{i_0},\dots,x_{i_k}) \mathcal{O}\left( \Vert u \Vert_{\mathrm{C}^3(\mathbb{R}^d)}^{p-1} \vert x_{i_1}-x_{i_0} \vert^{p+1} \right) \right]\vert_{\{x_1,\dots,\tilde{x}_i,\dots, x_n\}} \Biggr\vert \notag.
\end{align}
We note that each term in the latter sum is different from $0$ only if there exists $1\leq \ell \leq k$ with $i_\ell = i$. By Lemma \ref{lem:combinatorics}, there exists $S^{(n,k)}(i) \leq kn^{k-1}$ such cases and for each of those, the term in the sum can be bounded by 
$C\eps^{p+1} \|\eta\|^{t(k)}_{\Lp{\infty}} \|u\|^{p-1}_{\Ck{2}(\bbR^d)}$ where $t(k)$ is the number of terms in the double product in $\etaP$. By Assumptions \ref{ass:Main:Ass:W2} and \ref{ass:Main:Ass:S1}, this leads to:
\begin{align}
    \vert T_3(x_1,\dots,x_i,\dots,x_n) - T_3(x_1,\dots,\tilde{x}_i,\dots,x_n) \vert
    &\leq Cn^{k-1}\eps^{p+1}\|u\|_{\Ck{3}}^{p-1} \notag 
\end{align}
Similarly, 
\begin{align}
    \vert T_1(x_1,\dots,x_i,\dots,x_n) - T_1(x_1,\dots,\tilde{x}_i,\dots,x_n) \vert 
    &\leq Cn^{k-1} \|\eta\|_{\Lp{\infty}}^{t(k)} \eps^{p-1} \|u\|_{\Ck{1}}^{p-1} \notag    
\end{align}
and 
\begin{align}
    \vert T_2(x_1,\dots,x_i,\dots,x_n) - T_2(x_1,\dots,\tilde{x}_i,\dots,x_n) \vert 
    &\leq Cn^{k-1}\eps^p \|\eta\|_{\Lp{\infty}}^{t(k)} \|u\|_{\Ck{2}}^{p-1}. \notag
\end{align}
Using Theorem \ref{thm:mcdiarmid}, we therefore obtain that $$\bbP(\vert T_i - \E(T_i)\vert \geq t) \leq 2 \exp\left( -\frac{t^2}{Cn^{2k-1} \eps_n^{2p -2} \Vert u \Vert_{\Ck{3}} } \right)$$ and with $t = n^k \eps_n^{p + kd}\delta \Vert u \Vert^{p-1}_{\Ck{3}}$, 
\(
\bbP(\vert T_i - \E(T_i)\vert \geq n^k \eps_n^{p + kd}\delta ) 
\leq 2 \exp\left( -C n \eps_n^{2(1+kd)} \delta^2\right)
\)
for $1\leq i \leq 3$.

We now estimate $\E(T_i)$ for $1 \leq i \leq 3$. In particular, 
\begin{align}
    \frac{1}{n^k\eps_n^{p+kd}} \E(T_3) &= \frac{1}{n^k\eps_n^{p+kd}} \mathcal{O}\left( \eps_n^{p+1} \Vert u \Vert^{p-1}_{\mathrm{C}^3(\mathbb{R}^d)} \E\left( \sum_{i_1,\dots,i_k = 1}^n \etaP(x_{i_0},\dots,x_{i_k}) \right)   \right)\label{eq:pointwise:lengthscale1} \\
    &= \mathcal{O}\left( \eps_n \Vert u \Vert^{p-1}_{\mathrm{C}^3(\mathbb{R}^d)} \frac{1}{\eps_n^{kd}} \int_{\Omega^{k}} \etaP(x_{i_0},x_1,\dots,x_k) \left[\prod_{\ell = 1}^k \rho(x_\ell) \right] \, \dd x_k \cdots \dd x_1 \right) \notag \\
    &= \mathcal{O} \left(  \eps_n \Vert u \Vert^{p-1}_{\mathrm{C}^3(\mathbb{R}^d)} \int_{(\bbR^d)^{k}} \etaPTilde(z_1,\dots,z_k) \, \dd z_k \cdots \dd z_1 \right) \label{eq:pointwise:changeOfVariables} \\
    &= \mathcal{O} \left( \eps_n \Vert u \Vert^{p-1}_{\mathrm{C}^3(\mathbb{R}^d)} \right) \label{eq:pointwise:compactSupport}
\end{align}
where we used Assumption \ref{ass:Main:Ass:W2} to deduce that $\vert x_{i_0} - x_{i_1} \vert = \mathcal{O}(\eps_n)$ for \eqref{eq:pointwise:lengthscale1}, Assumption \ref{ass:Main:Ass:M2} and the change of variables $z_j = (x_j - x_{i_0})/\eps_n$ for $1 \leq j \leq k$ for \eqref{eq:pointwise:changeOfVariables} as well as Assumption \ref{ass:Main:Ass:W2} for \eqref{eq:pointwise:compactSupport}.

For $T_1$, for $n$ large enough, we proceed as follows:
\begin{align}
    &\frac{1}{n^k \eps_n^{p + kd}} \E(T_1) = \frac{1}{ \eps_n^{p + kd}} \int_{\Omega^k}  \etaP(x_{i_0},x_1,\dots,x_{k}) \notag \\
    &\quad\times \vert (x_{1}-x_{i_0})\cdot \nabla u(x_{i_0}) \vert^{p-2}  (x_{1}-x_{i_0})\cdot \nabla u(x_{i_0}) \left[\prod_{\ell = 1}^k \rho(x_\ell) \right] \, \dd x_k \cdots \dd x_1 \notag \\
    &= \frac{1}{\eps_n} \int_{\otimes_j (\{z_j \, | \, x_{i_0} + \eps_n z_j \in \Omega \} \cap \mathrm{supp}(\eta))}  \etaPTilde(z_1,\dots,z_k) \vert z_1 \cdot \nabla u(x_{i_0}) \vert^{p-2}  z_1 \cdot \nabla u(x_{i_0}) \notag \\
    &\quad\times \left[\prod_{\ell = 1}^k \rho(x_{i_0} + \eps_n z_\ell) \right] \, \dd z_k \cdots \dd z_1 \label{eq:pointwise:changeOfVariables2} \\
    &= \frac{1}{\eps_n} \int_{\mathrm{supp}(\eta)^k}  \etaPTilde(z_1,\dots,z_k) \vert z_1 \cdot \nabla u(x_{i_0}) \vert^{p-2}  z_1 \cdot \nabla u(x_{i_0}) \rho(x_{i_0})^k \, \dd z_k \cdots \dd z_1 \notag \\
    &\quad+ \int_{\mathrm{supp}(\eta)^k}  \hspace{-0.7cm}\etaPTilde(z_1,\dots,z_k) \vert z_1 \cdot \nabla u(x_{i_0}) \vert^{p-2}  z_1 \cdot \nabla u(x_{i_0}) \rho(x_{i_0})^{k-1} \nabla \rho(x_{i_0})\notag \\
    &\quad\times (z_1 + \dots + z_k) \, \dd z_k \cdots \dd z_1 \label{eq:pointwise:lemmas} \\
    &\quad+ \mathcal{O}\left(\eps_n \int_{\mathrm{supp}(\eta)^k}  \etaPTilde(z_1,\dots,z_k) \vert z_1 \cdot \nabla u(x_{i_0}) \vert^{p-2}  z_1 \cdot \nabla u(x_{i_0}) \, \dd z_k \cdots \dd z_1 \right) \notag \\
    &= \int_{(\bbR^d)^k} \hspace{-0.3cm} \etaPTilde(z_1,\dots,z_k) \vert z_1 \cdot \nabla u(x_{i_0}) \vert^{p-2}  z_1 \cdot \nabla u(x_{i_0}) \rho(x_{i_0})^{k-1} \nabla \rho(x_{i_0}) \notag \\ 
    &\quad\times (z_1 + \dots + z_k) \, \dd z_k \cdots \dd z_1 + \mathcal{O} \left( \eps_n \Vert u \Vert_{\mathrm{C}^3(\mathbb{R}^d)}^{p-1} \right) \label{eq:pointwise:odd} \\
    &= \rho(x_{i_0})^{k-1} \sum_{i=1}^d \frac{\partial \rho}{\partial x_i}(x_{i_0}) \int_{(\bbR^d)^k}  \etaPTilde(z_1,\dots,z_k) \psi \left(z_1 \cdot \nabla u(x_{i_0}) \right) \notag \\
    &\quad\times(z_1 + \dots + z_k)_i \, \dd z_k \cdots \dd z_1 + \mathcal{O} \left( \eps_n \Vert u \Vert_{\mathrm{C}^3(\mathbb{R}^d)}^{p-1} \right) \label{eq:pointwise:continue}
\end{align}
where we used the change of variables $z_j = (x_j - x_{i_0})/\eps_n$ for $1 \leq j \leq k$ for \eqref{eq:pointwise:changeOfVariables2}, Lemmas \ref{lem:product} and \ref{lem:domain} for \eqref{eq:pointwise:lemmas}, Assumption \ref{ass:Main:Ass:W2} as well as the fact that $f(z_1,\dots,z_k) := \etaPTilde(z_1,\dots,z_k) \vert z_1 \cdot \nabla u(x_{i_0}) \vert^{p-2}  z_1 \cdot \nabla u(x_{i_0}) \rho(x_{i_0})^k$ is odd symmetric for \eqref{eq:pointwise:odd}, and recalling $\psi(t) = |t|^{p-2}t$. Let $O$ be the orthogonal matrix so that $Oe_d = \nabla u(x_{i_0}) / \Vert \nabla u(x_{i_0}) \Vert_2$ where $e_d = (0,\dots,0,1) \in \bbR^d$. By the change of variables $\tilde{z}_j = O^\top z_j$ for $1\leq j \leq k$ and noting that $z_1 \cdot \nabla u(x_{i_0}) = \tilde{z}_1 \cdot O^\top \nabla u(x_{i_0}) = (\tilde{z}_1)_d \Vert \nabla u(x_{i_0}) \Vert_2$, we can continue our computation from~\eqref{eq:pointwise:continue}:
\begin{align}
    &\frac{1}{n^k \eps_n^{p + kd}} \E(T_1) = \Vert \nabla u(x_{i_0}) \Vert_2^{p-1} \rho(x_{i_0})^{k-1} \sum_{i=1}^d \frac{\partial \rho}{\partial x_i}(x_{i_0}) \biggl[ \int_{(\bbR^d)^k}  \etaPTilde(\tilde{z}_1,\dots,\tilde{z}_k) \psi \left( (\tilde{z}_1)_d \right) \notag \\ 
    &\qquad\times \sum_{j=1}^d (O)_{ij} (\tilde{z}_1 + \dots + \tilde{z}_k)_j \, \dd \tilde{z}_k \cdots \dd \tilde{z}_1 \biggr] + \mathcal{O} \left( \eps_n \Vert u \Vert_{\mathrm{C}^3(\mathbb{R}^d)}^{p-1} \right) \notag \\
    &= \Vert \nabla u(x_{i_0}) \Vert_2^{p-1} \rho(x_{i_0})^{k-1} \sum_{i,j=1}^d \frac{\partial \rho}{\partial x_i}(x_{i_0}) (O)_{ij} \sum_{r=1}^k \int_{(\bbR^d)^k}  \hspace{-0.4cm}\etaPTilde(\tilde{z}_1,\dots,\tilde{z}_k) \notag \\ 
    &\qquad\times \psi \left( (\tilde{z}_1)_d \right) (\tilde{z}_r)_j \, \dd \tilde{z}_k \cdots \dd \tilde{z}_1 + \mathcal{O} \left( \eps_n \Vert u \Vert_{\mathrm{C}^3(\mathbb{R}^d)}^{p-1} \right). \notag
\end{align}
For $j \neq d$, we note that 
\begin{align}
    &T_4 := \int_{(\bbR^d)^k}  \etaPTilde(\tilde{z}_1,\dots,\tilde{z}_k) \psi \left( (\tilde{z}_1)_d \right) (\tilde{z}_r)_j \, \dd \tilde{z}_k \cdots \dd \tilde{z}_1 \notag \\
    &= \int_{(\bbR)^{k(d-1)}} \hspace{-5mm}(\tilde{z}_r)_j \left[ \int_{\bbR^k} \etaPTilde(\tilde{z}_1,\dots,\tilde{z}_k) \psi \left( (\tilde{z}_1)_d \right)  \, \dd (\tilde{z}_k)_d \cdots \dd (\tilde{z}_1)_d \right] \, \dd (\tilde{z}_k)_{1:d-1} \cdots \dd (\tilde{z}_{1})_{1:d-1} \notag
\end{align}
(denoting by $(a)_{j:k}$ the elements $a_j,a_{j+1}\dots,a_{k-1},a_{k}$)
and the function $f:\bbR^k \mapsto \bbR$ defined as 
\[
f(y_1,\dots,y_k) = \etaPTilde((\tilde{z}_1)_{1:d-1},y_1,(\tilde{z}_2)_{1:d-1},y_2,\dots,(\tilde{z}_k)_{1:d-1},y_k) \psi \left( y_1 \right)
\]
is odd symmetric for any fixed $(\tilde{z}_1)_{1:d-1},\dots, 
(\tilde{z}_k)_{1:d-1}$ and therefore $T_4 = 0$ and 
\begin{align}
        \frac{1}{n^k \eps_n^{p + kd}} \E(T_1) & = \Vert \nabla u(x_{i_0}) \Vert_2^{p-1} \rho(x_{i_0})^{k-1} \sum_{i=1}^d \frac{\partial \rho}{\partial x_i}(x_{i_0}) (O)_{id} \notag \\ 
        &\quad\times \sum_{r=1}^k \int_{(\bbR^d)^k}  \hspace{-0.3cm}\etaPTilde(\tilde{z}_1,\dots,\tilde{z}_k)  \psi( (\tilde{z}_1)_d ) (\tilde{z}_r)_d  \, \dd \tilde{z}_k \cdots \dd \tilde{z}_1 + \mathcal{O} \left( \eps_n \Vert u \Vert_{\mathrm{C}^3(\mathbb{R}^d)}^{p-1} \right) \notag \\
    &= \Vert \nabla u(x_{i_0}) \Vert_2^{p-1} \rho(x_{i_0})^{k-1} \sum_{i=1}^d \frac{\partial \rho}{\partial x_i}(x_{i_0}) (O)_{id} (\sigma_{\eta}^{(k,p)} + (k-1) \sigma_\eta^{(k,p,2)}) \notag \\
    &\quad+ \mathcal{O} \left( \eps_n \Vert u \Vert_{\mathrm{C}^3(\mathbb{R}^d)}^{p-1} \right). \notag
\end{align}
By recalling that $(O)_{id} = (\nabla u (x_{i_0}))_i/\Vert \nabla u (x_{i_0}) \Vert_2$, we can conclude: 
\begin{align} 
    \frac{1}{n^k \eps_n^{p + kd}} \E(T_1) &= \Vert \nabla u(x_{i_0}) \Vert_2^{p-2} \rho(x_{i_0})^{k-1} \nabla \rho(x_{i_0})\cdot \nabla u(x_{i_0}) (\sigma_{\eta}^{(k,p)} + (k-1) \sigma_\eta^{(k,p,2)}) \notag \\
    & \qquad \qquad + \mathcal{O} \left( \eps_n \Vert u \Vert_{\mathrm{C}^3(\mathbb{R}^d)}^{p-1} \right). \label{eq:pointwise:T1}
\end{align}

Let us now tackle $T_2$. We estimate as follows, for $n$ large enough:
\begin{align}
    &\frac{1}{n^k\eps_n^{p + kd}} \E(T_2) \notag \\
    &\quad= \frac{(p-1)}{2}\int_{(\bbR^d)^k} \hspace{-0.3cm}\etaPTilde(z_1,\dots,z_k) \vert z_1\cdot \nabla u(x_{i_0}) \vert^{p-2}  z_1^\top  \nabla^2 u(x_{i_0})   z_1 \notag \\
    &\qquad\times\left[\prod_{\ell=1}^k \rho(x_{i_0} + \eps_n z_\ell)\right] \, \dd z_k \cdots \dd z_1 \label{eq:pointwise:T2:lemma} \\
    &\quad= \frac{(p-1)}{2} \rho(x_{i_0})^k \int_{(\bbR^d)^k}  \etaPTilde(z_1,\dots,z_k) \vert z_1 \cdot \nabla u(x_{i_0}) \vert^{p-2} \notag \\
    &\qquad\times  z_1^\top \nabla^2 u(x_{i_0})   z_1 \, \dd z_k \cdots \dd z_1 + \mathcal{O} \left( \eps_n \Vert u \Vert_{\mathrm{C}^3(\mathbb{R}^d)}^{p-1} \right) \label{eq:pointwise:T2:lemma2} \\
    &\quad= \frac{(p-1)}{2} \rho(x_{i_0})^k \Vert \nabla u(x_{i_0}) \Vert_2^{p-2} \int_{(\bbR^d)^k} \hspace{-0.4cm}\etaPTilde(\tilde{z}_1,\dots,\tilde{z}_k) \vert (\tilde{z}_1)_d \vert^{p-2}  \notag \\
    &\qquad\times (O \tilde{z}_1)^\top \nabla^2 u(x_{i_0})   (O \tilde{z}_1) \, \dd \tilde{z}_k \cdots \dd \tilde{z}_1 + \mathcal{O} \left( \eps_n \Vert u \Vert_{\mathrm{C}^3(\mathbb{R}^d)}^{p-1} \right) \label{eq:pointwise:T2:changeOfVariables} \\
    &\quad= \frac{(p-1)}{2} \rho(x_{i_0})^k \Vert \nabla u(x_{i_0}) \Vert_2^{p-2} \int_{(\bbR^d)^k} \etaPTilde(\tilde{z}_1,\dots,\tilde{z}_k) \vert (\tilde{z}_1)_d \vert^{p-2} \notag \\
    &\qquad\times \sum_{i,j = 1}^d (\nabla^2 u(x_{i_0}))_{ij} (O \tilde{z}_1)_i (O \tilde{z}_1)_j \, \dd \tilde{z}_k \cdots \dd \tilde{z}_1 + \mathcal{O} \left( \eps_n \Vert u \Vert_{\mathrm{C}^3(\mathbb{R}^d)}^{p-1} \right) \notag \\
    &\quad= \frac{(p-1)}{2} \rho(x_{i_0})^k \Vert \nabla u(x_{i_0}) \Vert_2^{p-2} \sum_{i,j = 1}^d (\nabla^2 u(x_{i_0}))_{ij} \sum_{r,\ell = 1}^d (O)_{ir} (O)_{j\ell} \notag \\
    &\qquad\times \int_{(\bbR^d)^k} \etaPTilde(\tilde{z}_1,\dots,\tilde{z}_k) \vert (\tilde{z}_1)_d \vert^{p-2}  (\tilde{z}_1)_\ell (\tilde{z}_1)_r \, \dd \tilde{z}_k \cdots \dd \tilde{z}_1 + \mathcal{O} \left( \eps_n \Vert u \Vert_{\mathrm{C}^3(\mathbb{R}^d)}^{p-1} \right) \notag \\  
    &\quad= \frac{(p-1)}{2} \rho(x_{i_0})^k \Vert \nabla u(x_{i_0}) \Vert_2^{p-2} \sum_{i,j = 1}^d (\nabla^2 u(x_{i_0}))_{ij} \sum_{r,\ell = 1}^d (O)_{ir} (O)_{j\ell} \notag \\
    &\qquad\times \int_{(\bbR^d)^k} \hspace{-0.5cm}\etaPTilde(\tilde{z}_1,\dots,\tilde{z}_k) \vert (\tilde{z}_1)_d \vert^{p-2} (\tilde{z}_1)_\ell (\tilde{z}_1)_r \, \dd \tilde{z}_k \cdots \dd \tilde{z}_1 + \mathcal{O} \left( \eps_n \Vert u \Vert_{\mathrm{C}^3(\mathbb{R}^d)}^{p-1} \right) \notag
\end{align}
where we used the change of variables $z_j = (x_j - x_{i_0})/\eps_n$ for $1\leq j \leq k$ and Lemma \ref{lem:domain} for \eqref{eq:pointwise:T2:lemma}, Lemma \ref{lem:product} and Assumption \ref{ass:Main:Ass:W2} for \eqref{eq:pointwise:T2:lemma2} and the change of variables $\tilde{z}_j = O^\top z_j$ for $1\leq j \leq k$ for \eqref{eq:pointwise:T2:changeOfVariables}. Similarly to the above, for $\ell \neq r \neq d$, 
\begin{align}
    T_5 &:= \int_{(\bbR^d)^k} \etaPTilde(\tilde{z}_1,\dots,\tilde{z}_k) \vert (\tilde{z}_1)_d \vert^{p-2} (\tilde{z}_1)_\ell (\tilde{z}_1)_r \, \dd \tilde{z}_k \cdots \dd \tilde{z}_1 \notag \\
    &= \int_{(\bbR)^{k(d-1)}}  \hspace{-8mm}  \vert (\tilde{z}_1)_d \vert^{p-2} (\tilde{z}_1)_\ell \left[ \int_{\bbR^k} \etaPTilde(\tilde{z}_1,\dots,\tilde{z}_k) (\tilde{z}_1)_r \, \dd (\tilde{z}_k)_r \cdots \dd (\tilde{z}_1)_r \right] \, \dd (\tilde{z}_k)_{-r} \cdots \dd (\tilde{z}_1)_{-r} \notag 
 \end{align}
(denoting by $(a)_{-r}$ the vector $(a_1,\dots,a_{r-1},a_{r+1}, \dots, a_d)$)
and the function $f:\bbR^k \mapsto \bbR$ defined as 
\[
f(y_1,\dots,y_k) = \etaPTilde((\tilde{z}_1)_{1:r-1},y_1,(\tilde{z}_1)_{r+1:d},(\tilde{z}_2)_{1:r-1},y_2,(\tilde{z}_2)_{r+1:d},\dots,(\tilde{z}_k)_{r+1:d}) y_1
\]
is odd symmetric for any fixed $(\tilde{z}_1)_{-r},\dots,(\tilde{z}_k)_{-r}$, so $T_5=0$ in this case. 
By symmetry the case $r\neq \ell\neq d$ follows.
We therefore have:
\begin{align}
    &\frac{1}{n^k\eps_n^{p + kd}} \E(T_2) \notag \\
    &\qquad= \frac{(p-1)}{2} \rho(x_{i_0})^k \Vert \nabla u(x_{i_0}) \Vert_2^{p-2} \sum_{i,j = 1}^d (\nabla^2 u(x_{i_0}))_{ij} \sum_{r=1}^d (O)_{ir} (O)_{jr}  \notag \\
    &\qquad\qquad\times  \int_{(\bbR^d)^k} \hspace{-0.4cm}\etaPTilde(\tilde{z}_1,\dots,\tilde{z}_k) \vert (\tilde{z}_1)_d \vert^{p-2} (\tilde{z}_1)_r^2 \, \dd \tilde{z}_k \cdots \dd \tilde{z}_1 + \mathcal{O} \left( \eps_n \Vert u \Vert_{\mathrm{C}^3(\mathbb{R}^d)}^{p-1} \right) \notag \\
    &\qquad= \frac{(p-1)}{2} \rho(x_{i_0})^k \Vert \nabla u(x_{i_0}) \Vert_2^{p-2} \sum_{r=1}^d (O^\top\nabla^2 u(x_{i_0}) O)_{rr}  \notag \\
    &\qquad\qquad\times  \int_{(\bbR^d)^k} \etaPTilde(\tilde{z}_1,\dots,\tilde{z}_k) \vert (\tilde{z}_1)_d \vert^{p-2} (\tilde{z}_1)_r^2 \, \dd \tilde{z}_k \cdots \dd \tilde{z}_1 + \mathcal{O} \left( \eps_n \Vert u \Vert_{\mathrm{C}^3(\mathbb{R}^d)}^{p-1} \right) \notag \\
    &\qquad= \frac{(p-1)}{2} \rho(x_{i_0})^k \Vert \nabla u(x_{i_0}) \Vert_2^{p-2} \notag \\
    &\qquad\qquad\times\biggl[ \mathrm{Tr}(\nabla^2 u(x_{i_0})) \sigma_{\eta}^{(k,p,1)} + (\sigma_{\eta}^{(k,p)} - \sigma_{\eta}^{(k,p,1)} ) (O^\top\nabla^2 u(x_{i_0}) O)_{dd} \biggr]\notag \\
    &\qquad\qquad+ \mathcal{O} \left( \eps_n \Vert u \Vert_{\mathrm{C}^3(\mathbb{R}^d)}^{p-1} \right) \notag \\ 
    &\qquad= \frac{(p-1)}{2} \rho(x_{i_0})^k \Vert \nabla u(x_{i_0}) \Vert_2^{p-2} \biggl[ \mathrm{Tr}(\nabla^2 u(x_{i_0})) \sigma_{\eta}^{(k,p,1)} \notag \\
    &\qquad\qquad+(\sigma_{\eta}^{(k,p)} - \sigma_{\eta}^{(k,p,1)} ) \frac{1}{\Vert \nabla u(x_{i_0})\Vert_2^2} \nabla u(x_{i_0})^\top  \nabla^2 u(x_{i_0}) \nabla u(x_{i_0})  \biggr] \label{eq:pointwise:trace} \\
    & \qquad\qquad + \mathcal{O} \left( \eps_n \Vert u \Vert_{\mathrm{C}^3(\mathbb{R}^d)}^{p-1} \right) \notag
\end{align}
where we used the fact that 
\begin{align*}
(O^\top\nabla^2 u(x_{i_0}) O)_{dd} &= \sum_{i,j=1}^n (\nabla^2 u(x_{i_0}))_{ij} (O)_{id} (O)_{jd} \\
&= \sum_{i,j=1}^n (\nabla^2 u(x_{i_0}))_{ij} \frac{(\nabla u (x_{i_0}))_i}{\Vert \nabla u (x_{i_0}) \Vert_2} \frac{(\nabla u (x_{i_0}))_j}{\Vert \nabla u (x_{i_0}) \Vert_2}
\end{align*} for \eqref{eq:pointwise:trace}. 

Combining \eqref{eq:pointwise:T1}, \eqref{eq:pointwise:trace} and \eqref{eq:pointwise:compactSupport}, with probability 
$1 - 6\exp \left( -C n\eps_n^{2(1+kd)}\delta^2 \right)$, since $\eps_n\leq \delta$:
\begin{align}
    \Delta_{n,\eps_n}^{(k,p)}(u)(x_{i_0}) &= \Vert \nabla u(x_{i_0}) \Vert_2^{p-2} \rho(x_{i_0})^{k-1} \nabla \rho(x_{i_0})\cdot \nabla u(x_{i_0}) (\sigma_{\eta}^{(k,p)} + (k-1) \sigma_\eta^{(k,p,2)}) \notag \\
    &\qquad\qquad+ \frac{(p-1)}{2} \rho(x_{i_0})^k \Vert \nabla u(x_{i_0}) \Vert_2^{p-2} \biggl[ \mathrm{Tr}(\nabla^2 u(x_{i_0})) \sigma_{\eta}^{(k,p,1)} \notag \\
    &\qquad\qquad+ (\sigma_{\eta}^{(k,p)} - \sigma_{\eta}^{(k,p,1)} ) \frac{1}{\Vert \nabla u(x_{i_0})\Vert_2^2} \nabla u(x_{i_0})^\top \nabla^2 u(x_{i_0}) \nabla u(x_{i_0})  \biggr] \notag \\
    &\qquad\qquad+ \mathcal{O} \left( \delta \Vert u \Vert_{\mathrm{C}^3(\mathbb{R}^d)}^{p-1} \right). \notag
\end{align}
As in \cite{calderGameTheoretic}, by taking a union bound on all $x_{i_0} \in \Omega_n\cap\Omega^\prime$ and using Corollary \ref{cor:alternativeLaplacian}, we obtain that
    $ \biggl\vert \Delta_{n,\eps_n}^{(k,p)} (u)(x_{i_0}) -  \rho(x_{i_0})\Delta_\infty^{(k,p)}(u)(x_{i_0}) \biggr\vert = \mathcal{O} \left( \delta \Vert u \Vert_{\mathrm{C}^3(\mathbb{R}^d)}^{p-1} \right)$
 with probability $1 - C n \exp \left( -Cn\eps_n^{2(1+kd)}\delta^2 \right)$. To conclude the proof, we sum over $1 \leq k \leq q$.

When $p=2$, we have $\psi(t) = t$ and directly obtain the estimate \eqref{eq:pointwise:finalTaylor}:
\[
u(x+z) - u(x) = \nabla u(x)\cdot z + \frac{1}{2}z^\top\nabla^2 u(x)z + \mathcal{O}(\Vert u \Vert_{\mathrm{C}^3(\mathbb{R}^d)} \vert z \vert^{3}).
\]
For the remainder of the proof, we proceed exactly as above with $p$ replaced by $2$.
\end{proof}

\subsection{\texorpdfstring{$\Gamma$}{Gamma}-convergence}

By re-adapting the results in \cite{Trillos3,Slepcev}, we are able to show the following $\Gamma$-convergence results. In particular, we perform a decomposition of our problem: in Section \ref{sec:gammaConvergenceHypergraph}, we first show $\Gamma$-convergence of a nonlocal version of our continuum energies 
\[
\cE_{\eps,\mathrm{NL}}^{(k,p)}(v,\eta) = \frac{1}{\eps^{p + kd}} \hspace{-0.8mm}  \int_{\Omega^{k+1}} \hspace{-1.8mm}  \ls \prod_{j=1}^{k} \prod_{r=0}^{j-1} \eta\l \frac{\vert x_{j} - x_{r}\vert}{\eps} \r \rs \hspace{-1.6mm} \left\vert v(x_1) - v(x_0) \right\vert^p  \prod_{\ell = 0}^k \rho(x_\ell) \, \dd x_k \cdots \dd x_0
\]
to their local counterparts; next, we establish $\Gamma$-convergence of the discrete energies to the nonlocal continuum energies. 

We will use the below inequality in our computations. For $a,b \in \bbR$, $\delta > 0$ and $p > 1$, there exists a constant $C_\delta$ such that 
\begin{equation}\label{eq:intro:equationTriangle}
    \vert \vert c \vert^{p} - \vert a \vert^{p} \vert \leq C_\delta \vert c-a \vert^p + \delta \vert a \vert^p.
\end{equation}
We also note that $C_\delta \to \infty$ as $\delta \to 0$.

\subsubsection{\texorpdfstring{$\Gamma$}{Gamma}-convergence of the discrete energies} 

\begin{proposition}[$\liminf$-inequality in the ill-posed case] \label{prop:proofs:gammaConvergence:liminf}
    Assume that \ref{ass:Main:Ass:S1}, \ref{ass:Main:Ass:M1}, \ref{ass:Main:Ass:M2}, \ref{ass:Main:Ass:W2}, \ref{ass:Main:Ass:D1}, \ref{ass:Main:Ass:D2} and \ref{ass:Main:Ass:L1} hold. Then, $\bbP$-a.s., for every $(\nu,v) \in \TLp{p}(\Omega)$ and $\{(\nu_n,v_{_n})\}_{n=1}^\infty$ with $(\nu_n,v_{n}) \to (\nu,v)$ in $\TLp{p}(\Omega)$, we have
    \begin{align}
        &\liminf_{n \to \infty} (\cS\cF)_{n,\eps_n}^{(q,p)}((\nu_n,v_n)) \geq (\cS\cG)_{\infty}^{(q,p)}((\nu,v)).  \label{eq:proofs:liminfDiscrete:liminf}
    \end{align}
\end{proposition}

\begin{proof}
With probability one, we can assume that the conclusions of Theorem \ref{thm:Back:TLp:LinftyMapsRate} hold.

Since \eqref{eq:proofs:liminfDiscrete:liminf} is trivial if 
$\liminf_{n\to \infty} (\cS\cF)_{n,\eps_n}^{(q,p)}((\nu_n,v_n)) = \infty$, we might assume without loss of generality (see \cite{weihs2023consistency}) that $\sup_{n > 0}(\cS\cF)_{n,\eps_n}^{(q,p)}((\nu_n,v_n)) \leq  C$. This implies that $\nu_n = \mu_n$ and, since $(\nu_n,v_{n}) \to (\nu,v)$ in $\TLp{p}(\Omega)$, we have $\nu = \mu$. 
We start by showing
\begin{equation} \label{eq:proofs:gammaConvergence:Eliminf}
\liminf_{n \to \infty} \cE_{n,\eps_n}^{(k,p)}(v_n) \geq \cE_{\infty}^{(k,p)}(v).
\end{equation}

We follow the three-step decomposition of \cite[Theorem 1.1]{Trillos3}. First, suppose that $\eta(t) = a$ if $0 \leq t \leq b$ and $\eta(t) = 0$ else where $a$ and $b$ are positive constants. Define $\tilde{\eps}_n = \eps_n - \frac{2\Vert T_n - \Id \Vert_{\Lp{\infty}}}{b}$. From \cite[Lemma 4.2]{Slepcev}, we know that $\frac{\eps_n}{\tilde{\eps}_n} \to 1$ and 
\[
\eta\l\frac{\vert x - y \vert}{\tilde{\eps}_n}\r \leq \eta\l\frac{\vert T_n(x) - T_n(y) \vert}{\eps_n}\r.
\]
Using a change of variables and the above, we obtain that 
\begin{align}
    \cE_{n,\eps_n}^{(k,p)}(v_n) &\geq \frac{1}{\eps_n^{p + kd}} \int_{\Omega^{k+1}} \ls \prod_{j=1}^{k} \prod_{r=0}^{j-1} \eta\l \frac{\vert x_{j} - x_{r}\vert}{\tilde{\eps}_n} \r \rs \left\vert v_n \circ T_n(x_1) - v_n \circ T_n (x_0) \right\vert^p \notag \\
    &\qquad\qquad\times \ls \prod_{\ell=0}^k \rho(x_\ell) \rs \, \dd x_k \cdots \dd x_0 \notag \\
    &= \l \frac{\tilde{\eps}_n}{\eps_n} \r^{p + kd} \cE_{\tilde{\eps}_n,\mathrm{NL}}^{(k,p)}(v_n \circ T_n,\eta). \notag
\end{align}
Since $u_n \to u$ in $\TLp{p}(\Omega)$, we have that $u_n \circ T_n \to u$ in $\Lp{p}(\Omega)$ and we can therefore use Proposition \ref{prop:proofs:gammaConvergence:liminfNonlocal} to deduce that:
\begin{align}
&\liminf_{n \to \infty } \cE_{n,\eps_n}^{(k,p)}(u_n) \geq \liminf_{n \to \infty} \l \frac{\tilde{\eps}_n}{\eps_n} \r^{p + kd} \cE_{\tilde{\eps}_n,\mathrm{NL}}^{(k,p)}(v_n \circ T_n,\eta) \geq \cE_{\infty}^{(k,p)}(v). \notag
\end{align}

Our next step is to assume that $\eta = \sum_{k=1}^\ell \eta_l$ satisfies Assumption \ref{ass:Main:Ass:W2} where $\eta_l$ are functions of the type considered in the above. Then, as in \cite[Theorem 1.1]{Trillos3}, we use the linearity of the integral to obtain \eqref{eq:proofs:liminfDiscrete:liminf}.

Our final step is to let $\eta$ be a general function satisfying Assumption \ref{ass:Main:Ass:W2}. Then, as in \cite[Theorem 1.1]{Trillos3}, we use the monotone convergence theorem and approximation of $\eta$ by functions as in the previous step to obtain~\eqref{eq:proofs:gammaConvergence:Eliminf}.

By subadditivity of the $\liminf$, we can conclude \eqref{eq:proofs:liminfDiscrete:liminf}.
\end{proof}

\begin{proposition}[$\limsup$-inequality in the well-posed case]\label{prop:proofs:gammaConvergence:limsupDiscrete}
    Assume that \ref{ass:Main:Ass:S1}, \ref{ass:Main:Ass:M1}, \ref{ass:Main:Ass:M2}, \ref{ass:Main:Ass:W2}, \ref{ass:Main:Ass:D1}, \ref{ass:Main:Ass:D2} and \ref{ass:Main:Ass:L1} hold. 
    Then, $\bbP$-a.s., for every $(\nu,v) \in \TLp{p}(\Omega)$, there exists a sequence $\{(\nu_n,v_{_n})\}_{n=1}^\infty$ with $(\nu_n,v_{n}) \to (\nu,v)$ in $\TLp{p}(\Omega)$ and
    \begin{align}
        &\limsup_{n \to \infty} \cF_{n,\eps_n}^{(k,p)}((\nu_n,v_n)) \leq \cF_{\infty}^{(k,p)}((\nu,v)) \label{eq:proofs:limsupDiscrete:limsup}.
    \end{align}
    In particular, for $1 \leq k \leq q$ and $v \in \Ck{\infty}(\bar{\Omega})$ with $v(x_i) = y_i$ for $i \leq N$, we can pick $\{(\nu_n,v_n)\}_{n=1}^\infty = \{(\mu_n,v|_{\Omega_n})\}_{n=1}^\infty $.
\end{proposition}

\begin{proof}
With probability one, we can assume that the conclusions of Theorem \ref{thm:Back:TLp:LinftyMapsRate} hold.

In the proof $C>0$ will denote a constant that can be arbitrarily large, is independent of $n$, and that may change from line to line.

    We start by noting that \eqref{eq:proofs:limsupDiscrete:limsup} is trivial if $\cF_{\infty}^{(k,p)}(\nu,v) = \infty$ so that we assume $\cF_{\infty}^{(k,p)}(\nu,v) < \infty$ which implies that $\nu= \mu$, $v \in \Wkp{1}{p}(\Omega)$ and $v(x_i) = y_i$ for all $i \leq N$.
    Furthermore, we are going to apply \cite[Remark 2.7]{Trillos3}, so it is sufficient to verify \eqref{eq:proofs:limsupDiscrete:limsup} on a dense subset of $\{\mu\} \times \Wkp{1}{p}(\Omega)$, namely we consider $(\mu,v) \in \{ \mu \} \times \Ck{\infty}(\bar{\Omega})$ with $v(x_i) = y_i$ for all $i \leq N$. We let $\nu_n = \mu_n$ and $v_n = v\vert_{\Omega_n}$ so $v_n(x_i) = y_i$ for all $i \leq N$ and \eqref{eq:proofs:limsupDiscrete:limsup} is equivalent to \(
    \limsup_{n \to \infty} \cE_{n,\eps_n}^{(k,p)}(v_n) \leq \cE_{\infty}^{(k,p)}(v).
    \)
    The fact that $(\mu_n,v_n) \to (\mu,v)$ in $\TLp{p}(\Omega)$ follows analogously from what is shown in \cite[Proposition 4.17]{weihs2023consistency}: it relies on the fact that $v$ is uniformly continuous as well as $\Vert T_n - \Id \Vert_{\Lp{\infty}} \to 0$.

    We follow the three-step decomposition of \cite[Theorem 1.1]{Trillos3}. First, suppose that $\eta(t) = a$ if $0 \leq t \leq b$ and $\eta(t) = 0$ else where $a$ and $b$ are positive constants. Define $\tilde{\eps}_n = \eps_n + \frac{2\Vert T_n - \Id \Vert_{\Lp{\infty}}}{b}$. From \cite[Theorem 1.1]{Trillos3}, we know that $\frac{\eps_n}{\tilde{\eps}_n} \to 1$ and similarly to the previous proposition:
    \begin{equation} \label{eq:proofs:limsupDiscrete:NLUpperBound}
        \cE_{n,\eps_n}^{(k,p)}(v_n) \leq \l\frac{\tilde{\eps}_n}{\eps_n}\r^{p+kd} \cE_{\tilde{\eps}_n,\mathrm{NL}}(v_n\circ T_n,\eta).
    \end{equation}

Let $\delta > 0$ and estimate as follows: \begin{align}
    T_1 &:= \vert \cE_{\tilde{\eps}_n,\mathrm{NL}}^{(k,p)}(v,\eta) - \cE_{\tilde{\eps}_n,\mathrm{NL}}^{(k,p)}(v \circ T_n,\eta) \vert \notag \\
    &\leq \frac{1}{\tilde{\eps}_n^{p + kd}} \int_{\Omega^{k+1}} \ls \prod_{j=1}^{k} \prod_{r=0}^{j-1} \eta\l \frac{\vert x_{j} - x_{r}\vert}{\tilde{\eps}_n} \r \rs \ls \prod_{\ell=0}^k \rho(x_\ell) \rs \notag \\
    &\qquad\qquad\times \left\vert \left\vert v(x_1) - v(x_0) \right\vert^p - \left\vert v_n \circ T_n(x_1) - v_n \circ T_n (x_0) \right\vert^p  \right\vert  \, \dd x_k \cdots \dd x_0 \notag \\
    &\leq \delta \cE_{\tilde{\eps}_n,\mathrm{NL}}^{(k,p)}(v,\eta) + \frac{C_\delta}{\tilde{\eps}_n^{p + kd}} \int_{\Omega^{k+1}} \ls \prod_{j=1}^{k} \prod_{r=0}^{j-1} \eta\l \frac{\vert x_{j} - x_{r}\vert}{\tilde{\eps}_n} \r \rs \ls \prod_{\ell=0}^k \rho(x_\ell) \rs \notag \\
    &\qquad\qquad\times \left\vert v(x_1) - v\circ T_n(x_1) - (v(x_0) - v \circ T_n(x_0)) \right\vert^p  \, \dd x_k \cdots \dd x_0 \label{eq:proofs:limsupDiscrete:equation1} \\
    &\leq \delta \cE_{\tilde{\eps}_n,\mathrm{NL}}^{(k,p)}(v,\eta) + C C_\delta \l \frac{\Vert \Id - T_n \Vert_{\Lp{\infty}}}{\tilde{\eps}_n} \r^p \notag 
\end{align}
where we used \eqref{eq:intro:equationTriangle} for \eqref{eq:proofs:limsupDiscrete:equation1} and the fact that $v \in \Ck{\infty}(\bar{\Omega})$ as well as Assumptions \ref{ass:Main:Ass:W2} and \ref{ass:Main:Ass:M2}. 
We obtain:
\begin{align}
    &\limsup_{n \to \infty} \cE_{n,\eps_n}^{(k,p)}(v_n) \leq  \limsup_{n \to \infty} \l \frac{\tilde{\eps}_n}{\eps_n} \r^{p + kd} \cE_{\tilde{\eps}_n,\mathrm{NL}}^{(k,p)}(v \circ T_n,\eta) \label{eq:proofs:limsupDiscrete:equation5} \\
    & \leq \limsup_{n\to\infty} \l\frac{\tilde{\eps}_n}{\eps_n}\r^{p+kd} \frac{1}{1+\delta} \l \cE_{\tilde{\eps}_n,\mathrm{NL}}^{(k,p)}(v,\eta) + C C_\delta\l\frac{\|\Id-T_n\|_{\Lp{\infty}}}{\tilde{\eps}_n}\r^p\r \label{eq:proofs:limsupDiscrete:equation6} \\
    &\leq \frac{1}{1+\delta} \cE_{\infty}^{(k,p)}(v) \label{eq:proofs:limsupDiscrete:equation7}
\end{align}
where we used \eqref{eq:proofs:limsupDiscrete:NLUpperBound} for \eqref{eq:proofs:limsupDiscrete:equation5}, \eqref{eq:proofs:limsupDiscrete:equation1} for \eqref{eq:proofs:limsupDiscrete:equation6} and the fact that the recovery sequence in Proposition \ref{prop:gammaConvergence:NonlocalLimsup} for $v$ was $v$ for \eqref{eq:proofs:limsupDiscrete:equation7}.
Letting $\delta\to0$ proves~\eqref{eq:proofs:limsupDiscrete:limsup} for $\eta$ in this form.

Next we proceed as in Proposition \ref{prop:proofs:gammaConvergence:liminf} or \cite[Theorem 1.1]{Trillos3}: by assuming that $\eta = \sum_{k=1}^\ell \eta_l$ satisfies Assumption \ref{ass:Main:Ass:W2} where $\eta_l$ are functions of the type considered in the above, we use the linearity of the integral to deduce \eqref{eq:proofs:limsupDiscrete:limsup}; assuming that $\eta$ is a general function satisfying Assumption \ref{ass:Main:Ass:W2}, we approximate $\eta$ by functions of the type considered in the previous step and use the monotone convergence theorem to conclude.
\end{proof}

Using the result of Proposition \ref{prop:proofs:gammaConvergence:limsupDiscrete}, we can prove the next straightforward corollary. The key point to note is that, since we have the same recovery sequence for all $1 \leq k \leq q$, we just apply the subadditivity of $\limsup$ to conclude.  

\begin{corollary}[$\limsup$-inequality for the sum of semi-supervised energies in the well-posed case] \label{cor:proofs:gammaConvergence:limsupDiscrete}
Assume that \ref{ass:Main:Ass:S1}, \ref{ass:Main:Ass:M1}, \ref{ass:Main:Ass:M2}, \ref{ass:Main:Ass:W2}, \ref{ass:Main:Ass:D1}, \ref{ass:Main:Ass:D2} and \ref{ass:Main:Ass:L1} hold. 
Then, $\bbP$-a.s., for every $(\nu,v) \in \TLp{p}(\Omega)$, 
there exists $\{(\nu_n,v_{_n})\}_{n=1}^\infty$ with $(\nu_n,v_{n}) \to (\nu,v)$ in $\TLp{p}(\Omega)$
such that \(
        \limsup_{n \to \infty} \l\cS\cF\r_{n,\eps_n}^{(q,p)}((\nu_n,v_n)) \leq \l\cS\cF\r_{\infty}^{(q,p)}((\nu,v)).
\)
\end{corollary}

The following proofs use arguments from \cite{Slepcev}. For Proposition \ref{prop:proofs:gammaConvergence:liminfSSL}, the sum of all energies $\{\cF_{n,\eps_n}^{(k,p)}\}_{k=1}^q$ has to be considered directly as we plan to use the uniform convergence results for the $k = 1$ case from \cite[Lemma 4.5]{Slepcev}. In Proposition \ref{prop:proofs:gammaConvergence:limsupSSL}, we show that $n\eps_n^p \to \infty$ is the common lower bound for all energies $\{\cF_{n,\eps_n}^{(k,p)}\}_{k=1}^q$ in order for them to converge to the ill-posed continuum objective functions.  

\begin{proposition}[$\liminf$-inequality for the sum of semi-supervised energies in the well-posed case] \label{prop:proofs:gammaConvergence:liminfSSL}
    Assume that \ref{ass:Main:Ass:S1}, \ref{ass:Main:Ass:M1}, \ref{ass:Main:Ass:M2}, \ref{ass:Main:Ass:W2}, \ref{ass:Main:Ass:D1}, \ref{ass:Main:Ass:D2} and \ref{ass:Main:Ass:L1} hold. Assume that $n\eps_n^{p} \to 0$. Then, $\bbP$-a.s., for every $(\nu,v) \in \TLp{p}(\Omega)$ and $\{(\nu_n,v_{_n})\}_{n=1}^\infty$ with $(\nu_n,v_{n}) \to (\nu,v)$ in $\TLp{p}(\Omega)$, we have: 
\begin{equation}\label{eq:proofs:gammaConvergence:liminfSSL:wellPosed}
        \liminf_{n \to \infty} \l\cS\cF\r_{n,\eps_n}^{(q,p)}((\nu_n,v_n)) \geq \l\cS\cF\r_{\infty}^{(q,p)}((\nu,v)).
\end{equation}

\end{proposition}

\begin{proof}

With probability one, we can assume that the conclusions of Proposition \ref{prop:proofs:gammaConvergence:liminf} and \cite[Lemma 4.5]{Slepcev} hold.

In the proof $C>0$ will denote a constant that can be arbitrarily large, is independent of $n$, and that may change from line to line.
    
First, by the same argument as in Proposition \ref{prop:proofs:gammaConvergence:liminf}, we can assume that $$\sup_{n\geq 1} \l\cS\cF\r_{n,\eps_n}^{(q,p)}((\nu_n,v_n)) \leq C$$ and therefore $\nu_n = \mu_n$ and $\nu = \mu$. In particular, we also have that $\cE_{n,\eps_n}^{(1,p)}(v_n) \leq C$ and, by \cite[Lemma 4.5]{Slepcev}, we deduce the existence of a continuous function $\hat{v}$ such that for any $\Omega' \subset \subset \Omega$, $\max_{ \{ i \leq n_{k} \spaceBar x_i \in \Omega' \} } \vert v_{n_{k}}(x_i) - \hat{v}(x_i) \vert \to 0$: this implies that $\hat{v}(x_i) = y_i$ for all $i \leq N$ with probability one. 
We also note that $v = \hat{v}$ (in particular, $v(x_i) = y_i$ for all $i \leq N$) and \eqref{eq:proofs:gammaConvergence:liminfSSL:wellPosed} reduces to proving
\[
\liminf_{n \to \infty} \sum_{k=1}^q \lambda_k \cE_{n,\eps_n}^{(k,p)}(v_n) \geq \sum_{k=1}^q \lambda_k \cE_{\infty}^{(k,p)}(v) = \sum_{k=1}^q \lambda_k \cG_{\infty}^{(k,p)}((\mu,v)).
\]

By Proposition \ref{prop:proofs:gammaConvergence:liminf}, we know that $ \liminf_{n \to \infty}\cE_{n,\eps_n}^{(k,p)}(v_n) \geq \cG_\infty((\mu,v))$ so that:
\begin{align}
    \liminf_{n \to \infty} \sum_{k=1}^q \lambda_k \cE_{n,\eps_n}^{(k,p)}(v_n) &\geq \sum_{k=1}^q \lambda_k \liminf_{n \to \infty} \cE_{n,\eps_n}^{(k,p)}(v_n) \geq \sum_{k=1}^q \lambda_k \cG_{\infty}^{(k,p)}((\mu,v)). \notag 
\end{align}
\end{proof}

\begin{proposition}[$\limsup$-inequality in the ill-posed case] \label{prop:proofs:gammaConvergence:limsupSSL}
Assume that \ref{ass:Main:Ass:S1}, \ref{ass:Main:Ass:M1}, \ref{ass:Main:Ass:M2}, \ref{ass:Main:Ass:W2}, \ref{ass:Main:Ass:D1}, \ref{ass:Main:Ass:D2} and \ref{ass:Main:Ass:L1} hold. Assume that $n \eps_n^{p} \to \infty$. Then, $\bbP$-a.s., for every $(\nu,v) \in \TLp{p}(\Omega)$, 
there exists $\{(\nu_n,v_{_n})\}_{n=1}^\infty$ with $(\nu_n,v_{n}) \to (\nu,v)$ in $\TLp{p}(\Omega)$
such that: \begin{equation}\label{eq:proofs:gammaConvergence:limsupSSL:illPosed}
        \limsup_{n \to \infty} \cF_{n,\eps_n}^{(k,p)}((\nu_n,v_n)) \leq \cG_{\infty}^{(k,p)}((\nu,v)).
\end{equation}
In particular, for $1 \leq k \leq q$ and $v \in \Ck{\infty}(\bar{\Omega})$, we can pick $(\nu_n,v_n) = (\mu_n,\hat{v}_n)$ where $\hat{v}_n(x_i) = y_i$ for $i \leq N$ and $\hat{v}_n = v|_{\Omega_n}$ else.
\end{proposition}

\begin{proof}
    With probability one, we can assume that the conclusions of Proposition \ref{prop:proofs:gammaConvergence:limsupDiscrete} hold.
    
    In the proof $C>0$ will denote a constant that can be arbitrarily large, is independent of $n$, and that may change from line to line.

    We start by noting that \eqref{eq:proofs:gammaConvergence:limsupSSL:illPosed} is trivial if $\cG_{\infty}^{(k,p)}((\nu,v)) = \infty$ so that we assume $\cG_{\infty}^{(k,p)}((\nu,v)) < \infty$ which implies that $\nu= \mu$ and $v \in \Wkp{1}{p}(\Omega)$. Furthermore, we are going to apply \cite[Remark 2.7]{Trillos3}, so it is sufficient to verify \eqref{eq:proofs:gammaConvergence:limsupSSL:illPosed} on a dense subset of $\{\mu\} \times \Wkp{1}{p}(\Omega)$, namely we consider $(\mu,v) \in \{ \mu \} \times \Ck{\infty}(\bar{\Omega})$. We let $\nu_n = \mu_n$ and $\hat{v}_n = v\vert_{\Omega_n}$. 
    
    By repeating the proof of Proposition \ref{prop:proofs:gammaConvergence:limsupDiscrete}, we can show that $(\nu_n,\hat{v}_{n}) \to (\nu,v)$ in $\TLp{p}(\Omega)$ and
    \begin{equation} \label{eq:proofs:gammaConvergence:limsupSSL:equation1}
    \limsup_{n \to \infty} \cE_{n,\eps_n}^{(k,p)}(\hat{v}_n) \leq \cG_{\infty}^{(k,p)}((\mu,v)).
    \end{equation} 
    The subtlety of \eqref{eq:proofs:gammaConvergence:limsupSSL:equation1} compared to \eqref{eq:proofs:limsupDiscrete:limsup} is that $\hat{v}_n$ does not necessarily satisfy $\hat{v}_n(x_i) = y_i$ for all $i \leq N$ since this condition is not imposed on $v$.
    We note that $\Vert \hat{v}_n \Vert_{\Lp{\infty}} \leq C$ since $\hat{v}_n = v\vert_{\Omega_n}$ and $v \in \Ck{\infty}(\bar{\Omega})$.

Define $(\mu_n,v_n)$ with
\[
v_n(x_i) = \begin{cases}
    y_i & \text{if $i \leq N$,}\\
    \hat{v}_n(x_i)& \text{else.}
\end{cases}
\]
Again, we have $\Vert v_n \Vert_{\Lp{\infty}} \leq C$ and, using the arguments of \cite[Proposition 4.24]{weihs2023consistency}, we can show that $(\mu_n, v_n) \to (\mu,v)$ in $\TLp{p}(\Omega)$. 
Since $v_n(x_i) = y_i$ for all $i \leq N$, in order to show \eqref{eq:proofs:gammaConvergence:limsupSSL:illPosed}, it therefore is sufficient to show that
\[
\lim_{n \to \infty} \l \underbrace{\cF_{n,\eps_n}^{(k,p)}(\mu_n,v_n) - \cE_{n,\eps_n}^{(k,p)}(\hat{v}_n)}_{=:T_2} \r = 0.
\]

We estimate as follows:
\begin{align}
    \vert T_2 \vert &\leq \frac{1}{n^{k+1}\eps_n^{p + kd}} \hspace{-3mm} \sum_{i_0,\cdots, i_k = 1}^{n} \eta_{\mathrm{p}}(x_{i_0},\dots,x_{i_k})  \left\vert \vert v_n(x_{i_1}) - v_n(x_{i_0}) \vert^p - \vert \hat{v}_n(x_{i_1}) - \hat{v}_n(x_{i_0}) \vert^p  \right\vert \notag
\end{align}
By definition of $v_n$, for $(i_0,\cdots,i_k) \in S$ where
\[ S:= \{ (i_0,\cdots,i_k) \spaceBar N \leq i_j \leq n \text{ for all $0 \leq j \leq 1$}\}, \] 
the corresponding term in the above sum vanishes. This means that we need to consider all indices in \(
S^c = \{ (i_0,\cdots,i_k) \spaceBar \text{there exists $0 \leq j \leq 1$ such that $1 \leq i_j \leq N$}\}.
\) 
Summing over all sets in $S^c$ therefore yields:
\begin{align}
    &T_2 \leq \frac{1}{n^{k+1}\eps_n^{p + kd}} \sum_{t=0}^1 \sum_{i_t =1}^N \sum_{\substack{i_s = 1 \\s \neq t}}^n \ls \prod_{j=1}^{k} \prod_{r=0}^{j-1} \eta\l \frac{\vert x_{i_j} - x_{i_{r}}\vert}{\eps_n} \r \rs \notag \\
    &\qquad\qquad\times\left\vert \vert v_n(x_{i_1}) - v_n(x_{i_0}) \vert^p - \vert \hat{v}_n(x_{i_1}) - \hat{v}_n(x_{i_0}) \vert^p  \right\vert \notag \\
    &\leq \frac{C}{n\eps_n^{p}}  \sum_{t=0}^1 \sum_{i_t =1}^N \frac{1}{n^k\eps_n^{dk}} \sum_{\substack{i_s = 1 \\s \neq t}}^n  \ls \prod_{j=1}^{k} \prod_{r=0}^{j-1} \eta\l \frac{\vert x_{i_j} - x_{i_{r}}\vert}{\eps_n} \r \rs \label{eq:proofs:gammaConvergence:limsupSSL:equation3}\\
    &\leq \frac{C}{n\eps_n^{p}}  \sum_{t=0}^1 \sum_{i_t =1}^N \frac{1}{n^k\eps_n^{dk}} \sum_{\substack{i_s = 1 \\s \neq t}}^n  \ls \prod_{j=1}^{k} \eta\l \frac{\vert x_{i_j} - x_{i_{j-1}}\vert}{\eps_n}  \r \rs. \notag
\end{align}

For $t \in\{0,1\}$, using $\eta(s) = 0$ for all $|s|>1$, 
\begin{align}
    &\frac{1}{n^k\eps_n^{dk}} \sum_{\substack{i_s = 1 \\s \neq t}}^n \ls \prod_{j=1}^{k} \eta\l \frac{\vert x_{i_j} - x_{i_{j-1}}\vert}{\eps_n}  \r \rs \notag \\
    &\qquad\qquad\leq \frac{\eta(0)^k}{n^k\eps_n^{dk}} \# \{ (i_{0},\cdots,i_{t-1},i_{t+1},\cdots,i_{k}) \spaceBar \vert x_{i_j} - x_{i_{j-1}} \vert < \eps_n \text{ for $0\leq j \leq k$}\} \notag.
\end{align}
Now, for an element in $(i_0,\cdots,i_k) \in \{ (i_{0},\cdots,i_{t-1},i_{t+1},\cdots,i_{k}) \spaceBar \vert x_{i_j} - x_{i_{j-1}} \vert < \eps_n \text{ for $0\leq j \leq k$}\} =: \hat{S}$, we have $x_{i_{t-1}} \in B(x_{i_t},\eps_n)$, $x_{i_{t-2}} \in B(x_{i_{t-1}},\eps_n)$ until $x_{i_0} \in B(x_{i_1},\eps_n)$ as well as $x_{i_{t+1}} \in B(x_{i_t},\eps_n)$, $x_{i_{t+2}} \in B(x_{i_{t+1}}, \eps_n)$ until $x_{i_k} \in B(x_{i_{k-1}},\eps_n)$. 
Hence $x_{i_j}\in B(x_{i_t},k\eps_n)$ for all $j$.
This shows that
\begin{align*}
    \# \hat{S}  \leq \sum_{\substack{z_1,\cdots,z_{k} \in \Omega_n}} \prod_{j=1}^k \one_{B(x_{i_t},bk\eps_n)}(z_j) = \l n\mu_n(B(x_{i_t},k\eps_n)) \r^k.
\end{align*}

Using the latter, we continue estimating:
\begin{align}
    \frac{1}{n^k\eps_n^{dk}} \sum_{\substack{i_s = 1 \\s \neq t}}^n  \prod_{j=1}^{k} \eta\l \frac{\vert x_{i_j} - x_{i_{j-1}}\vert}{\eps_n} \r &\leq \frac{\eta(0)^k}{n^k\eps_n^{dk}} \# \hat{S} \notag \\
    &\leq C \l \eps_n^{-d} \mu_n(x_{i_t},k\eps_n) \r^k \notag \\
    &= C \l \eps_n^{-d} \int_{\Omega} \one_{\{ \vert T_n(x)- x_{i_t} \vert < k \eps_n \}} \rho(x) \, \dd x\r^k \notag \\
    &\leq C \l  \eps_n^{-d} \int_{\Omega} \one_{\{ \vert x - x_{i_t} \vert < k \eps_n - \Vert T_n - \Id \Vert_{\Lp{\infty}} \}} \rho(x) \, \dd x \r^k \notag \\
    &\leq C \l \textrm{Vol}(B(0,1)) \l \frac{k \eps_n - \Vert T_n - \Id \Vert_{\Lp{\infty}}}{\eps_n} \r^d \r^k \label{eq:proofs:gammaConvergence:limsupSSL:equation4} \\
    &\leq C \label{eq:proofs:gammaConvergence:limsupSSL:equation5}
\end{align}
where we used Assumption \ref{ass:Main:Ass:M2} for \eqref{eq:proofs:gammaConvergence:limsupSSL:equation4} and Assumption \ref{ass:Main:Ass:L1} for \eqref{eq:proofs:gammaConvergence:limsupSSL:equation5}.

Inserting \eqref{eq:proofs:gammaConvergence:limsupSSL:equation5} in \eqref{eq:proofs:gammaConvergence:limsupSSL:equation3}, we obtain that \(
    T_2 \leq \frac{C}{n\eps_n^{p}} \notag
\)
from which we deduce that $T_2 \to 0$ and \eqref{eq:proofs:gammaConvergence:limsupSSL:illPosed}.
\end{proof}

The next corollary is the analogue of Corollary~\ref{cor:proofs:gammaConvergence:limsupDiscrete}, and its proof follows the same lines.

\begin{corollary}[$\limsup$-inequality for the sum of semi-supervised energies in the ill-posed case] \label{cor:proofs:gammaConvergence:limsupSSLSum}
Assume that \ref{ass:Main:Ass:S1}, \ref{ass:Main:Ass:M1}, \ref{ass:Main:Ass:M2}, \ref{ass:Main:Ass:W2}, \ref{ass:Main:Ass:D1}, \ref{ass:Main:Ass:D2} and \ref{ass:Main:Ass:L1} hold. Assume that $n \eps_n^{p} \to \infty$. Then, $\bbP$-a.s., for every $(\nu,v) \in \TLp{p}(\Omega)$,
there exists $\{(\nu_n,v_{_n})\}_{n=1}^\infty$ with $(\nu_n,v_{n}) \to (\nu,v)$ in $\TLp{p}(\Omega)$
such that: \begin{equation*}
        \limsup_{n \to \infty} \l\cS\cF\r_{n,\eps_n}^{(q,p)}((\nu_n,v_n)) \leq \l\cS\cG\r_{\infty}^{(q,p)}((\nu,v)).
\end{equation*}
\end{corollary}

We conclude with a lemma summarizing our $\Gamma$-convergence results for our semi-supervised objectives. 

\begin{lemma}[$\Gamma$-convergence of energies] \label{lem:proofs:gammaConvergence:sum}
    Assume that \ref{ass:Main:Ass:S1}, \ref{ass:Main:Ass:M1}, \ref{ass:Main:Ass:M2}, \ref{ass:Main:Ass:W2}, \ref{ass:Main:Ass:D1}, \ref{ass:Main:Ass:D2} and \ref{ass:Main:Ass:L1} hold. If $n \eps_n^{p} \to 0$, then, $\bbP$-a.s, $(\cS\cF)_{n,\eps_n}^{(q,p)}$ $\Gamma$-converges to $(\cS\cF)_{\infty}^{q,p}$ in $\TLp{p}(\Omega)$. If $n \eps_n^{p} \to \infty$, then, $\bbP$-a.s., $(\cS\cF)_{n,\eps_n}^{(q,p)}$ $\Gamma$-converges to $(\cS\cG)_{\infty}^{(q,p)}$ in $\TLp{p}(\Omega)$. 
\end{lemma}

\subsubsection{Proof of Theorem \ref{thm:main}}

\begin{proof}[Proof of Theorem \ref{thm:main}]
With probability one, we can assume that the conclusions of Lemmas \ref{lem:proofs:gammaConvergence:sum} and \ref{lem:proofs:boundedEnergies} hold.

In the proof $C>0$ will denote a constant that can be arbitrarily large, is independent of $n$, and that may change from line to line. 

By Lemma \ref{lem:proofs:boundedEnergies}, there exists $C$ such that $\sup_{n > 0} (\cS\cF)_{n,\eps_n}^{(q,p)}((\mu_n,u_n)) < C$ and in particular, $ \cE_{n,\eps_n}^{(1,p)}(u_n)$ is uniformly bounded. Furthermore, analogously to what is described in the proof of \cite[Theorem 2.1]{Slepcev}, we know that $\sup_{n > 0} \Vert u_n \Vert_{\Lp{\infty}} < C$ with probability 1. We can therefore apply \cite[Proposition 4.4]{Slepcev} to obtain a subsequence $\{n_r\}_{r=1}^\infty$ and $(\mu,u) \in \TLp{p}(\Omega)$ such that $(\mu_{n_r},u_{n_r}) \to (\mu,u)$ in $\TLp{p}(\Omega)$.
\begin{enumerate}
    \item Since $n \eps^{p}_n \to 0$, by \cite[Lemma 4.5]{Slepcev}, we know that $u$ is continuous and, for every $\Omega' \subset \subset \Omega$, we have that $\max_{\{s \leq n_r \spaceBar x_s \in \Omega'\}} \vert u(x_s) - u_{n_r}(x_s) \vert \to 0$ and, with probability 1, $u(x_i) = y_i$ for all $i \leq N$. By Lemma \ref{lem:proofs:gammaConvergence:sum} and Proposition \ref{prop:Back:Gamma:minimizers}, we also have that $(\mu,u)$ is a minimizer of $(\cS\cF)_{\infty}^{(q,p)}$. Finally, by the uniqueness of the minimizer of $(\cS\cF)_{\infty}^{(q,p)}$, we conclude that the whole sequence $(\mu_n,u_n)$ converges to $(\mu,u)$ in $\TLp{p}(\Omega)$ and for every $\Omega' \subset \subset \Omega$, we have that $\max_{\{s \leq n \spaceBar x_s \in \Omega'\}} \vert u(x_s) - u_{n}(x_s) \vert \to 0$.
    \item By Lemma \ref{lem:proofs:gammaConvergence:sum}, Proposition \ref{prop:Back:Gamma:minimizers} and the assumption $n\eps_n^{p}\to \infty$, $(\mu,u)$ is a minimizer of $(\cS\cG)_{\infty}^{(q,p)}$.
\end{enumerate}
\end{proof}

\subsection{Higher-order hypergraph learning}

The proofs in this section are simple corollaries from the results in \cite{weihs2023consistency}. In contrast to the discrete-continuum nonlocal-continuum local decomposition used for the proofs in Section \ref{sec:proofs}, everything in this section relies on spectral convergence results between the discrete Laplace operators $\Delta_{n,\eps_n}$ and its continuum counterpart $\Delta_\rho$. 

For our first result, the proof follows from an application of \cite[Proposition 4.21]{weihs2023consistency}. The key observation is that for any 
$v \in \Ck{\infty}(\bar{\Omega})$ with $v(x_i) = y_i$ for $i \leq N$, we can pick $(\nu_n,v_n) = (\mu_n,v|_{\Omega_n})$ to be a common recovery sequence for $\cJ^{(p_k)}_{n,\Delta_{n,\eps_n^{(k)}}}$ with $1 \leq k \leq q$. This allows us to use the subadditivity of $\limsup$ to deduce the result.

\begin{proposition}[$\limsup$-inequality for the sum of semi-supervised energies in the well-posed case]
    Assume that \ref{ass:Main:Ass:S2}, \ref{ass:Main:Ass:M1}, \ref{ass:Main:Ass:M2}, \ref{ass:Main:Ass:W2}, \ref{ass:Main:Ass:D1} and \ref{ass:Main:Ass:D2} hold. 
    Let $q \geq 1$, $P = \{p_k\}_{k=1}^q \subseteq (0,\infty)$ with $p_1 \leq \cdots \leq p_q$ and $ E_n = \{\eps_n^{(k)}\}_{k=1}^q$ with $\eps_n^{(1)} > \cdots > \eps_n^{(q)}$. Assume that $\eps_n^{(q)}$ satisfies \ref{ass:Main:Ass:L2}, that $\eps_n^{(1)} \to 0$ and that $\rho \in \Ck{\infty}$.
    Then, $\bbP$-a.s., for every $(\nu,v) \in \TLp{2}(\Omega)$, there exists a sequence $\{(\nu_n,v_{_n})\}_{n=1}^\infty$ with $(\nu_n,v_{n}) \to (\nu,v)$ in $\TLp{2}(\Omega)$ and
    \(
        \limsup_{n \to \infty} (\cS\cJ)_{n,E_n}^{(q,P)}((\nu_n,v_n)) \leq (\cS\cJ)_{\infty}^{(q,P)}((\nu,v)).  
    \)
\end{proposition}

The next result is shown analogously to Proposition \ref{prop:proofs:gammaConvergence:liminfSSL}. In particular, one relies on the compactness result \cite[Proposition 4.13]{weihs2023consistency}: we only require that our smallest length-scale $\eps_n^{(q)}$ satisfies the appropriate upper bound and that its associated power $p_q$ scales correctly with the dimension of $\Omega$. Then, the problem reduces to using the superadditivity of $\liminf$ and \cite[Theorem 4.14]{weihs2023consistency}.

\begin{proposition}[$\liminf$-inequality for the sum of semi-supervised energies in the well-posed case]
    Assume that \ref{ass:Main:Ass:S2}, \ref{ass:Main:Ass:M1}, \ref{ass:Main:Ass:M2}, \ref{ass:Main:Ass:W2}, \ref{ass:Main:Ass:D1} and \ref{ass:Main:Ass:D2} hold. 
    Let $q \geq 1$, $P = \{p_k\}_{k=1}^q \subseteq (0,\infty)$ with $p_1 \leq \cdots \leq p_q$ and $ E_n = \{\eps_n^{(k)}\}_{k=1}^q$ with $\eps_n^{(1)} > \cdots > \eps_n^{(q)}$. Assume that $\eps_n^{(q)}$ satisfies \ref{ass:Main:Ass:L2}, that $n \cdot (\eps_n^{(q)})^{p_q/2 - 1/2}$ is bounded, that $p_q > \frac{5}{2}d + 4$ and that $\eps_n^{(1)} \to 0$.
    Then, $\bbP$-a.s., for every sequence $\{(\nu_n,v_{_n})\}_{n=1}^\infty \subseteq \TLp{2}(\Omega)$ with $(\nu_n,v_{n}) \to (\nu,v)$ in $\TLp{2}(\Omega)$, we have
    \(
        \liminf_{n \to \infty} (\cS\cJ)_{n,E_n}^{(q,P)}((\nu_n,v_n)) \geq (\cS\cJ)_{\infty}^{(q,P)}((\nu,v)).
    \)
\end{proposition}

For the next result, we again rely on the fact that \cite[Proposition 4.24]{weihs2023consistency} implies that the same recovery sequence can be chosen for all $\cJ_{n,\Delta_{n,\eps_n^{(k)}}}^{(p_k)}$ with $1 \leq k \leq q$. In particular, we need to assume that all $\eps_n^{(k)}$ satisfy an appropriate lower bound and that their associated powers $p_k$ scale correctly with the dimension of $\Omega$. We then conclude using the subadditivity of $\limsup$. 

\begin{proposition}[$\limsup$-inequality for the sum of semi-supervised energies in the ill-posed case]
    Assume that \ref{ass:Main:Ass:S2}, \ref{ass:Main:Ass:M1}, \ref{ass:Main:Ass:M2}, \ref{ass:Main:Ass:W2}, \ref{ass:Main:Ass:D1} and \ref{ass:Main:Ass:D2} hold. 
    Let $q \geq 1$, $P = \{p_k\}_{k=1}^q \subseteq (0,\infty)$ with $p_1 \leq \cdots \leq p_q$ and $ E_n = \{\eps_n^{(k)}\}_{k=1}^q$ with $\eps_n^{(1)} > \cdots > \eps_n^{(q)}$. Assume that $\rho \in \Ck{\infty}$, that 
    $\eps_n^{(q)}$
    satisfies \ref{ass:Main:Ass:L1}, that $n(\eps_n^{(q)})^{2p_q} \to \infty$ and that $\eps_n^{(1)} \to 0$.
    Then, $\bbP$-a.s., for every $(\nu,v) \in \TLp{2}(\Omega)$, there exists a sequence $\{(\nu_n,v_{_n})\}_{n=1}^\infty$ with $(\nu_n,v_{n}) \to (\nu,v)$ in $\TLp{2}(\Omega)$ and $\limsup_{n \to \infty} (\cS\cJ)_{n,E_n}^{(q,P)}((\nu_n,v_n)) \leq (\cS\cK)_{\infty}^{(q,P)}((\nu,v))$.
\end{proposition}

Summarizing all our previous results and using the superadditivity of $\liminf$ in conjunction with \cite[Proposition 4.22]{weihs2023consistency}, we obtain the following result.

\begin{lemma}[$\Gamma$-convergence of energies] \label{lem:proofs:gammaConvergence:sumHOHL}
    Assume that \ref{ass:Main:Ass:S2}, \ref{ass:Main:Ass:M1}, \ref{ass:Main:Ass:M2}, \ref{ass:Main:Ass:W2}, \ref{ass:Main:Ass:D1}, \ref{ass:Main:Ass:D2} hold. Let $q \geq 1$, $P = \{p_k\}_{k=1}^q \subseteq \bbR$ with $p_1 \leq \cdots \leq p_q$ and $ E_n = \{\eps_n^{(k)}\}_{k=1}^q$ with $\eps_n^{(1)} > \cdots > \eps_n^{(q)}$. Assume that $\rho \in \Ck{\infty}$ and that $\eps_n^{(1)} \to 0$
    \begin{enumerate}
        \item Assume that $\eps_n^{(q)}$ satisfies \ref{ass:Main:Ass:L2}, that $n \cdot (\eps_n^{(q)})^{p_q/2 - 1/2}$ is bounded and that $p_q > \frac{5}{2}d + 4$. Then, $\bbP$-a.s., $(\cS\cJ)_{n,E_n}^{(q,P)}$ $\Gamma$-converges to $(\cS\cJ)_{\infty}^{(q,P)}$ in $\TLp{2}(\Omega)$. 
        \item Assume that $\eps_n^{(q)}$ satisfies \ref{ass:Main:Ass:L1} as well as $n(\eps_n^{(q)})^{2p_q} \to \infty$. Then, $\bbP$-a.s., $(\cS\cJ)_{n,E_n}^{(q,P)}$ $\Gamma$-converges to $(\cS\cK)_{\infty}^{(q,P)}$ in $\TLp{2}(\Omega)$. 
    \end{enumerate}    
\end{lemma}

\begin{proof}[Proof of Theorem \ref{thm:main:HOHL}]
    The proof is analogous to the proof of Theorem \ref{thm:main}.
    
    In particular, for the well-posed case, we use \cite[Proposition 4.17 and Lemma 4.25]{weihs2023consistency} to obtain a uniform bound on the $\Lp{2}$-norms of $u_n$. Then, uniform and $\TLp{2}$-convergence of a subsequence of $u_n$ to some continuous $u$ follows from \cite[Proposition 4.13 and Theorem 4.14]{weihs2023consistency}. By the uniqueness of the minimizer, Lemma \ref{lem:proofs:gammaConvergence:sumHOHL} and Proposition \ref{prop:Back:Gamma:minimizers}, the result follows.

    For the ill-posed case, convergence of a subsequence in $\TLp{2}$ to some $u$ follows from \cite[Theorem 4.14]{weihs2023consistency}. Again, Lemma \ref{lem:proofs:gammaConvergence:sumHOHL} and Proposition \ref{prop:Back:Gamma:minimizers} allow us to conclude. 
\end{proof}

\section{Numerical Experiments} \label{sec:discussion}

Multiscale Laplace learning has demonstrated strong empirical performance on point clouds, outperforming many existing graph-based semi-supervised learning methods~\cite{Merkurjev}. Since we approximate HOHL through this framework, our evaluation emphasizes sensitivity analyses. In particular, we test whether choosing increasing powers \(p_\ell=\ell\) improves performance over constant-exponent settings, which would support the intuition that higher-order regularization is most beneficial at finer scales.

We report experiments on four datasets of varying size and difficulty: iris~\cite{misc_iris_53}, digits~\cite{misc_optical_recognition_of_handwritten_digits_80}, Salinas A~\cite{SalinasDatasetCCWINTCO}, and MNIST~\cite{LeCun1998}. Notation is summarized in Table~\ref{tab:qj:terminology}.

\paragraph{\(q\)-Experiment}
For \(1 \leq q \leq 5\), we test \eqref{eq:multiscale} using \(q\) scales \(\eps^{(1)} \geq \eps^{(2)} \geq \cdots \geq \eps^{(q)}\), forming Laplacians \(\Delta_{n,\eps^{(\ell)}}\) for \(1 \leq \ell \leq q\). We vary both the weight coefficients \(\lambda_\ell\) and the powers \(p_\ell\) through the following configurations:
\begin{itemize}
    \item Coefficients: constant (CC) \(\lambda_\ell=1\); slowly increasing (SC) \(\lambda_\ell=\ell\); quickly increasing (QC) \(\lambda_\ell=\ell^2\).
    \item Powers: constant (CP) \(p_\ell=1\); increasing (IP) \(p_\ell=\ell\).
\end{itemize}
The goal is to quantify how performance changes with the number of scales \(q\) and with these coefficient/power choices.

\paragraph{\(j\)-Experiment}
For \(1 \leq q \leq 3\), we form Laplacians \(\Delta_{n,\eps^{(\ell)}}\) at scales \(\eps^{(1)} \geq \eps^{(2)} \geq \eps^{(3)}\). We then test a family of weight schedules indexed by \(1 \leq j \leq 4\):
\[
\lambda_1 = 1, \qquad \lambda_2 = j^2, \qquad \lambda_3 = (j+1)^2,
\]
with fixed powers \(p_1=1\), \(p_2=2\), \(p_3=3\) (denoted VQC(\(q\)), with \(q\in\{2,3\}\)). The goal is to evaluate sensitivity to the relative magnitudes of \(\lambda_\ell\) for a fixed number of scales.

\paragraph{Graph construction and evaluation protocol}
We always use the full dataset as nodes in the graph construction. For MNIST, we use the same data embedding as in \cite{98b487bb64994720ba648f45328e2135}. 
For the smaller datasets, iris and digits, we use $\eps$-graphs with weights $w_{\eps^{(\ell)},ij} = \exp\left(\frac{-4|x_i - x_j|^2}{\l\eps^{(\ell)}\r^2} \right)$. To illustrate that our model works with different (hyper)graph types and in order to speed-up computations, we rely on $k$-nearest neighbors ($k$NN) graphs for the larger/high-dimensional datasets (naturally substituting the sequence $\eps^{(1)} \geq \eps^{(2)} \geq \eps^{(3)} \geq \eps^{(4)} \geq \eps^{(5)}$ with $k^{(1)} \geq k^{(2)} \geq k^{(3)} \geq k^{(4)} \geq k^{(5)}$) with weights $w_{k^{(\ell)},ij} = \exp\left(\frac{-4|x_i - x_j|^2}{d_{k^{(\ell)}}(x_i)^2} \right)$ where $d_{k^{(\ell)}}(x_i)$ denotes the distance from $x_i$ to its $k^{(\ell)}$-th nearest neighbor. 

\begin{remark}[Use of nearest-neighbor graphs in the experiments]
\label{rem:knn-graphs}
Our theoretical analysis is formulated for \(\varepsilon\)-graphs, where the
neighborhood radius is prescribed explicitly. In the numerical section, however,
we sometimes replace this construction by \(k\)-nearest neighbor graphs. The
reason is practical: prescribing a fixed number of neighbors gives more uniform
finite-sample connectivity and is often preferable for larger or
higher-dimensional datasets.

This substitution should be interpreted as a change in how the local scale is
chosen, not as a change in the underlying locality principle. An
\(\varepsilon\)-graph around a point \(x\) has typical degree proportional to
\(
    n\rho(x)\varepsilon^d,
\)
up to constants depending on the kernel and the ambient dimension. Thus a
nearest-neighbor graph implicitly selects a data-dependent radius
\[
    \varepsilon_k(x) \asymp
    \left(\frac{k}{n\rho(x)}\right)^{1/d}.
\]
The radius is smaller where the sample is dense and larger where the sample is
sparse. In this sense, the \(k\)-nearest neighbor construction may be viewed as
an adaptive-bandwidth analogue of the fixed-radius construction used in the
analysis.

Both graph models are standard in discrete-to-continuum studies and in
graph-based semi-supervised learning; see, for example,
\cite{Trillos3,GARCIATRILLOS2018239,98b487bb64994720ba648f45328e2135,CALDER2022123}. We therefore
use the analytically cleaner \(\varepsilon\)-graph model for the proofs and the
more numerically robust nearest-neighbor construction in the experiments where
appropriate.
\end{remark}

Each experiment is run over 100 trials. In each trial, we re-sample labeled points that are used as fixed constraints (as in \eqref{eq:SSLObjectiveSum}). We report mean accuracy and standard deviation (in brackets), in percentages. Labeling rates range from \(0.02\) to \(0.8\); for Salinas A, the rate parameter ranges from 1 to 100 and denotes the number of labeled points per class.

\paragraph{Baselines}
Theorems~\ref{thm:pointwiseConsistency}, \ref{thm:main}, and
\ref{thm:main:HOHL} show that, in the continuum limit, our hypergraph
objectives induce effective graph-type regularization operators
(see Figure~\ref{fig:classificationAlgorithms}). However, sharing a continuum limit does not imply equivalence at fixed sample size. Accordingly, our experiments assess whether the hypergraph/multiscale structure yields empirical
benefits beyond standard graph regularization in practice.
We compare against standard and state-of-the-art graph SSL methods, including Laplace learning~\cite{LapRef} (which corresponds to \eqref{eq:multiscale} with \(q=1\) and therefore serves as a single-scale ablation), Poisson learning~\cite{98b487bb64994720ba648f45328e2135}, Fractional Laplace (FL) learning~\cite{weihs2023consistency} (which can be viewed as a complementary ablation of \eqref{eq:multiscale}: it keeps a single scale \(q=1\) but replaces the standard Laplacian regularizer by a higher-order), with \(s=2\) and \(s=3\) (iris only), Weighted Nonlocal Laplacian (WNLL)~\cite{shi2017weighted}, \(p\)-Laplace learning~\cite{flores2019algorithms}, Random Walk (RW)~\cite{zhou2004lazy}, Centered Kernel (CK)~\cite{Mai}, Sparse LP (SLP)~\cite{jung2016semi}, and the Properly Weighted Graph Laplacian~\cite{calderSlepcev}. 

Lastly, we note that no methodical hyperparameter optimization has been performed for the choice of \(\eps^{(\ell)}\) and \(k^{(\ell)}\). We specify the scale sequences \(\{\eps^{(\ell)}\}\) and \(\{k^{(\ell)}\}\) in the captions of the corresponding tables.

\paragraph{Results summary}

Highlights are shown in Tables~\ref{tab:digitsQ:small}, \ref{tab:digitsJ:small}, \ref{tab:salinasaQ:small}, \ref{tab:salinasaJ:small}, \ref{tab:mnistQ:small}, and \ref{tab:mnistJ:small}, with complete results in Section~\ref{sec:appendix:numerical}. Overall, we observe:
\begin{itemize}
    \item \textbf{Competitive performance.} Across datasets, our multiscale/hypergraph models frequently outperform the baselines, with the largest gains appearing when the labeling rate is sufficiently high and when \(q\) and/or \(j\) are at least moderate (e.g., for Salinas A when using at least \(2\) labels per class; see Table~\ref{tab:salinasaJ:small}; and for MNIST for \(q=3\) at labeling rates above \(0.02\); see Table~\ref{tab:mnistJ:small}.)
    \item \textbf{Benefit of increasing powers.} Configurations with increasing powers (IP) consistently outperform constant-power variants of \eqref{eq:discussion:higherOrder}. This supports the choice \(p_\ell=\ell\) over \(p_\ell=1\), indicating that our HOHL surrogate \eqref{eq:multiscale} benefits from stronger higher-order regularization at finer scales.
    \item \textbf{Moderate \(q\) often suffices.} In many cases, \(q=2\) or \(q=3\) already yields most of the performance gains (
    Tables~\ref{tab:irisQ:full}, \ref{tab:digitsQ:full}, \ref{tab:salinasaQ:full}). Larger \(q\) increases computation time because each additional scale requires constructing a new Laplacian (or hyperedge set) and incurs additional cost from matrix products in the objective. A similar saturation effect is observed in the \(j\)-experiment: increasing the weights \(\lambda_\ell\) is beneficial up to a point beyond which performance improvements taper off (see also Tables~\ref{tab:irisJ:full} and \ref{tab:digitsJ:full}). For this reason, we restrict \(q\) and \(j\) to modest ranges for larger datasets.
\end{itemize}

\section{Conclusion} \label{sec:conclusion}

We establish continuum limits for variational hypergraph-based semi-supervised
learning and show that, despite their multiway construction, classical
pairwise-aggregation hypergraph energies converge to first-order
density-weighted Sobolev \((\Wkp{1}{p})\) regularization. A central technical
component of the analysis is the pointwise consistency theorem for the
hypergraph Euler--Lagrange operators, where the product-type hyperedge weights
lead to multi-index kernel statistics and nontrivial density-dependent
constants. This analysis yields a weighted \(p\)-Laplacian continuum operator
and, together with the variational convergence results, gives sharp well-/ill-posedness
regimes clarifying when the method produces nontrivial label propagation versus
collapse to trivial smoothing.

We also introduce HOHL, a multiscale model that penalizes powers of Laplacians
on hypergraph-induced subgraphs, as a higher-order regularization framework
designed to go beyond the first-order limiting behavior of classical hypergraph
energies. For its surrogate on metric-space point clouds,
we prove \(\Gamma\)-convergence to higher-order Sobolev-type energies and
characterize the corresponding well-/ill-posedness thresholds. Experiments on
standard SSL benchmarks support the practical benefits of the proposed
higher-order, multiscale regularization. Broader extensions to non-geometric
hypergraphs and further computational aspects are discussed in
\cite{weihs2025HOHL}.

A natural direction for future work is to study normalized versions \cite{hoffmann2020spectral} of the HOHL
operators. Replacing the skeleton Laplacians by
symmetric or random-walk normalized variants would change how degree
heterogeneity and sampling-density effects enter the regularizer, and may be
particularly relevant for datasets that induce hypergraphs with highly
non-uniform degrees. Finally, our continuum framework provides a basis for
organizing regularization-based SSL methods
(Figure~\ref{fig:classificationAlgorithms}) and suggests extensions to other
hypergraph constructions
\cite{fazeny,pmlr-v80-li18e,hgLearningPractice,TVHg,hgPLaplacianGeometric}.

\bibliographystyle{siamplain}
\bibliography{references}

\FloatBarrier

\begin{table}[H]
\centering
\small
\renewcommand{\arraystretch}{1.5}
\setlength{\tabcolsep}{8pt}
\begin{tabularx}{\linewidth}{>{\bfseries}p{5.5cm} Y Y}
\toprule
\textbf{Term / Abbreviation} & \textbf{$q$-Experiment} & \textbf{$j$-Experiment} \\
\midrule
Aim of experiment & Analysis of HOHL as a function of maximum powers $q$ & Analysis of HOHL as a function of coefficients $\lambda_\ell$ \\
\midrule
$\ell$ & Index over scales $1 \leq \ell \leq q$ & Same meaning \\
$q$ & Number of Laplacians $1 \leq q \leq 5$ & Number of Laplacians $2 \leq q \leq 3$ \\
$j$ & — & Coefficients $\lambda_\ell$ are a function of parameter $1\leq j \leq 4$ \\
\midrule
$\eps^{(\ell)}$ & Scale for $\ell$-th $\eps$-graph Laplacian & Same meaning \\
$k^{(\ell)}$ & Scale for $\ell$-th $k$NN-graph Laplacian & Same meaning \\
$\Delta_{n,\eps^{(\ell)}}$ & $\ell$-th $\eps$-graph Laplacian & Same meaning \\
\midrule
$\lambda_\ell$ & Fixed or increasing ($1$ or $\ell$ or $\ell^2$) & Varies with $j$: $\lambda_1 = 1, \lambda_2 = j^2, \lambda_3 = (j+1)^2$ \\
$p_\ell$ & Constant or increasing ($1$ or $\ell$) & Increasing: $p_\ell = \ell$ \\ \\
\midrule
CC & $\lambda_\ell = 1$ & — \\
SC & $\lambda_\ell = \ell$ & — \\
QC & $\lambda_\ell = \ell^2$ & — \\
CP & $p_\ell = 1$ & — \\
IP & $p_\ell = \ell$ & $p_l = \ell$ \\
VQC($q$) & — & For $q=2$: $\lambda_1 = 1$, $\lambda_2 = j^2$. For $q=3$: $\lambda_1 = 1$, $\lambda_2 = j^2$, $\lambda_3 = (j+1)^2$. 
\\
\bottomrule
\end{tabularx}
\caption{Terminology used in the $q$- and $j$-experiments.}
\label{tab:qj:terminology}
\end{table}

\begin{table}[H]
\sc
\begin{center}
\resizebox{\textwidth}{!}{
\begin{tabular}{cccccccccc}
\toprule
Rate & Laplace & Poisson & \textbf{IP-QC} & \textbf{CP-QC} & \textbf{IP-SC} & \textbf{CP-SC} & \textbf{IP-CC} & \textbf{CP-CC} \\
\midrule
0.02 & 11.96 (4.03) & \textbf{78.81} (2.98) & 22.57 (9.14) & 15.02 (5.8) & 20.91 (8.57) & 15.46 (5.4) & 18.96 (7.82) & 13.79 (5.43) \\
0.05 & 19.35 (6.62) & \textbf{84.87} (1.63) & 61.81 (7.17) & 37.24 (7.55) & 58.56 (7.5) & 31.54 (9.11) & 52.93 (7.74) & 24.84 (7.85) \\
0.10 & 42.87 (7.4) & \textbf{87.13} (1.12) & 81.57 (3.51) & 60.04 (7.23) & 80.78 (3.71) & 54.66 (7.07) & 78.93 (4.26) & 50.4 (6.7) \\
0.20 & 68.58 (4.38) & 87.61 (0.94) & \textbf{89.12} (1.5) & 85.79 (2.17) & 89.06 (1.5) & 82.83 (2.57) & 88.82 (1.47) & 79.01 (3.19) \\
0.30 & 82.1 (2.02) & 87.58 (0.74) & 91.74 (0.87) & 90.98 (1.02) & \textbf{91.75} (0.87) & 89.44 (1.15) & 91.73 (0.88) & 87.57 (1.28) \\
0.50 & 88.3 (1.11) & 87.85 (0.78) & 93.87 (0.71) & 93.39 (0.81) & \textbf{93.89} (0.7) & 92.45 (0.86) & \textbf{93.89} (0.71) & 91.37 (0.92) \\
0.80 & 89.73 (1.43) & 87.88 (1.42) & \textbf{94.98} (0.99) & 94.33 (1.13) & 94.96 (0.98) & 93.3 (1.2) & 94.91 (0.96) & 92.18 (1.21) \\
\bottomrule
\end{tabular}
}
\caption{Accuracy of various SSL methods on the digits dataset for the $q$-experiment with $q=3$. We pick $\eps^{(\ell)} = 100^{2-\ell}$ for $1 \leq \ell \leq 5$. Proposed methods are in bold.}
\label{tab:digitsQ:small}
\end{center}
\end{table}

\begin{table}[H]
\sc
\begin{center}
\resizebox{\textwidth}{!}{
\begin{tabular}{cccccccccccc}
\toprule
Rate & Laplace & Poisson & WNLL & Properly & $p$-Lap & RW & CK & \textbf{IP-VQC (2)} & \textbf{IP-VQC (3)} \\
\midrule
0.02 & 12.20 (4.75) & \textbf{79.00} (2.75) & 67.07 (6.07) & 78.29 (3.14) & 77.83 (3.23) & 30.17 (11.33) & 60.00 (4.17) & 25.16 (9.35) & 24.25 (9.65) \\
0.05 & 20.42 (7.03) & \textbf{84.61} (1.72) & 69.20 (4.38) & 83.11 (2.08) & 82.50 (2.19) & 32.00 (5.96) & 66.19 (3.73) & 62.69 (6.84) & 61.96 (6.85) \\
0.10 & 41.62 (6.59) & 86.73 (1.36) & 80.73 (3.07) & \textbf{87.67} (1.45) & 87.45 (1.51) & 31.95 (5.56) & 71.98 (2.73) & 81.51 (3.66) & 81.25 (3.61) \\
0.20 & 68.47 (4.79) & 87.61 (0.99) & 86.21 (1.53) & 89.04 (0.97) & 88.93 (1.00) & 40.94 (4.75) & 78.25 (1.53) & \textbf{89.49} (1.09) & 89.41 (1.10) \\
0.30 & 82.17 (2.32) & 87.62 (0.80) & 88.00 (1.20) & 89.81 (0.87) & 89.74 (0.89) & 44.89 (5.34) & 82.11 (0.81) & \textbf{91.83} (0.86) & 91.79 (0.83) \\
0.50 & 88.18 (1.00) & 87.84 (0.96) & 89.04 (1.00) & 89.98 (1.00) & 89.94 (0.99) & 37.33 (2.51) & 85.67 (0.98) & \textbf{93.79} (0.91) & 93.77 (0.90) \\
0.80 & 89.65 (1.49) & 87.88 (1.40) & 89.68 (1.45) & 89.97 (1.42) & 89.97 (1.41) & 33.93 (1.16) & 88.34 (1.39) & 94.91 (1.01) & \textbf{94.93} (1.00) \\
\bottomrule
\end{tabular}
}
\caption{Accuracy of various SSL methods on the digits dataset for the $j$-experiment with $j=2$. We pick $\eps^{(\ell)} = 100^{2-\ell}$ for $1 \leq \ell \leq 5$. Proposed methods are in bold.}
\label{tab:digitsJ:small}
\end{center}
\end{table}

\begin{table}[H]
\sc
\begin{center}
\resizebox{\textwidth}{!}{
\begin{tabular}{cccccccccc}
\toprule
Rate & Laplace & Poisson & \textbf{IP-QC} & \textbf{CP-QC} & \textbf{IP-SC} & \textbf{CP-SC} & \textbf{IP-CC} & \textbf{CP-CC} \\
\midrule
1   & 58.08 (8.37) & 57.12 (7.32) & \textbf{60.98} (7.28) & 59.25 (7.54) & 59.73 (7.89) & 59.00 (7.85) & 58.81 (8.09) & 58.67 (8.07) \\
2   & 66.85 (5.49) & 57.32 (6.44) & \textbf{67.75} (5.42) & 67.45 (5.44) & 67.26 (5.60) & 67.32 (5.46) & 66.85 (5.77) & 67.22 (5.46) \\
5   & 73.46 (2.31) & 56.83 (5.31) & 73.59 (2.36) & \textbf{73.65} (2.35) & 73.61 (2.42) & 73.63 (2.34) & 73.58 (2.48) & 73.59 (2.27) \\
10  & 75.86 (1.82) & 56.08 (5.31) & 76.09 (1.88) & \textbf{76.21} (1.81) & 76.15 (1.84) & 76.14 (1.83) & 76.15 (1.84) & 76.06 (1.83) \\
20  & 77.61 (1.15) & 56.20 (4.25) & \textbf{78.52} (1.51) & 78.14 (1.18) & 78.42 (1.43) & 78.02 (1.17) & 78.26 (1.31) & 77.87 (1.15) \\
50  & 79.60 (0.88) & 56.44 (3.93) & \textbf{80.95} (0.91) & 80.37 (0.89) & 80.83 (0.94) & 80.18 (0.90) & 80.64 (0.93) & 80.00 (0.90) \\
100 & 80.86 (0.57) & 56.06 (2.98) & \textbf{82.47} (0.70) & 81.82 (0.56) & 82.33 (0.62) & 81.61 (0.56) & 82.10 (0.61) & 81.35 (0.55) \\
\bottomrule
\end{tabular}
}
\caption{Accuracy of various SSL methods on the Salinas A dataset for the $q$-experiment with $q=3$. We pick $k^{(1)} = 50$, $k^{(2)} = 30$, $k^{(3)} = 20$ and $k^{(4)} = 10$. Proposed methods are in bold.}
\label{tab:salinasaQ:small}
\end{center}
\end{table}

\begin{table}[H]
\sc
\begin{center}
\resizebox{\textwidth}{!}{
\begin{tabular}{cccccccccccc}
\toprule
Rate & Laplace & Poisson & WNLL & Properly & $p$-Lap & RW & CK & \textbf{IP-VQC (2)} & \textbf{IP-VQC (3)} \\
\midrule
1    & 59.28 (8.54) & 58.31 (6.46) & \textbf{64.13} (6.05) & 64.10 (6.04) & 60.26 (5.44) & 63.10 (5.14) & 28.50 (5.98) & 61.88 (7.11) & 62.23 (6.78) \\
2    & 66.82 (5.35) & 56.76 (7.03) & 67.54 (5.04) & 67.42 (5.10) & 64.65 (5.13) & 66.94 (4.76) & 33.05 (6.65) & 67.53 (5.07) & \textbf{67.68} (5.12) \\
5    & 73.74 (2.71) & 55.56 (5.89) & 73.42 (3.07) & 73.14 (3.15) & 72.26 (3.07) & 73.70 (2.60) & 46.37 (5.32) & \textbf{73.94} (2.84) & 73.86 (2.85) \\
10   & 75.88 (1.67) & 56.49 (5.18) & 75.81 (1.73) & 75.32 (1.81) & 74.80 (1.85) & 75.98 (1.73) & 55.54 (4.27) & \textbf{76.23} (1.76) & 76.14 (1.81) \\
20   & 77.44 (1.37) & 55.99 (4.62) & 78.23 (1.40) & 77.56 (1.58) & 77.51 (1.61) & 77.99 (1.22) & 66.04 (3.10) & 78.34 (1.31) & \textbf{78.40} (1.37) \\
50   & 79.58 (0.94) & 56.69 (4.19) & 80.87 (0.90) & 80.21 (0.93) & 80.36 (0.88) & 79.10 (0.85) & 75.21 (1.76) & 80.87 (0.98) & \textbf{80.98} (1.01) \\
100  & 80.96 (0.73) & 55.83 (2.75) & 82.10 (0.63) & 81.88 (0.66) & 82.12 (0.61) & 79.27 (0.70) & 79.82 (1.02) & 82.41 (0.72) & \textbf{82.53} (0.75) \\
\bottomrule
\end{tabular}
}
\caption{Accuracy of various SSL methods on the Salinas A dataset for the $j$-experiment with $j=2$. We pick $k^{(1)} = 50$, $k^{(2)} = 30$, $k^{(3)} = 20$ and $k^{(4)} = 10$. Proposed methods are in bold.}
\label{tab:salinasaJ:small}
\end{center}
\end{table}

\begin{table}[H]
\sc
\begin{center}
\resizebox{\textwidth}{!}{
\begin{tabular}{cccccccccc}
\toprule
Rate & Laplace & Poisson & \textbf{IP-QC} & \textbf{CP-QC} & \textbf{IP-SC} & \textbf{CP-SC} & \textbf{IP-CC} & \textbf{CP-CC} \\
\midrule
0.02 & 97.07 (0.07) & 96.80 (0.06) & \textbf{97.36} (0.07) & 97.29 (0.08) & 97.26 (0.07) & 97.24 (0.07) & 97.19 (0.07) & 97.19 (0.07) \\
0.05 & 97.37 (0.05) & 96.85 (0.04) & \textbf{97.64} (0.06) & 97.59 (0.06) & 97.56 (0.05) & 97.54 (0.06) & 97.50 (0.05) & 97.48 (0.06) \\
0.10 & 97.58 (0.04) & 96.85 (0.04) & \textbf{97.82} (0.04) & 97.77 (0.04) & 97.76 (0.04) & 97.74 (0.04) & 97.70 (0.04) & 97.69 (0.04) \\
0.20 & 97.81 (0.04) & 96.87 (0.04) & \textbf{98.01} (0.04) & 97.98 (0.04) & 97.97 (0.04) & 97.95 (0.04) & 97.92 (0.04) & 97.91 (0.04) \\
0.30 & 97.92 (0.04) & 96.87 (0.05) & \textbf{98.10} (0.04) & 98.07 (0.04) & 98.07 (0.04) & 98.05 (0.04) & 98.02 (0.04) & 98.02 (0.05) \\
0.50 & 98.08 (0.06) & 96.87 (0.08) & \textbf{98.24} (0.06) & 98.21 (0.06) & 98.21 (0.06) & 98.19 (0.06) & 98.18 (0.06) & 98.17 (0.06) \\
0.80 & 98.25 (0.09) & 96.90 (0.12) & \textbf{98.38} (0.09) & 98.36 (0.10) & 98.37 (0.09) & 98.34 (0.09) & 98.34 (0.09) & 98.32 (0.09) \\
\bottomrule
\end{tabular}
}
\caption{Accuracy of various SSL methods on the MNIST dataset for the $q$-experiment with $q=3$. We pick $k^{(\ell)} = 30 - (\ell-1)\cdot 10$ for $1 \leq \ell \leq 3$. Proposed methods are in bold.}
\label{tab:mnistQ:small}
\end{center}
\end{table}

\begin{table}[H]
\sc
\begin{center}
\resizebox{\textwidth}{!}{
\begin{tabular}{cccccccccccc}
\toprule
Rate & Laplace & Poisson & WNLL & Properly & $p$-Lap & RW & CK & \textbf{IP-VQC (2)} & \textbf{IP-VQC (3)} \\
\midrule
0.02 & 97.06 (0.09) & 96.79 (0.07) & 96.55 (0.09) & 94.76 (0.17) & 94.48 (0.17) & 97.15 (0.10) & 95.34 (0.16) & 97.31 (0.09) & \textbf{97.34} (0.09) \\
0.05 & 97.37 (0.06) & 96.85 (0.05) & 97.20 (0.05) & 94.49 (0.12) & 95.49 (0.10) & 97.37 (0.07) & 96.46 (0.08) & 97.62 (0.05) & \textbf{97.64} (0.05) \\
0.10 & 97.59 (0.04) & 96.86 (0.04) & 97.58 (0.05) & 95.59 (0.08) & 96.88 (0.06) & 97.45 (0.05) & 97.18 (0.06) & 97.80 (0.04) & \textbf{97.82} (0.04) \\
0.20 & 97.80 (0.04) & 96.87 (0.04) & 97.86 (0.04) & 97.08 (0.05) & 97.71 (0.04) & 97.50 (0.05) & 97.68 (0.04) & 97.99 (0.04) & \textbf{98.00} (0.04) \\
0.30 & 97.92 (0.05) & 96.87 (0.05) & 97.98 (0.05) & 97.61 (0.06) & 97.88 (0.05) & 97.51 (0.05) & 97.88 (0.05) & \textbf{98.10} (0.05) & \textbf{98.10} (0.05) \\
0.50 & 98.08 (0.06) & 96.86 (0.06) & 98.11 (0.06) & 98.01 (0.06) & 98.07 (0.06) & 97.51 (0.06) & 98.09 (0.06) & \textbf{98.24} (0.06) & \textbf{98.24} (0.05) \\
0.80 & 98.22 (0.10) & 96.87 (0.14) & 98.23 (0.10) & 98.22 (0.11) & 98.23 (0.10) & 97.52 (0.13) & 98.24 (0.11) & \textbf{98.37} (0.11) & \textbf{98.37} (0.11) \\
\bottomrule
\end{tabular}
}
\caption{Accuracy of various SSL methods on the MNIST dataset for the $j$-experiment with $j=2$. We pick $k^{(\ell)} = 30 - (\ell-1)\cdot 10$ for $1 \leq \ell \leq 3$. Proposed methods are in bold.}
\label{tab:mnistJ:small}
\end{center}
\end{table}

\appendix

\section{Additional proofs}

In this section, we collect additional proofs and technical lemmas that complement the results of Section~\ref{sec:proofs}.

\subsection{Pointwise convergence of hypergraph learning}

\subsubsection{Auxiliary results}

The following several auxiliary results are useful in the proof of Theorem \ref{thm:pointwiseConsistency}. First, we recall the McDiarmid/Azuma inequality \cite{McDiarmid_1989}. 
\begin{theorem}[McDiarmid/Azuma Inequality]\label{thm:mcdiarmid}
    Let $X_1,\dots,X_n$ be iid random variables satisfying $\vert X_i \vert \leq M$ almost surely. Let $Y_n=f(X_1,\dots,X_n)$ for some function $f$. If there exists $b >0$ such that $f$ satisfies 
    \[
    \vert f(x_1,\dots,x_i,\dots,x_n) - f(x_1,\dots,\tilde{x_i},\dots,x_n) \vert \leq b
    \]
    for all $x_i$ and $\tilde{x}_i$, $1\leq i\leq n$, then for all $t > 0$,
    \[
    \bbP(\vert Y_n - \E(Y_n) \vert \geq t ) \leq 2 \exp\left( -\frac{t^2}{2nb^2}\right).
    \]
\end{theorem}

The next result is a straight-forward counting lemma. 

\begin{lemma} \label{lem:combinatorics}
    Let $$S^{(n,k)}(i) = \# \{(\alpha_1,\dots,\alpha_k) \in \{1,\dots,n\}^k \, | \, \exists 1 \leq \ell \leq k \text{ such that } \alpha_\ell = i\}.$$ Then, for $1\leq i \leq n$, $S^{(n,k)}(i) = n^{k-1} + (n-1) S^{(n,k-1)}(i)$ and $S^{(n,k)}(i) \leq \mathcal{O}(n^{k-1})$.
\end{lemma}

\begin{proof}
    Let $(\alpha_1,\dots,\alpha_k) \in \{1,\dots,n\}^k$. If we fix $\alpha_1 = i$, then there exists $n^{k-1}$ tuples of the form $(i,\dots,\alpha_k)$. Now, if $\alpha_1 = j \neq i$, there exist $S^{(n,k-1)}(i)$ tuples of the form $(j,\alpha_2,\dots,\alpha_k)$ that contain at least one $i$. Since, $j$ can take $n-1$ values, we conclude that $S^{(n,k)}(i) = n^{k-1} + (n-1) S^{(n,k-1)}(i)$. The second claim can be proven simply by induction. 
\end{proof}

By induction and Taylor's expansion, we also establish the following lemma.

\begin{lemma}[Product identity] \label{lem:product}
    Let $\rho \in \mathrm{C}^2(\bbR^d)$. Then, for $k \geq 1$, we have
    \[
    \prod_{\ell = 1}^k \rho(x_{i_0} + \eps_n z_\ell) = \rho(x_{i_0})^k + \eps_n \rho(x_{i_0})^{k-1} \nabla\rho(x_{i_0})(z_1 + \cdots + z_k) + \mathcal{O}(\eps_n^2)
    \]
    for $x_{i_0},z_1,\dots,z_k\in \bbR^d$ and $\eps_n \in \bbR$.
\end{lemma}

Finally, we recall a lemma from \cite{weihs2023discreteToContinuum}.

\begin{lemma}[Asymptotics of domain of integration] \label{lem:domain}
    Assume that $\Omega \subset \bbR^d$ is a bounded open domain. Let $\eps_n > 0$ be a sequence that tends to $0$, $\Omega'$ be compactly contained in $\Omega$ and $C \subset \bbR^d$ be a compact subset. Then, for $n$ large enough, for all $x_{i_0} \in \Omega'$, the set $S_{\eps_n}(x_{i_0}) = \{z \in \bbR^d \, | \, x_{i_0} + \eps_n z \in \Omega \} \cap C$ is equal to $C$.  
\end{lemma}

\subsubsection{Equivalent representation of the continuum Laplacian} \label{sec:equivalentSupplement}

The first result in this section establishes the rotational invariance of the distribution of interest in Lemma \ref{lem:identity1}; the second result makes explicit the key properties of the geometric construction used in Lemma~\ref{lem:identity2}.

\begin{lemma}[Radial marginal] \label{lem:radialMarginal}
Assume that Assumption \ref{ass:Main:Ass:W2} holds. Let $(Z_1,\dots,\allowbreak Z_k)$ be a random vector in $(\mathbb R^d)^k$ with distribution $\bbQ$ defined through the density
\[
f(z_1,\dots,z_k)
= \frac{1}{\mathcal Z}\,
 \widetilde\eta_p(z_1,\dots,z_k),
\qquad
\mathcal Z := \int_{(\mathbb R^d)^k} \widetilde\eta_p(z)\, \dd z.
\]
Then, the marginal law of $Z_1$ is rotation-invariant. 
\end{lemma}

\begin{proof}
Let $Q\in O(d)$ be any orthogonal matrix.  Orthogonality implies
that $\Vert Qx \Vert = \Vert x \Vert$ for all $x\in\mathbb R^d$, and more generally
\[
\Vert Qx - Qy \Vert = \Vert x-y \Vert
\]
for all $x,y\in\mathbb R^d.$
Therefore each factor in $\widetilde\eta_p$ is invariant under the
simultaneous rotation
\[
(z_1,\dots,z_k)\mapsto(Qz_1,\dots,Qz_k).
\]
Indeed,
\[
\eta(\Vert Qz_s \Vert) = \eta(\Vert z_s \Vert),
\qquad
\eta(\Vert Qz_j - Qz_r \Vert) = \eta(\Vert z_j - z_r \Vert),
\]
so
\[
\widetilde\eta_p(Qz_1,\dots,Qz_k)
= \widetilde\eta_p(z_1,\dots,z_k)
\]
for all $z_1,\dots,z_k \in \bbR^d$.
Since the Jacobian determinant of a rotation is $1$, it follows that
the probability density $f$ satisfies
\[
f(Qz_1,\dots,Qz_k) = f(z_1,\dots,z_k).
\]
Thus the law of $(Z_1,\dots,Z_k)$ is rotation-invariant under
simultaneous rotations of all coordinates:
\[
(Z_1,\dots,Z_k) \overset{d}{=}
(QZ_1,\dots,QZ_k).
\]
By \cite[Proposition 4.1.1]{Bryc1995}, this implies that every marginal of $\bbQ$ is also rotation-invariant. 
\end{proof}

\begin{proof}[Proof of Lemma \ref{lem:reflections}]

\begin{enumerate}
    \item $R_{z_1}$ is linear by definition. We check that $R_{z_1}$ is an isometry and has the expected geometric action.
For any $y\in\mathbb R^d$, decompose
\[
y = (y\cdot v)\,v + y_\perp,
\]
where $y_\perp := y - (y\cdot v)v$ satisfies $y_\perp\cdot v = 0$. Then
\begin{align}
R_{z_1}(y)
&= y - 2(y\cdot v)\,v \notag  \\
&= (y\cdot v)v + y_\perp - 2(y\cdot v)v \notag \\
&= - (y\cdot v)v + y_\perp. \label{eq:Raction}
\end{align}
Thus, $R_{z_1}$ flips the component along $v$ and preserves the orthogonal component, which is precisely the reflection across the hyperplane $\{y : y\cdot v = 0\}$. 

Moreover, by the orthogonality of $(y\cdot v)\,v$ and $y_\perp$, we have
\begin{align*}
\|R_{z_1}(y)\|^2 
&= \|-(y\cdot v)v + y_\perp\|^2 \\
&= \|y_\perp\|^2 + (y\cdot v)^2 \\
&= \|(y\cdot v)v + y_\perp\|^2 \\
&= \|y\|^2,
\end{align*}
so that $R_{z_1}$ is an isometry.

\item 
$S_{z_1}$ fixes every point of the hyperplane $H_{z_1}$.
Indeed, if $y\in H_{z_1}$, then $(y-m)\cdot v = 0$, and hence
\[
R_{z_1}(y-m)
= (y-m) - 2\bigl((y-m)\cdot v\bigr)v
= y-m.
\]
It follows that
\[
S_{z_1}(y)
= m + R_{z_1}(y-m)
= m + (y-m)
= y.
\]

For a general point $y\in\mathbb R^d$, the vector $y-m$ has the orthogonal
decomposition
\[
y-m = \bigl((y-m)\cdot v\bigr)v + (y-m)_\perp,
\qquad (y-m)_\perp\cdot v = 0.
\]
Using the reflection identity~\eqref{eq:Raction}, we obtain
\[
R_{z_1}(y-m)
= -(y-m)\cdot v\, v + (y-m)_\perp,
\]
so $R_{z_1}$ reverses the normal component $\bigl((y-m)\cdot v\bigr)v$ and
preserves the tangential component $(y-m)_\perp$.

Geometrically, $R_{z_1}$ is the reflection across the hyperplane
\[
H_0 := \{y \in \mathbb R^d : y\cdot v = 0\}
\]
The hyperplane
\[
H_{z_1} := \{y \in \mathbb R^d : (y-m)\cdot v = 0\}
\]
is simply the translation of $H_0$ by the vector $m$.  Therefore, to
obtain the reflection across $H_{z_1}$, we must conjugate $R$ by this
translation, which yields the affine map
\[
S_{z_1}(y) = m + R_{z_1}(y-m).
\]
Thus $S_{z_1}$ is precisely the affine reflection across $H_{z_1}$ (see
Figure~\ref{fig:translation}).

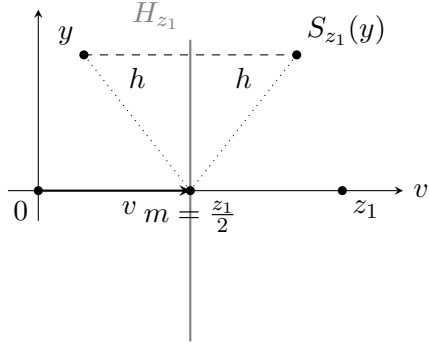
\begin{figure}[h]
\centering
\begin{tikzpicture}[scale=2,>=stealth]

  \coordinate (O)  at (0,0);      
  \coordinate (z1) at (2,0);      
  \coordinate (m)  at (1,0);      

  \coordinate (y)   at (0.3,0.9);
  \coordinate (Sy)  at (1.7,0.9);

  \draw[gray,thick] (1,-1) -- (1,1) node[above left] {$H_{z_1}$};

  \draw[->,thin] (-0.2,0) -- (2.4,0) node[right] {$v$};
  \draw[->,thin] (0,-0.2) -- (0,1.2);

  \fill (O)  circle (0.03) node[below left] {$0$};
  \fill (z1) circle (0.03) node[below right] {$z_1$};
  \fill (m)  circle (0.03) node[below] {$m = \tfrac{z_1}{2}$};

  \fill (y)  circle (0.03) node[above left] {$y$};
  \fill (Sy) circle (0.03) node[above right] {$S_{z_1}(y)$};

  \draw[->,thick] (O) -- (1,0) node[below,pos=0.6] {$v$};

  \draw[dashed] (y) -- (Sy);
  \draw[dotted] (y) -- (m);
  \draw[dotted] (Sy) -- (m);

  \node at (0.65,0.75) {$h$};
  \node at (1.35,0.75) {$h$};

\end{tikzpicture}
\caption{Geometric interpretation of the affine reflection $S_{z_1}$ across the hyperplane $H_{z_1}$ orthogonal to $v$ and passing through $m = z_1/2$. The points $y$ and $S_{z_1}(y)$ are symmetric with respect to $H_{z_1}$.} \label{fig:translation}
\end{figure}

\item For any $y,r\in\mathbb R^d$,
\begin{align}
S_{z_1}(y) - S_{z_1}(r)
&= \bigl(m + R_{z_1}(y-m)\bigr) - \bigl(m + R_{z_1}(r-m)\bigr) \notag \\
&= R_{z_1}(y-m) - R_{z_1}(r-m) \notag \\
&= R_{z_1}\bigl((y-m) - (r-m)\bigr) \label{eq:linearity} \\
&= R_{z_1}(y-r), \notag
\end{align}
where we used the linearity of $R_{z_1}$ for \eqref{eq:linearity}. Since $R_{z_1}$ is an isometry by part 1 of the lemma,
\[
\|S_{z_1}(y) - S_{z_1}(r)\| = \|R_{z_1}(y-r)\| = \|y-r\|.
\]

\item We have 
\begin{align}
R_{z_1}(-m)
&= -m - 2\bigl((-m)\cdot v\bigr)v \notag \\
&= -m - 2\Bigl(-\frac{z_1}{2} \cdot \frac{z_1}{\Vert z_1 \Vert} \Bigr) \frac{z_1}{\Vert z_1 \Vert} \notag \\
&= -m + z_1 \notag \\
&= m. \label{eq:m}
\end{align}
Therefore, 
\[
S_{z_1}(0) = m + R_{z_1}(0-m) = m + R_{z_1}(-m) = 2m = z_1.
\]

Similarly, by linearity of $R$ and \eqref{eq:m},
\[
S_{z_1}(z_1) = m + R_{z_1}(z_1-m) = m + R_{z_1}(m) = m - R_{z_1}(-m) = 0.
\]

\item We compute as follows:
\begin{align}
\|S_{z_1}(y)\|
&= \|S_{z_1}(y) - S_{z_1}(z_1)\| \label{eq:part4}\\
&= \|y - z_1\| \label{eq:isometryS}
\end{align}
where we used part 4 of the lemma for \eqref{eq:part4}, and part 3 of the lemma for \eqref{eq:isometryS}. Similarly,
\begin{align}
\|S_{z_1}(y) - z_1\|
&= \|S_{z_1}(y) - S_{z_1}(0)\| \label{eq:part4:2}\\
&= \|y - 0\| \label{eq:isometryS:2}\\
&= \|y\| \notag
\end{align}
where we used part 4 of the lemma for \eqref{eq:part4:2}, and part 3 of the lemma for \eqref{eq:isometryS:2}. 
\end{enumerate}
\end{proof}

\subsection{$\Gamma$-convergence of hypergraph learning} \label{sec:gammaConvergenceHypergraph}

\subsubsection{\texorpdfstring{$\Gamma$}{Gamma}-convergence of the nonlocal energies} \label{subsec:supplementary:nonlocal}

For $v:\Omega \mapsto \bbR$ and $\eps > 0$ we define the nonlocal energies
\[
\cE_{\eps,\mathrm{NL}}^{(k,p)}(v,\eta) = \frac{1}{\eps^{p + kd}} \hspace{-0.8mm}  \int_{\Omega^{k+1}} \hspace{-1.8mm}  \ls \prod_{j=1}^{k} \prod_{r=0}^{j-1} \eta\l \frac{\vert x_{j} - x_{r}\vert}{\eps} \r \rs \hspace{-1.6mm} \left\vert v(x_1) - v(x_0) \right\vert^p  \prod_{\ell = 0}^k \rho(x_\ell) \, \dd x_k \cdots \dd x_0
\]
which are useful intermediary quantities when going from the discrete setting to the continuum one. In this section, by re-adapting the results in \cite{Trillos3}, our aim is to prove the $\Gamma$-convergence of our nonlocal energies to the local ones in the continuum. We start with a few technical lemmas used in the subsequent results. 

\begin{lemma}[Integral identity] \label{lem:proofs:integralIdentity}
Assume that \ref{ass:Main:Ass:S1} and \ref{ass:Main:Ass:W2} hold. For $k \geq 1$, we have
\begin{equation*}
\frac{1}{\eps_n^{p + dk}} \int_{\Omega^{k+1}} \ls \prod_{j=1}^{k} \prod_{r=0}^{j-1} \eta\l \frac{\vert x_{j} - x_{r}\vert}{\eps_n} \r \rs \vert x_1 - x_0 \vert^{2p}   \, \dd x_k \cdots \dd x_0 = O(\eps_n^p).
\end{equation*}
\end{lemma}

\begin{proof}
In the proof $C>0$ will denote a constant that can be arbitrarily large, is independent of $n$ and that may change from line to line.

By using the change of variables $z_j = (x_j - x_0)/\eps_n$ for $1 \leq j \leq k$, we obtain that $(x_j - x_r)/\eps_n = z_j - z_r$ for $1 \leq r < j \leq k$. By the latter, 
    \begin{align}
        T_1 &:= \frac{1}{\eps_n^{p + dk}} \int_{\Omega^{k+1}} \ls \prod_{j=1}^{k} \prod_{r=0}^{j-1} \eta\l \frac{\vert x_{j} - x_{r}\vert}{\eps_n} \r \rs \vert x_1 - x_0 \vert^{2p}   \, \dd x_k \cdots \dd x_0 \notag \\
        &= \frac{\eps_n^{2p}}{\eps_n^{p}} \int_\Omega \int_{\{ z_j \spaceBar x_0 + \eps_n z_j  \in \Omega \}} \vert z_1 \vert^{2p} \ls  \prod_{s=1}^k \eta(\vert z_s \vert) \rs \ls \prod_{j=1}^{k} \prod_{r=1}^{j-1} \eta\l \vert z_{j} - z_{r} \vert \r \rs \, \dd z_k \cdots \dd z_1 \dd x_0 \notag \\
        &\leq C \eps_n^{p} \int_{(\bbR^d)^k} \vert z_1 \vert^{2p} \ls  \prod_{s=1}^k \eta(\vert z_s \vert) \rs \ls \prod_{j=1}^{k} \prod_{r=1}^{j-1} \eta\l \vert z_{j} - z_{r} \vert \r \rs \, \dd z_k \cdots \dd z_1 \notag \\
        &=O(\eps_n^p) \notag
    \end{align}
    where the last equality follows from Assumption \ref{ass:Main:Ass:W2}.
\end{proof}

\begin{lemma}[Product identities]
Let $\rho:\bbR^d \mapsto \bbR$ be a Lipschitz function that is bounded above. For $x_0, z_1, \cdots,z_k \in \bbR^d$ and $k \geq 1$, we have the following identities:
\begin{equation} \label{eq:proofs:productIdentities1}
    \left\vert \prod_{r=1}^k \rho\l z_r + x_0 \r - \rho(x_0)^{k} \right\vert \leq C(\rho) \sum_{r=1}^k \vert z_r \vert
\end{equation}
and 
\begin{equation} \label{eq:proofs:productIdentities2}
    \left\vert \prod_{r=0}^k \rho\l x_r + z  \r - \prod_{r=0}^k \rho(x_r) \right\vert \leq C(\rho) \left\vert z \right\vert
\end{equation}
for constants $C(\rho)$ only depending on $\rho$.
\end{lemma}

\begin{proof}
We only show how to derive \eqref{eq:proofs:productIdentities1} as the proof of \eqref{eq:proofs:productIdentities2} is similar.

We proceed by induction. For $k = 1$, $\vert \rho(x_0 + z_1) - \rho(x_0) \vert \leq \textrm{Lip}(\rho) \vert z_1 \vert$. Now assume that \eqref{eq:proofs:productIdentities1} holds for $k-1$. We compute as follows:
\begin{align}
\left\vert \prod_{r=1}^k \rho\l z_r + x_0 \r - \rho(x_0)^{k} \right\vert &\leq \left\vert \prod_{r=1}^k \rho\l z_r + x_0 \r - \rho(x_0)\prod_{r=1}^{k-1} \rho\l z_r + x_0 \r \right\vert \notag \\
&+ \left\vert \rho(x_0)\prod_{r=1}^{k-1} \rho\l z_r + x_0 \r - \rho(x_0)^k \right\vert \notag \\
&= \left\vert \prod_{r=1}^{k-1} \rho\l z_r + x_0 \r  \right\vert \left\vert \rho\l z_k + x_0 \r - \rho(x_0) \right\vert \notag \\
&+\vert \rho(x_0) \vert \left\vert \prod_{r=1}^{k-1} \rho\l z_r+ x_0 \r - \rho(x_0)^{k-1} \right\vert \notag \\
&\leq \Vert \rho \Vert_{\Lp{\infty}}^{k-1} \textrm{Lip}(\rho) \left\vert z_k \right\vert + \Vert \rho \Vert_{\Lp{\infty}}  C(\rho) \sum_{r=1}^{k-1} \vert z_r \vert. \notag 
\end{align}
\end{proof}

\begin{lemma}[Pointwise convergence of nonlocal energies]\label{lem:proofs:pointwiseConvergenceNonLocal}
Assume that \ref{ass:Main:Ass:S1}, \ref{ass:Main:Ass:M1}, \ref{ass:Main:Ass:M2} and \ref{ass:Main:Ass:W2} hold. Let $\{v_{\eps_n}\}$ be a sequence of functions in $\Ck{2}(\bbR^d)$ such that 
\begin{equation}\label{eq:proofs:pointwiseConvergenceNonlocal:boundedness}
\sup_{n \in \bbN} \{ \Vert \nabla v_{\eps_n} \Vert_{\Lp{\infty}(\bbR^d)} + \Vert \nabla^{2} v_{\eps_n} \Vert_{\Lp{\infty}(\bbR^d)}\} < \infty.
\end{equation}
Suppose that $\rho$ is a positive Lipschitz function and that  $\nabla v_{\eps_n} \to \nabla v^*$ in $\Lp{p}(\Omega)$ for some $v^* \in \Ck{2}(\bbR^d)$. Then, 
\begin{equation} \label{eq:proofs:pointwiseConvergenceNonlocal:convergence}
\lim_{n \to \infty} \cE_{\eps_n,\mathrm{NL}}^{(k,p)}(v_{\eps_n},\eta) = \cE_{\infty}^{(k,p)}(v^*). 
\end{equation}
\end{lemma}

\begin{proof}
In the proof $C>0$ will denote a constant that can be arbitrarily large, is independent of $n$ and that may change from line to line.

For a function $v \in \Ck{2}(\bbR^d)$ and $x_0,x_1\in \Omega$, we have:
\begin{align}
v(x_1) - v(x_0) &= \nabla v(x_0) \cdot (x_1 - x_0) + (x_1 - x_0)^T \nabla^2 v(c) (x_1 - x_0). \notag
\end{align}
for some constant $c$ depending on $x_0$ and $x_1$. Now, define
\begin{align*}
&H_{\eps_n}(v) \\
&= \frac{1}{\eps_n^{p + dk}} \int_{\Omega^{k+1}} \ls \prod_{j=1}^{k} \prod_{r=0}^{j-1} \eta\l \frac{\vert x_{j} - x_{r}\vert}{\eps_n} \r \rs  \vert \nabla v(x_0) \cdot (x_1 - x_0) \vert^{p} \ls \prod_{\ell = 0}^k \rho(x_\ell) \rs \, \dd x_k \cdots \dd x_0.
\end{align*}
We note that by Assumption \ref{ass:Main:Ass:W2}, we have $H_{\eps_n}(v) \leq C \Vert \nabla v \Vert_{\Lp{\infty}}$. 
Then, we estimate as follows for $\delta > 0$:
\begin{align}
    &T_1 :=\vert \cE_{\eps_n,\mathrm{NL}}^{(k,p)}(v_{\eps_n},\eta) - H_{\eps_n}(v_{\eps_n}) \vert \notag \\
    &\leq \frac{C C_\delta}{\eps_n^{p + dk}} \int_{\Omega^{k+1}} \ls \prod_{j=1}^{k} \prod_{r=0}^{j-1} \eta\l \frac{\vert x_{j} - x_{r}\vert}{\eps_n} \r \rs \notag \\
    &\times \vert v_{\eps_n}(x_1) - v_{\eps_n}(x_0) - \nabla v_{\eps_n}(x_0) \cdot (x_1 - x_0) \vert^p   \, \dd x_k \cdots \dd x_0 + \delta H_{\eps_n}(v_{\eps_n}) \label{eq:proofs:pointwiseConvergenceNonlocal:inequalityDelta} \\
    &\leq \frac{C_\delta \Vert \nabla^2 v_{\eps_n} \Vert_{\Lp{\infty}(\bbR^d)}}{\eps_n^{p + dk}} \int_{\Omega^{k+1}} \ls \prod_{j=1}^{k} \prod_{r=0}^{j-1} \eta\l \frac{\vert x_{j} - x_{r}\vert}{\eps_n} \r \rs \vert x_1 - x_0 \vert^{2p}   \, \dd x_k \cdots \dd x_0 \notag \\
    &+ \delta H_{\eps_n}(v_{\eps_n}) \notag \\
    &= C_\delta O(\eps_n^p) + \delta H_{\eps_n}(v_{\eps_n}) \label{eq:proofs:pointwiseConvergenceNonlocal:integral}
\end{align}
where we used \eqref{eq:intro:equationTriangle} and Assumption \ref{ass:Main:Ass:M2} for \eqref{eq:proofs:pointwiseConvergenceNonlocal:inequalityDelta} as well as Lemma \ref{lem:proofs:integralIdentity} for \eqref{eq:proofs:pointwiseConvergenceNonlocal:integral}.

Next, we define 
\begin{align}
\tilde{H}_{\eps_n}(v) &= \frac{1}{\eps_n^{p + dk}} \int_{\Omega} \int_{\substack{\{z_j \spaceBar x_0 + z_j \in \Omega \} \\ \text{ for all $1 \leq j \leq k$}}} \ls  \prod_{s=1}^k \eta \l \frac{\vert z_s \vert}{\eps_n} \r \rs \ls \prod_{j=1}^{k} \prod_{r=1}^{j-1} \eta\l \frac{\vert z_{j} - z_{r} \vert}{\eps_n} \r \rs  \notag \\
&\times \left\vert \nabla v(x_0) \cdot z_1 \right\vert^p \  \rho(x_0)^{k+1} \, \dd z_k \cdots \dd z_1 \dd x_0. \notag
\end{align}
We note that with the change of variables $z_j = x_j - x_{0}$ for $1\leq j \leq k$, we have
\begin{align}
H_{\eps_n}(v) &= \frac{1}{\eps_n^{p + dk}} \int_{\Omega} \int_{\substack{\{z_j \spaceBar x_0 + z_j \in \Omega \} \\ \text{ for all $1 \leq j \leq k$}}} \ls  \prod_{s=1}^k \eta \l \frac{\vert z_s \vert}{\eps_n} \r \rs \ls \prod_{j=1}^{k} \prod_{r=1}^{j-1} \eta\l \frac{\vert z_{j} - z_{r} \vert}{\eps_n} \r \rs  \notag \\
&\times  \left\vert \nabla v(x_0) \cdot z_1 \right\vert^p  \rho(x_0) \prod_{t = 1}^{k} \rho\l z_t +  x_0\r \, \dd z_k \cdots \dd z_1 \dd x_0. \notag
\end{align}
This leads us to 
\begin{align}
    &\vert H_{\eps_n}(v_{\eps_n}) - \tilde{H}_{\eps_n}(v_{\eps_n}) \vert \leq \frac{C \Vert \nabla v_{\eps_n} \Vert_{\Lp{\infty}} \Vert \rho \Vert_{\Lp{\infty}} }{\eps_n^{p + dk}} \int_{\Omega} \int_{\substack{\{z_j \spaceBar x_0 + z_j \in \Omega \} \\ \text{ for all $1 \leq j \leq k$}}} \vert z_1 \vert^p \notag \\
    &\times \ls  \prod_{s=1}^k \eta \l \frac{\vert z_s \vert}{\eps_n} \r \rs \ls \prod_{j=1}^{k} \prod_{r=1}^{j-1} \eta\l \frac{\vert z_{j} - z_{r} \vert}{\eps_n} \r \rs \cdot \left\vert \prod_{t = 1}^{k} \rho\l  z_t +  x_0\r  - \rho(x_0)^{k} \right\vert \, \dd z_k \cdots \dd z_1 \dd x_0 \notag \\
    &\leq \frac{C }{\eps_n^{p + dk}} \int_{\Omega} \int_{\substack{\{z_j \spaceBar x_0 + z_j \in \Omega \} \\ \text{ for all $1 \leq j \leq k$}}} \vert z_1 \vert^p \cdot \ls  \prod_{s=1}^k \eta \l \frac{\vert z_s \vert}{\eps_n} \r \rs \ls \prod_{j=1}^{k} \prod_{r=1}^{j-1} \eta\l \frac{\vert z_{j} - z_{r} \vert}{\eps_n} \r \rs \notag \\ 
    &\times \sum_{r=1}^k \vert z_r \vert \, \dd z_k \cdots \dd z_1 \dd x_0 \label{eq:proofs:pointwiseConvergenceNonlocal:lemma2}\\
    &\leq C \eps_n \int_{\Omega} \int_{\substack{\{\tilde{z}_j \spaceBar \eps_n \vert \tilde{z}_j \vert \leq \mathrm{diam}(\Omega) \} \\ \text{ for all $1 \leq j \leq k$}}} \vert \tilde{z}_1 \vert^p \cdot \ls  \prod_{s=1}^k \eta \l \vert \tilde{z}_s \vert \r \rs \ls \prod_{j=1}^{k} \prod_{r=1}^{j-1} \eta\l \vert \tilde{z}_{j} - \tilde{z}_{r} \vert \r \rs \notag \\
    &\times \sum_{r=1}^k \vert \tilde{z}_r \vert \, \dd \tilde{z}_k \cdots \dd \tilde{z}_1 \dd x_0  \label{eq:proofs:pointwiseConvergenceNonlocal:changeOfVariables2}\\
    &= O(\eps_n) \label{eq:proofs:pointwiseConvergenceNonlocal:HHfinal}
\end{align}
where we used \eqref{eq:proofs:productIdentities1} and \eqref{eq:proofs:pointwiseConvergenceNonlocal:boundedness} for \eqref{eq:proofs:pointwiseConvergenceNonlocal:lemma2}, the change of variables $\tilde{z}_j = z_j/\eps_n$ for \eqref{eq:proofs:pointwiseConvergenceNonlocal:changeOfVariables2} and Assumption \ref{ass:Main:Ass:W2}.

We define
\begin{align}
\bar{H}_{\eps_n}(v) &= \frac{1}{\eps_n^{p + dk}} \int_{\Omega} \int_{\substack{\{z_j \spaceBar x_0 + z_j \notin \Omega \} \\ \text{ for any $1 \leq j \leq k$}}} \ls  \prod_{s=1}^k \eta \l \frac{\vert z_s \vert}{\eps_n} \r \rs \ls \prod_{j=1}^{k} \prod_{r=1}^{j-1} \eta\l \frac{\vert z_{j} - z_{r} \vert}{\eps_n} \r \rs  \notag \\
&\times \left\vert \nabla v(x_0) \cdot z_1 \right\vert^p \  \rho(x_0)^{k+1} \, \dd z_k \cdots \dd z_1 \dd x_0. \notag
\end{align}
For the latter, we have:
\begin{align}
    &\bar{H}_{\eps_n}(v_{\eps_n}) \notag \\
    &\leq \frac{C}{\eps_n^{p + dk}} \hspace{-1.5mm} \int_{\Omega} \int_{\substack{\{z_j \spaceBar x_0 + z_j \notin \Omega \} \\ \text{ for any $1 \leq j \leq k$}}} \hspace{-2mm} \vert z_1 \vert^p  \ls  \prod_{s=1}^k \eta \l \frac{\vert z_s \vert}{\eps_n} \r \rs \hspace{-2mm} \ls \prod_{j=1}^{k} \prod_{r=1}^{j-1} \eta\l \frac{\vert z_{j} - z_{r} \vert}{\eps_n} \r \rs \, \dd z_k \cdots \dd z_1 \dd x_0 \notag \\
    &= C  \int_{\Omega} \int_{\substack{\{\tilde{z}_j \spaceBar x_0 + \eps_n \tilde{z}_j \notin \Omega \} \\ \text{ for any $1 \leq j \leq k$}}} \vert \tilde{z}_1 \vert^p \cdot \ls  \prod_{s=1}^k \eta \l \vert \tilde{z}_s \vert \r \rs \ls \prod_{j=1}^{k} \prod_{r=1}^{j-1} \eta\l \vert \tilde{z}_{j} - \tilde{z}_{r} \vert \r \rs \, \dd \tilde{z}_k \cdots \dd \tilde{z}_1 \dd x_0 \label{eq:proofs:pointwiseConvergenceNonlocal:changeOfVariables3} \\
    &\leq C  \int_{\Omega} \int_{\substack{\{\tilde{z}_j \spaceBar \vert \tilde{z}_j \vert \geq \frac{\dist(x_0,\partial \Omega)}{\eps_n} \} \\ \text{ for any $1 \leq j \leq k$}}} \vert \tilde{z}_1 \vert^p \cdot \ls  \prod_{s=1}^k \eta \l \vert \tilde{z}_s \vert \r \rs \ls \prod_{j=1}^{k} \prod_{r=1}^{j-1} \eta\l \vert \tilde{z}_{j} - \tilde{z}_{r} \vert \r \rs \, \dd \tilde{z}_k \cdots \dd \tilde{z}_1 \dd x_0 \notag
\end{align}
where we used the change of variables $\tilde{z}_j = z_j/\eps_n$ for \eqref{eq:proofs:pointwiseConvergenceNonlocal:changeOfVariables3}.
Now, by using the dominated convergence and 
Assumption \ref{ass:Main:Ass:W2}, we get that
\begin{equation} \label{eq:proofs:pointwiseConvergenceNonlocal:HbarFinal}
    \bar{H}_{\eps_n}(v_{\eps_n}) = o(1).
\end{equation}

We continue by defining:
\begin{align}
    &\hat{H}_{\eps_n}(v) := \bar{H}_{\eps_n}(v) + \tilde{H}_{\eps_n}(v) \notag\\
    &= \frac{1}{\eps_n^{p + dk}} \int_{\Omega} \int_{(\bbR^d)^k} \ls  \prod_{s=1}^k \eta \l \frac{\vert z_s \vert}{\eps_n} \r \rs \ls \prod_{j=1}^{k} \prod_{r=1}^{j-1} \eta\l \frac{\vert z_{j} - z_{r} \vert}{\eps_n} \r \rs  \notag \\
    &\times  \left\vert \nabla v(x_0) \cdot z_1 \right\vert^p \rho(x_0)^{k+1} \, \dd z_k \cdots \dd z_1 \dd x_0 \notag \\
&=\int_{\Omega} \int_{(\bbR^d)^k} \ls  \prod_{s=1}^k \eta \l \vert \tilde{z}_s \vert \r \rs \ls \prod_{j=1}^{k} \prod_{r=1}^{j-1} \eta\l \vert \tilde{z}_{j} - \tilde{z}_{r} \vert \r \rs \vert \nabla v(x_0) \cdot \tilde{z}_1 \vert^p \rho(x_0)^{k+1} \, \dd \tilde{z}_k \cdots \dd \tilde{z}_1 \dd x_0 \label{eq:proofs:pointwiseConvergenceNonlocal:changeOfVariables4}
\end{align}
where we used the change of variables $\tilde{z}_j = z_j/\eps_n$ for \eqref{eq:proofs:pointwiseConvergenceNonlocal:changeOfVariables4}. 
We also have
\begin{align*}
&\hat{H}_{\eps_n}(v) = \frac{1}{\eps_n^{p + dk}} \int_{\Omega} \int_{\bbR^{dk}} \ls  \prod_{s=1}^k \eta \l \frac{\vert z_s \vert}{\eps_n} \r \rs \ls \prod_{j=1}^{k} \prod_{r=1}^{j-1} \eta\l \frac{\vert z_{j} - z_{r} \vert}{\eps_n} \r \rs  \notag \\
&\times  \left\vert \nabla v(x_0) \cdot z_1 \right\vert^p \rho(x_0)^{k+1} \, \dd z_k \cdots \dd z_1 \dd x_0 \\
 & = \int_{\Omega} \int_{\bbR^{dk}} \ls  \prod_{s=1}^k \eta \l \vert \tilde{z}_s \vert \r \rs \ls \prod_{j=1}^{k} \prod_{r=1}^{j-1} \eta\l \vert \tilde{z}_{j} - \tilde{z}_{r} \vert \r \rs  \left\vert \nabla v(x_0) \cdot \tilde{z}_1 \right\vert^p \rho(x_0)^{k+1} \, \dd \tilde{z}_k \cdots \dd \tilde{z}_1 \dd x_0 \\
 & = \cE_\infty^{(k,p)}(v).
\end{align*}
For $\delta > 0$, we continue by noting that
\begin{align}
    &\left\vert \hat{H}(v_{\eps_n}) - \cE_\infty^{(k,p)}(v^*) \right\vert \notag \\
    &\qquad \leq \delta \cE_\infty^{(k,p)}(v^*)  + C C_{\delta} \int_{\Omega} \int_{(\bbR^d)^k} \vert z_1 \vert^p \cdot \ls  \prod_{s=1}^k \eta \l \vert z_s \vert \r \rs \ls \prod_{j=1}^{k} \prod_{r=1}^{j-1} \eta\l \vert z_{j} - z_{r} \vert \r \rs    \notag \\
    & \qquad \qquad \times \vert \nabla v_{\eps_n}(x_0) - \nabla v(x_0) \vert^p \, \dd z_k \cdots \dd z_1 \dd x_0 \label{eq:proofs:pointwiseConvergenceNonlocal:inequalityDelta2} \\
    &\qquad =\delta \cE_\infty^{(k,p)}(v^*) + C C_{\delta} C(\eta) \int_\Omega \vert \nabla v_{\eps_n}(x_0)- \nabla v^*(x_0)\vert^p \, \dd x_0 \notag \\
    & \qquad = \delta \cE_\infty^{(k,p)}(v^*) \label{eq:proofs:pointwiseConvergenceNonlocal:limit} + C_\delta o(\eps_n)
\end{align}
where $C(\eta) = \int_{(\bbR^d)^k} \vert z_1 \vert^p \cdot \ls  \prod_{s=1}^k \eta \l \vert z_s \vert \r \rs \ls \prod_{j=1}^{k} \prod_{r=1}^{j-1} \eta\l \vert z_{j} - z_{r} \vert \r \rs \, \dd z_k \cdots \dd z_1$ which is finite by Assumption \ref{ass:Main:Ass:W2} and 
where we used \eqref{eq:intro:equationTriangle} for \eqref{eq:proofs:pointwiseConvergenceNonlocal:inequalityDelta2} as well as the fact that $\nabla v_{\eps_n} \to \nabla v^*$ in $\Lp{p}$ for \eqref{eq:proofs:pointwiseConvergenceNonlocal:limit}.

We conclude the proof by the following chain of inequalities:
\begin{align}
    &\vert \cE_{\eps_n,\mathrm{NL}}^{(k,p)}(v_{\eps_n}) - \cE_\infty^{(k,p)}(v^*) \vert \leq  T_1 + \vert  H_{\eps_n}(v_{\eps_n}) - \cE_\infty^{(k,p)}(v^*) \vert \notag \\
    &\leq C C_\delta \eps_n^p + \delta H_{\eps_n}(v_{\eps_n}) + \vert  H_{\eps_n}(v_{\eps_n}) - \tilde{H}_{\eps_n}(v_{\eps_n}) \vert + \vert  \tilde{H}_{\eps_n}(v_{\eps_n}) - \cE_\infty^{(k,p)}(v^*) \vert \label{eq:proofs:pointwiseConvergenceNonlocal:firstBound} \\
    &\leq C C_\delta \eps_n^p + \delta H_{\eps_n}(v_{\eps_n}) + C\eps_n + \vert  \hat{H}_{\eps_n}(v_{\eps_n}) - \cE_\infty^{(k,p)}(v^*) \vert + \vert  \bar{H}_{\eps_n}(v_{\eps_n}) \vert \label{eq:proofs:pointwiseConvergenceNonlocal:secondBound} \\
    &\leq C C_\delta \eps_n^p + \delta H_{\eps_n}(v_{\eps_n}) + C\eps_n + \delta\cE_\infty^{(k,p)}(v^*) + C_\delta o(\eps_n) + o(1) \label{eq:proofs:pointwiseConvergenceNonlocal:thirdBound} 
\end{align}
where we used \eqref{eq:proofs:pointwiseConvergenceNonlocal:integral} for \eqref{eq:proofs:pointwiseConvergenceNonlocal:firstBound}, \eqref{eq:proofs:pointwiseConvergenceNonlocal:HHfinal} for \eqref{eq:proofs:pointwiseConvergenceNonlocal:secondBound} and \eqref{eq:proofs:pointwiseConvergenceNonlocal:HbarFinal} as well as \eqref{eq:proofs:pointwiseConvergenceNonlocal:limit} for \eqref{eq:proofs:pointwiseConvergenceNonlocal:thirdBound}. By assumption \eqref{eq:proofs:pointwiseConvergenceNonlocal:boundedness}, we have that $H_{\eps_n}(v_{\eps_n}) \leq C$ and therefore, by first letting $n \to \infty$ and then $\delta \to 0$, we obtain \eqref{eq:proofs:pointwiseConvergenceNonlocal:convergence}.
\end{proof}

\begin{proposition}[$\liminf$-inequality for the nonlocal energies] \label{prop:proofs:gammaConvergence:liminfNonlocal}
    Assume that \ref{ass:Main:Ass:S1}, \ref{ass:Main:Ass:M1}, \ref{ass:Main:Ass:M2} and \ref{ass:Main:Ass:W2} hold. For every $u \in \Lp{p}(\mu)$ and sequence $u_{\eps_n} \to u$ in $\Lp{p}(\mu)$, we have that:
    \begin{align}
    &\liminf_{n\to \infty} \cE_{\eps_n,\mathrm{NL}}^{(k,p)}(u_{\eps_n},\eta) \geq \cE_{\infty}^{(k,p)}(v). \label{eq:proofs:liminfNonLocal:liminf}
    \end{align}
\end{proposition}

\begin{proof}
In the proof $C>0$ will denote a constant that can be arbitrarily large, is independent of $n$, $\delta$ and that may change from line to line.

Since \eqref{eq:proofs:liminfNonLocal:liminf} is trivial if 
$\liminf_{n\to \infty} \cE_{\eps_n,\mathrm{NL}}^{(k,p)}(u_{\eps_n},\eta) = \infty$, we might assume without loss of generality (see \cite{weihs2023consistency}) that $\sup_{n > 0} \cE_{\eps_n,\mathrm{NL}}^{(k,p)}(u_{\eps_n},\eta) \leq  C$.

We first assume that $\rho$ is Lipschitz. We will be in the same setting as in \cite[Theorem 4.1]{Trillos3} and therefore let $\Omega'$ be compactly contained in $\Omega$. This implies that there exists $\delta' > 0$ such that $\Omega'' := \bigcup_{x \in \Omega'} B(x,\delta') \subset \Omega$. Furthermore, let $J$ be a positive mollifier supported in $\overline{B(0,1)}$ and for $0<\delta < \delta'$ as well as $v \in \Lp{p}(\mu)$ we set 
\[
v_\delta(x) = \int_{\bbR^d} J_\delta(x-z)v(z) \, \dd z.
\]
By \cite[Theorem C.16]{leoni2017first}, we have that $v_\delta \to v$ in $\Lp{p}(\mu)$, $v_\delta$ are smooth and in particular, by Young's convolution inequality,  for $\ell \in\{1, 2\}$, 
\begin{equation} \label{eq:proofs:liminfNonlocal:boundDerivatives}
\Vert \nabla^{\ell} v_\delta \Vert_{\Lp{\infty}(\bbR^d)} \leq \frac{C}{\delta^{\ell+d}} \Vert v \Vert_{\Lp{1}(\Omega)} \leq \frac{C}{\delta^{\ell+d}} \Vert v \Vert_{\Lp{p}(\Omega)}.
\end{equation}
If we therefore set $v = u_{\eps_n}$ and $u_{\eps_n,\delta} := (u_{\eps_n})_{\delta}$ and insert the latter in \eqref{eq:proofs:liminfNonlocal:boundDerivatives}, we obtain
\[
\sup_{\eps_n > 0} \sum_{\ell = 1}^{2} \Vert \nabla^{\ell} u_{\eps_n,\delta} \Vert_{\Lp{\infty}(\bbR^d)} \leq C \sum_{\ell = 1}^{2} \frac{1}{\delta^{\ell + d}}
\]
where the last inequality follows from the fact that $u_{\eps_n} \to u$ in $\Lp{p}(\Omega)$ implies that $\Vert u_{\eps_n} \Vert_{\Lp{p}(\Omega)} \leq C$ uniformly. 
For fixed $\delta > 0$, we deduce that \eqref{eq:proofs:pointwiseConvergenceNonlocal:boundedness} is satisfied.
Furthermore, 
\begin{align}
    \int_{\Omega'} \hspace{-2mm}  \vert \nabla u_{\eps_n,\delta}(x) - \nabla u_{\delta}(x) \vert^p \, \dd x &= \frac{1}{\delta^{d + 1}} \hspace{-1mm}  \int_{\Omega'} \hspace{-1mm} \left| \int_{B(0,\delta)} \hspace{-5mm} (\nabla J)\l\frac{z}{\delta} \r \l u_{\eps_n}(x-z) - u(x-z) \r \, \dd z \right|^p \, \dd x \notag \\
    &\leq \frac{C}{\delta^{d+1}} \int_{\Omega'} \int_{B(0,\delta)} \vert u_{\eps_n}(x-z) - u(x-z) \vert^p \, \dd z \dd x \notag \\
    &= \frac{C}{\delta^{d+1}}\int_{\Omega'} \int_{B(x,\delta)} \vert u_{\eps_n}(r) - u(r) \vert^p \, \dd r \dd x \label{eq:proofs:liminfNonlocal:changeofVariables1} \\
    &\leq \frac{C}{\delta^{d+1}}\int_{\Omega} \vert u_{\eps_n}(x) - u(x) \vert^p \, \dd x \label{eq:proofs:liminfNonlocal:deltaPrime}
\end{align}
where we used a change of variables for \eqref{eq:proofs:liminfNonlocal:changeofVariables1} and the definition of $\delta'$ as well as Assumption \ref{ass:Main:Ass:S1} for \eqref{eq:proofs:liminfNonlocal:deltaPrime}. We conclude from the latter that $\nabla u_{\eps_n,\delta} \to \nabla u_{\delta}$ in $\Lp{p}(\Omega')$ as $\eps_n \to 0$ and therefore, by Lemma \ref{lem:proofs:pointwiseConvergenceNonLocal},
\begin{align}
&\lim_{n \to \infty}\hspace{-0.7mm}  \frac{1}{\eps_n^{p + kd}} \hspace{-1.5mm}  \int_{(\Omega')^{k+1}} \hspace{-1.8mm}  \ls \prod_{j=1}^{k} \prod_{r=0}^{j-1} \eta\l \frac{\vert x_{j} - x_{r}\vert}{\eps_n} \r \rs \hspace{-1mm}  \left\vert u_{\eps_n,\delta}(x_1) - u_{\eps_n,\delta}(x_0) \right\vert^p  \prod_{\ell = 0}^k \rho(x_\ell) \, \dd x_k \cdots \dd x_0 \notag \\
&= \sigma_\eta^{(k)} \int_{\Omega'} \Vert \nabla u_\delta(x_0) \Vert_2^p \,  \rho(x_0)^{k+1} \, \dd x_0.
\label{eq:proofs:liminfNonLocal:convergence}
\end{align}

Let us define 
\begin{align}
a_{\eps_n,\delta} &= \frac{1}{\eps_n^{p + dk}} \int_{\bbR^d} \int_{(\Omega'')^{k+1}} \frac{1}{\delta^d}J\l \frac{z}{\delta}\r \ls \prod_{j=1}^{k} \prod_{r=0}^{j-1} \eta\l \frac{\vert x_{j} - x_{r}\vert}{\eps} \r \rs \left\vert u_{\eps_n}(x_1) - u_{\eps_n}(x_0) \right\vert^p \notag \\
&\times \l \prod_{\ell=0}^k \rho(x_\ell) - \prod_{\ell=0}^k \rho(x_\ell + z) \r \, \dd x_k \cdots \dd x_0 \dd z. \notag
\end{align}
We now estimate as follows:
\begin{align}
    &\cE_{\eps_n,\mathrm{NL}}^{(k,p)}(u_{\eps_n},\eta) \geq \frac{1}{\eps_n^{p + kd}} \int_{(\Omega'')^{k+1}} \ls \prod_{j=1}^{k} \prod_{r=0}^{j-1} \eta\l \frac{\vert x_{j} - x_{r}\vert}{\eps} \r \rs \left\vert u_{\eps_n}(x_1) - u_{\eps_n}(x_0) \right\vert^p \notag \\
    & \qquad \qquad \times \ls \prod_{\ell = 0}^k \rho(x_\ell) \rs \, \dd x_k \cdots \dd x_0 \notag \\
    &=\frac{1}{\eps_n^{p + dk}} \int_{\bbR^d} \int_{(\Omega'')^{k+1}} \frac{1}{\delta^d}J\l \frac{z}{\delta}\r \ls \prod_{j=1}^{k} \prod_{r=0}^{j-1} \eta\l \frac{\vert x_{j} - x_{r}\vert}{\eps} \r \rs \left\vert u_{\eps_n}(x_1) - u_{\eps_n}(x_0) \right\vert^p \notag \\
    &\qquad \qquad \times \ls \prod_{\ell = 0}^k \rho(x_\ell + z) \rs  \dd x_k \cdots \dd x_0 \dd z + a_{\eps_n,\delta} \notag \\
    &\geq \frac{1}{\eps_n^{p + dk}} \int_{\bbR^d} \int_{(\Omega')^{k+1}} \frac{1}{\delta^d}J\l \frac{z}{\delta}\r \ls \prod_{j=1}^{k} \prod_{r=0}^{j-1} \eta\l \frac{\vert \hat{x}_{j} - \hat{x}_{r}\vert}{\eps} \r \rs \left\vert u_{\eps_n}(\hat{x}_1-z) - u_{\eps_n}(\hat{x}_0-z) \right\vert^p \notag \\
    &\qquad \qquad \times \ls \prod_{\ell=0}^k \rho(\hat{x}_\ell) \rs \dd \hat{x}_k \cdots \dd \hat{x}_0 \dd z + a_{\eps_n,\delta} \label{eq:proofs:liminfNonLocal:changeofVariables2} \\
    &\geq \frac{1}{\eps_n^{p + dk}}  \int_{(\Omega')^{k+1}}  \ls \prod_{\ell=0}^k \rho(\hat{x}_\ell) \rs \ls \prod_{j=1}^{k} \prod_{r=0}^{j-1} \eta\l \frac{\vert \hat{x}_{j} - \hat{x}_{r}\vert}{\eps} \r \rs \notag \\
    &\qquad \qquad \times \left\vert \int_{\bbR^d} \frac{1}{\delta^d}J\l \frac{z}{\delta}\r \l u_{\eps_n}(\hat{x}_1-z) - u_{\eps_n}(\hat{x}_0-z) \r \, \dd z \right\vert^p  \dd \hat{x}_k \cdots \dd \hat{x}_0 + a_{\eps_n,\delta} \label{eq:proofs:liminfNonLocal:Jensen} \\
    &= \frac{1}{\eps_n^{p + dk}}  \int_{(\Omega')^{k+1}}  \ls \prod_{\ell=0}^k \rho(\hat{x}_\ell) \rs \ls \prod_{j=1}^{k} \prod_{r=0}^{j-1} \eta\l \frac{\vert \hat{x}_{j} - \hat{x}_{r}\vert}{\eps} \r \rs \notag \\
    &\qquad \qquad \times \left\vert u_{\eps_n,\delta}(\hat{x}_1) - u_{\eps_n,\delta}(\hat{x}_0) \right\vert^p  \dd \hat{x}_k \cdots \dd \hat{x}_0 +a_{\eps_n,\delta} \label{eq:proofs:liminfNonLocal:FinalEstimate}
\end{align}
where we used the change of variables $\hat{x}_j = x_j + z$ for $0 \leq j \leq k$ 
and the fact that, by definition of $\delta$, $\Omega' \subseteq \{w \in \bbR^d \spaceBar w + z \in \Omega'' \text{ for $z \in B(0,\delta)$} \}$ for \eqref{eq:proofs:liminfNonLocal:changeofVariables2} as well as Jensen's inequality with probability measure $\nu(A) = \int_{A} \frac{1}{\delta^d}J\l \frac{z}{\delta}\r \, \dd z$ for \eqref{eq:proofs:liminfNonLocal:Jensen}.
 
Now, using \eqref{eq:proofs:productIdentities2}, we have
\begin{align}
    &\vert a_{\eps_n,\delta} \vert \leq \frac{C}{\eps_n^{p + dk}} \int_{\bbR^d} \int_{(\Omega'')^{k+1}} \frac{1}{\delta^d}J\l \frac{z}{\delta}\r \ls \prod_{j=1}^{k} \prod_{r=0}^{j-1} \eta\l \frac{\vert x_{j} - x_{r}\vert}{\eps} \r \rs \notag \\
    &\times  \left\vert u_{\eps_n}(x_1) - u_{\eps_n}(x_0) \right\vert^p \vert z \vert \, \dd x_k \cdots \dd x_0 \dd z \notag \\
    &\leq \frac{C\delta}{\eps_n^{p + dk}} \int_{(\Omega'')^{k+1}} \ls \prod_{j=1}^{k} \prod_{r=0}^{j-1} \eta\l \frac{\vert x_{j} - x_{r}\vert}{\eps} \r \rs \left\vert u_{\eps_n}(x_1) - u_{\eps_n}(x_0) \right\vert^p \prod_{\ell=0}^k \rho(x_\ell) \, \dd x_k \cdots \dd x_0 \label{eq:proofs:liminfNonLocal:rho}\\
    &\leq C\delta \cE_{\eps_n,\mathrm{NL}}^{(k,p)}(u_{\eps_n},\eta) \label{eq:proofs:liminfNonLocal:error}
\end{align}
where we used Assumption \ref{ass:Main:Ass:M2} for \eqref{eq:proofs:liminfNonLocal:rho}. 
We therefore obtain:
\begin{align}
    &\liminf_{n\to \infty} \cE_{\eps_n,\mathrm{NL}}^{(k,p)}(u_{\eps_n},\eta) \geq \liminf_{\delta \to 0} \liminf_{n \to \infty} a_{\eps_n,\delta} \notag \\
    &+ \liminf_{\delta \to 0} \liminf_{n \to \infty} \frac{1}{\eps_n^{p + dk}}  \int_{(\Omega')^{k+1}}  \ls \prod_{\ell=0}^k \rho(\hat{x}_\ell) \rs \ls \prod_{j=1}^{k} \prod_{r=0}^{j-1} \eta\l \frac{\vert \hat{x}_{j} - \hat{x}_{r}\vert}{\eps_n} \r \rs \notag \\
    &\times\left\vert u_{\eps_n,\delta}(\hat{x}_1) - u_{\eps_n,\delta}(\hat{x}_0) \right\vert^p  \dd \hat{x}_k \cdots \dd \hat{x}_0 \label{eq:proofs:liminfNonLocal:estimateEnergy} \\
    &= \liminf_{\delta \to 0} \sigma_\eta^{(k)} \int_{\Omega'} \Vert \nabla u_\delta(x_0) \Vert_2^p \,  \rho(x_0)^{k+1} \, \dd x_0 \label{eq:proofs:liminfNonLocal:convergence2} \\
    &\geq \sigma_\eta^{(k)} \int_{\Omega'} \Vert \nabla u(x_0) \Vert_2^p \,  \rho(x_0)^{k+1} \, \dd x_0 \label{eq:proofs:liminfNonLocal:lsc}
\end{align}
where we used \eqref{eq:proofs:liminfNonLocal:FinalEstimate} for \eqref{eq:proofs:liminfNonLocal:estimateEnergy}, \eqref{eq:proofs:liminfNonLocal:error} and the fact that the energies $\cE_{\eps_n,\mathrm{NL}}^{(k,p)}(u_{\eps_n},\eta)$ are uniformly bounded as well as \eqref{eq:proofs:liminfNonLocal:convergence} for \eqref{eq:proofs:liminfNonLocal:convergence2} and, since $u_{\delta} \to u$ in $\Lp{p}(\Omega)$, the fact that $\cE_{\infty}^{(k,p)}$ is lower-semicontinuous for \eqref{eq:proofs:liminfNonLocal:lsc}. We conclude by noting that $\Omega^\prime$ was an arbitrary set compactly contained in $\Omega$ so we can take $\Omega' \uparrow \Omega$ in \eqref{eq:proofs:liminfNonLocal:lsc} to get \eqref{eq:proofs:liminfNonLocal:liminf}.

Dealing with the case where $\rho$ is not Lipschitz is done analogously to the proof of \cite[Theorem 4.1]{Trillos3}, i.e. relying on the approximation of continuous functions by a monotone sequence of Lipschitz functions and the monotone convergence theorem to deduce the result.
\end{proof}

\begin{proposition}[$\limsup$-inequality for the nonlocal energies] \label{prop:gammaConvergence:NonlocalLimsup}
    Assume that \ref{ass:Main:Ass:S1}, \ref{ass:Main:Ass:M1}, \ref{ass:Main:Ass:M2} and \ref{ass:Main:Ass:W2} hold. For every $u \in \Lp{p}(\mu)$, there exists a sequence $\{u_{\eps_n}\}_{n=1}^\infty \subseteq \Lp{p}(\mu)$ such that $u_{\eps_n} \to u$ in $\Lp{p}(\mu)$ and
    \begin{align}
        &\limsup_{n \to \infty} \cE_{\eps_n,\mathrm{NL}}^{(k,p)}(u_{\eps_n},\eta) \leq \cE_{\infty}^{(k,p)}(u). \label{eq:proofs:limsupNonLocal:limsup}
    \end{align}
    In particular, if $u\in\Ck{\infty}(\bar{\Omega})$ then we can choose $u_{\eps_n}=u$.
\end{proposition}

\begin{proof}
In the proof $C>0$ will denote a constant that can be arbitrarily large, is independent of $n$, and that may change from line to line.

We start by noting that \eqref{eq:proofs:limsupNonLocal:limsup} is trivial if $\cE_{\infty}^{(k,p)}(u) = \infty$ so that we assume $u \in \Wkp{1}{p}(\Omega)$. Furthermore, we are going to apply \cite[Remark 2.7]{Trillos3}, so it is sufficient to verify \eqref{eq:proofs:limsupNonLocal:limsup} on a dense subset of $\Wkp{1}{p}(\Omega)$, namely $\Ckc{\infty}(\Omega)$.

We first start by assuming that $\rho$ is Lipschitz. Let us define $u_n = u$ and we estimate as follows:
\begin{align}
&\cE_{\eps_n,\mathrm{NL}}^{(k,p)}(u,\eta) 
\leq \frac{1}{\eps_n^{p}} \int_{\Omega} \int_{(\bbR^d)^k}  \hspace{-1.7mm} \ls  \prod_{s=1}^k \eta \l \vert z_s \vert \r \rs \hspace{-1.7mm} \ls \prod_{j=1}^{k} \prod_{r=1}^{j-1} \eta\l \vert z_{j} - z_{r} \vert \r \rs \left\vert u(x_0 + \eps_n z_1) - u(x_0) \right\vert^p \notag \\
&\qquad \times \rho(x_0) \prod_{\ell = 1}^k \rho\l x_0 + \eps_n z_\ell \r \, \dd z_k \cdots \dd z_1 \dd x_0 \label{eq:proofs:limsupNonLocal:changeVariables1} \\
&\,\,\leq \int_{\Omega} \int_{(\bbR^d)^k} \ls  \prod_{s=1}^k \eta \l \vert z_s \vert \r \rs \ls \prod_{j=1}^{k} \prod_{r=1}^{j-1} \eta\l \vert z_{j} - z_{r} \vert \r \rs \left\vert \nabla u (x_0) \cdot z_1 \right\vert^p \notag \\
&\qquad \times \ls \rho(x_0) \prod_{\ell = 1}^k \rho\l x_0 + \eps_n z_\ell \r \rs \, \dd z_k \cdots \dd z_1 \dd x_0 \notag \\
&\qquad + \eps_n \Vert \rho \Vert_{\Lp{\infty}}^{k+1} \Vert u \Vert_{\Ck{2}}^p \int_{\Omega} \int_{(\bbR^d)^k} \ls  \prod_{s=1}^k \eta \l \vert z_s \vert \r \rs \ls \prod_{j=1}^{k} \prod_{r=1}^{j-1} \eta\l \vert z_{j} - z_{r} \vert \r \rs \, \dd z_k \cdots \dd z_1 \dd x_0 \notag \\
&\,\,\leq \int_{\Omega} \int_{(\bbR^d)^k} \ls  \prod_{s=1}^k \eta \l \vert z_s \vert \r \rs \ls \prod_{j=1}^{k} \prod_{r=1}^{j-1} \eta\l \vert z_{j} - z_{r} \vert \r \rs \left\vert \nabla u (x_0) \cdot z_1 \right\vert^p \notag \\
&\qquad \times \left\vert \rho(x_0) \prod_{\ell = 1}^k \rho\l x_0 + \eps_n z_\ell \r - \rho(x_0)^{k+1} \right\vert \, \dd z_k \cdots \dd z_1 \dd x_0 + \cE_{\infty}^{(k,p)}(u) + C \eps_n \notag \\
&\,\,\leq C \eps_n \int_{\Omega} \int_{(\bbR^d)^k} \ls  \prod_{s=1}^k \eta \l \vert z_s \vert \r \rs \ls \prod_{j=1}^{k} \prod_{r=1}^{j-1} \eta\l \vert z_{j} - z_{r} \vert \r \rs \left\vert \nabla u (x_0) \cdot z_1 \right\vert^p \notag \\
&\qquad \times \sum_{\ell = 1}^k \vert z_\ell \vert \, \dd z_k \cdots \dd z_1 \dd x_0 +\cE_{\infty}^{(k,p)}(u) + C \eps_n \label{eq:proofs:limsupNonLocal:product} \\
&\,\,\leq \cE_{\infty}^{(k,p)}(u) + C\eps_n \label{eq:proofs:limsupNonLocal:finiteness2}
\end{align} 
where we used the change of variables $z_j = (x_j - x_{0})/\eps_n$ for $1 \leq j \leq k$ for \eqref{eq:proofs:limsupNonLocal:changeVariables1}, Assumption \ref{ass:Main:Ass:W2}, \eqref{eq:proofs:productIdentities1} for \eqref{eq:proofs:limsupNonLocal:product} and Assumption \ref{ass:Main:Ass:W2} for \eqref{eq:proofs:limsupNonLocal:finiteness2}. Taking the limit as $\eps_n \to 0$, we obtain \eqref{eq:proofs:limsupNonLocal:limsup}.

In order to consider general $\rho$, we proceed as in \cite[Theorem 4.1]{Trillos3} which concludes the proof.
\end{proof}

\subsubsection{Compactness}

The following lemma is inspired by \cite{weihs2023consistency}.

\begin{lemma}[Uniform bound of energies of minimizers] \label{lem:proofs:boundedEnergies}
    Assume that \ref{ass:Main:Ass:S1}, \ref{ass:Main:Ass:M1}, \ref{ass:Main:Ass:M2}, \ref{ass:Main:Ass:W2}, \ref{ass:Main:Ass:D1}, \ref{ass:Main:Ass:D2} and \ref{ass:Main:Ass:L1} hold. Let $(\mu_n,u_n)$ be minimizers of $(\cS\cF)_{n,\eps_n}^{(q,p)}$. Then, $\bbP$-a.s., there exists $C > 0$ such that 
    \[
    \sup_{n > 0} (\cS\cF)_{n,\eps_n}^{(q,p)}((\mu_n,u_n)) \leq C.
    \]
\end{lemma}

\begin{proof}
    With probability one, we can assume that the conclusions of Lemma \ref{lem:proofs:gammaConvergence:sum} hold.
    
    In the proof $C>0$ will denote a constant that can be arbitrarily large, is independent of $n$, and that may change from line to line. We are going to follow the proof of \cite[Lemma 4.25]{weihs2023consistency}

    Let $v \in \Ckc{\infty}(\Omega)$ be a function that interpolates the points $\{(x_i,\ell_i)\}_{i=1}^\infty$. Then, $v \in \Wkp{1}{p}(\Omega)$ so in particular, by Assumption \ref{ass:Main:Ass:M2}, there exists $C_0$ such that 
    \[
    (\cS\cF)_{\infty}^{(q,p)}((\mu,v)) < C_0. 
    \]
    By Lemma \ref{lem:proofs:gammaConvergence:sum}, we can pick a recovery sequence $\{v_n\}_{n=1}^\infty$ for $v$ such that: 
    \begin{align*}
    \lim_{n \to \infty} h_n &:= \lim_{n \to \infty} \sup_{m \geq n } (\cS\cF)^{(q,p)}_{m,\eps_m}((\mu_m,v_m)) \\
    &= \limsup_{n \to \infty} (\cS\mathcal{F})_{n,\eps_n}^{(q,p)}((\mu_n,v_n)) \\
    & \leq (\cS\cF)_\infty^{(q,p)}((\mu,v)) \\
    &< C_0
    \end{align*}
    (since $(\cS\cF)_\infty^{(q,p)}((\mu,v)) = (\cS\cG)_\infty^{(q,p)}((\mu,v))$).
    Let $h := \limsup_{n \to \infty} (\cS\cF)^{(q,p)}_{n,\eps_n}((\mu_n,v_n))$ and let $\bar{\eps} = C_0 - h > 0$. Then, there exists $n_0$ such that for all $n \geq n_0$, $ h_n - h < \bar{\eps}/2$, 
    which means that $$h_n = \sup_{m \geq n} (\cS\cF)^{(q,p)}_{m,\eps_m}(v_m) < h + \bar{\eps}/2 < C_0.$$ Using the latter, we have 
\begin{align*}
&\sup_{n > 0}  (\cS\cF)^{(q,p)}_{n,\eps_n}((\mu_n,v_n)) \\
 &= \hspace{-0.5mm}\max\hspace{-0.8mm}\left\{ (\cS\cF)^{(q,p)}_{1,\eps_1}((\mu_1,v_1)),\dots, (\cS\cF)^{(q,p)}_{n_0,\eps_{n_0}}((\mu_{n_0},v_{n_0})),\sup_{n \geq n_0} (\cS\cF)^{(q,p)}_{n,\eps_n}((\mu_n,v_n))\right\} \\
 & \leq C.
\end{align*}
Since $\{u_n\}_{n=1}^\infty$ are minimizers, we use $(\cS\cF)^{(q,p)}_{n,\eps_n}((\mu_n,u_n)) \leq (\cS\cF)^{(q,p)}_{n,\eps_n}((\mu_n,v_n))$ to conclude.
\end{proof}

\section{Complete numerical experiments} \label{sec:appendix:numerical}

In this section, we present the complete numerical experiments of Section \ref{sec:discussion}. We refer to Table \ref{tab:qj:terminology} for a review of the terminology used throughout the experiments. 


\begin{table*}[!h]
\vspace{-3mm}
\caption{Accuracy of various SSL methods on the iris dataset. We pick $\eps^{(k)} = 2^{3-k}$ for $1 \leq k \leq 5$. Proposed methods are in bold.}
\vspace{-3mm}
\label{tab:irisQ:full}
\vskip 0.15in
\small
\resizebox{0.975\textwidth}{!}{
\begin{tabular}{llllllllllll}
\toprule
$q$ &   rate   &   Laplace  &   Poisson  &  FL ($s=2)$    & FL ($s=3$)  &  \textbf{IP-QC}   &  \textbf{CP-QC}  &    \textbf{IP-SC} &   \textbf{CP-SC}  &   \textbf{IP-CC} &    \textbf{CP-CC} \\
\midrule
\multirow[t]{7}{*}{2} & 0.02 & 64.56 (15.0) & \textbf{79.39} (8.83) & 70.73 (7.46) & 71.54 (7.76) & 75.56 (10.01) & 70.54 (10.84) & 75.02 (10.0) & 68.75 (12.77) & 73.89 (10.82) & 68.29 (12.5) \\
 & 0.05 & 73.15 (9.26) & 80.54 (6.05) & 74.47 (9.26) & 76.19 (9.13) & \textbf{82.26} (8.98) & 74.65 (9.73) & 81.62 (9.19) & 73.94 (9.66) & 80.46 (9.69) & 73.67 (9.78) \\
 & 0.10 & 81.05 (8.94) & 80.58 (3.25) & 82.55 (8.21) & 84.37 (7.27) & \textbf{90.39} (3.84) & 85.41 (7.95) & 90.29 (4.23) & 84.44 (8.47) & 89.93 (4.75) & 83.36 (8.81) \\
 & 0.20 & 87.43 (5.88) & 80.33 (2.3) & 87.91 (5.25) & 88.74 (4.14) & 92.48 (2.49) & 90.53 (3.9) & 92.52 (2.59) & 89.96 (4.38) & \textbf{92.54} (2.69) & 89.22 (4.92) \\
 & 0.30 & 90.82 (3.2) & 79.66 (2.18) & 90.88 (2.87) & 91.05 (2.53) & 93.57 (2.48) & 92.51 (2.82) & 93.57 (2.46) & 92.36 (3.01) & \textbf{93.61} (2.52) & 92.01 (3.02) \\
 & 0.50 & 91.81 (2.71) & 79.59 (2.34) & 91.67 (2.9) & 91.8 (2.77) & \textbf{94.92} (2.3) & 93.52 (2.77) & 94.89 (2.31) & 93.33 (2.73) & 94.85 (2.32) & 93.04 (2.75) \\
 & 0.80 & 92.2 (4.67) & 79.33 (5.44) & 92.2 (4.35) & 92.57 (4.44) & 95.5 (3.49) & 94.47 (3.52) & \textbf{95.53} (3.49) & 93.93 (3.49) & \textbf{95.53} (3.49) & 93.67 (3.5) \\
\cline{1-12}
\multirow[t]{7}{*}{3} & 0.02 & 64.56 (15.0) & \textbf{79.39} (8.83) & 70.73 (7.46) & 71.54 (7.76) & 76.52 (10.76) & 71.27 (11.86) & 75.99 (10.48) & 69.29 (13.27) & 74.5 (11.18) & 68.47 (12.74) \\
 & 0.05 & 73.15 (9.26) & 80.54 (6.05) & 74.47 (9.26) & 76.19 (9.13) & \textbf{84.78} (9.8) & 75.72 (10.08) & 83.78 (10.0) & 74.33 (9.99) & 82.26 (10.21) & 73.53 (9.82) \\
 & 0.10 & 81.05 (8.94) & 80.58 (3.25) & 82.55 (8.21) & 84.37 (7.27) & \textbf{91.99} (3.32) & 87.44 (7.5) & 91.48 (3.54) & 85.9 (8.5) & 91.04 (4.06) & 84.23 (8.92) \\
 & 0.20 & 87.43 (5.88) & 80.33 (2.3) & 87.91 (5.25) & 88.74 (4.14) & \textbf{94.04} (2.49) & 92.03 (3.19) & 93.79 (2.51) & 91.12 (3.9) & 93.53 (2.61) & 90.09 (4.69) \\
 & 0.30 & 90.82 (3.2) & 79.66 (2.18) & 90.88 (2.87) & 91.05 (2.53) & \textbf{95.52} (1.86) & 94.04 (2.64) & 95.3 (1.95) & 93.32 (2.87) & 94.95 (2.22) & 92.59 (3.04) \\
 & 0.50 & 91.81 (2.71) & 79.59 (2.34) & 91.67 (2.9) & 91.8 (2.77) & \textbf{95.89} (1.86) & 94.97 (2.45) & 95.87 (1.9) & 94.41 (2.62) & \textbf{95.89} (1.9) & 93.77 (2.8) \\
 & 0.80 & 92.2 (4.67) & 79.33 (5.44) & 92.2 (4.35) & 92.57 (4.44) & 96.17 (3.23) & 95.73 (3.67) & \textbf{96.27} (3.29) & 95.43 (3.72) & 96.13 (3.37) & 94.57 (3.72) \\
\cline{1-12}
\multirow[t]{7}{*}{4} & 0.02 & 64.56 (15.0) & \textbf{79.39} (8.83) & 70.73 (7.46) & 71.54 (7.76) & 76.53 (10.81) & 71.46 (12.53) & 75.97 (10.56) & 69.34 (13.33) & 74.46 (11.17) & 68.41 (12.69) \\
 & 0.05 & 73.15 (9.26) & 80.54 (6.05) & 74.47 (9.26) & 76.19 (9.13) & \textbf{84.88} (9.8) & 76.15 (10.27) & 83.81 (9.98) & 74.42 (9.92) & 82.33 (10.26) & 73.52 (9.87) \\
 & 0.10 & 81.05 (8.94) & 80.58 (3.25) & 82.55 (8.21) & 84.37 (7.27) & \textbf{92.01} (3.3) & 87.63 (7.6) & 91.47 (3.54) & 86.04 (8.52) & 91.07 (4.03) & 84.27 (8.96) \\
 & 0.20 & 87.43 (5.88) & 80.33 (2.3) & 87.91 (5.25) & 88.74 (4.14) & \textbf{94.08} (2.47) & 92.4 (3.25) & 93.8 (2.52) & 91.32 (3.98) & 93.54 (2.62) & 90.22 (4.6) \\
 & 0.30 & 90.82 (3.2) & 79.66 (2.18) & 90.88 (2.87) & 91.05 (2.53) & \textbf{95.53} (1.84) & 94.46 (2.68) & 95.33 (1.97) & 93.58 (2.87) & 94.95 (2.22) & 92.73 (3.06) \\
 & 0.50 & 91.81 (2.71) & 79.59 (2.34) & 91.67 (2.9) & 91.8 (2.77) & 95.88 (1.84) & 95.51 (2.33) & 95.87 (1.87) & 94.79 (2.64) & \textbf{95.92} (1.89) & 93.95 (2.71) \\
 & 0.80 & 92.2 (4.67) & 79.33 (5.44) & 92.2 (4.35) & 92.57 (4.44) & 96.17 (3.23) & 95.9 (3.6) & \textbf{96.23} (3.27) & 95.83 (3.65) & 96.17 (3.36) & 94.9 (3.8) \\
\cline{1-12}
\multirow[t]{7}{*}{5} & 0.02 & 64.56 (15.0) & \textbf{79.39} (8.83) & 70.73 (7.46) & 71.54 (7.76) & 76.6 (10.82) & 71.38 (12.49) & 75.98 (10.56) & 69.14 (13.28) & 74.46 (11.16) & 68.37 (12.68) \\
 & 0.05 & 73.15 (9.26) & 80.54 (6.05) & 74.47 (9.26) & 76.19 (9.13) & \textbf{84.89} (9.81) & 76.31 (10.26) & 83.83 (9.98) & 74.44 (9.86) & 82.33 (10.26) & 73.53 (9.83) \\
 & 0.10 & 81.05 (8.94) & 80.58 (3.25) & 82.55 (8.21) & 84.37 (7.27) & \textbf{92.02} (3.3) & 87.59 (7.53) & 91.47 (3.54) & 85.98 (8.48) & 91.07 (4.03) & 84.29 (8.96) \\
 & 0.20 & 87.43 (5.88) & 80.33 (2.3) & 87.91 (5.25) & 88.74 (4.14) & \textbf{94.08} (2.44) & 92.49 (3.27) & 93.81 (2.52) & 91.34 (4.0) & 93.54 (2.62) & 90.27 (4.63) \\
 & 0.30 & 90.82 (3.2) & 79.66 (2.18) & 90.88 (2.87) & 91.05 (2.53) & \textbf{95.53} (1.84) & 94.53 (2.69) & 95.33 (1.97) & 93.65 (2.87) & 94.95 (2.22) & 92.76 (3.07) \\
 & 0.50 & 91.81 (2.71) & 79.59 (2.34) & 91.67 (2.9) & 91.8 (2.77) & 95.88 (1.84) & 95.63 (2.28) & 95.87 (1.87) & 94.85 (2.61) & \textbf{95.92} (1.89) & 93.93 (2.71) \\
 & 0.80 & 92.2 (4.67) & 79.33 (5.44) & 92.2 (4.35) & 92.57 (4.44) & 96.17 (3.23) & 96.13 (3.47) & \textbf{96.23} (3.27) & 95.87 (3.61) & 96.17 (3.36) & 94.9 (3.8) \\
\cline{1-12}
\end{tabular}
}
\end{table*}

\begin{table*}[!h]
\vspace{-3mm}
\caption{Accuracy of various SSL methods on the iris dataset. We pick $\eps^{(\ell)} = 2^{3-\ell}$ for $1 \leq \ell \leq 5$. Proposed methods are in bold.}
\vspace{-3mm}
\label{tab:irisJ:full}
\vskip 0.15in
\small
\resizebox{0.975\textwidth}{!}{
\begin{tabular}{lllllllllllll}
\toprule
 $j$ &  rate    &    Laplace  &   Poisson   &  FL ($s = 2$)  &    FL ($s=3$) &  WNLL   &  $p$-Lap &   RW  &   CK &    SLP &  \textbf{IP-VQC (2)} &    \textbf{IP-VQC (3)} \\
\midrule
\multirow[t]{7}{*}{1} & 0.02 & 61.86 (13.18) & 82.17 (8.39) & 70.41 (6.98) & 71.54 (7.29) & 82.59 (8.64) & \textbf{86.56}(8.18) & 77.13 (9.09) & 76.24 (8.18) & 32.05 (6.24) & 59.76 (13.85) & 53.71 (15.23) \\
 & 0.05 & 71.06 (7.31) & 81.55 (5.92) & 72.73 (8.16) & 74.31 (7.75) & 84.42 (7.85) & \textbf{88.54} (6.52) & 78.38 (9.14) & 79.48 (7.88) & 34.97 (7.18) & 70.28 (9.06) & 69.32 (8.69) \\
 & 0.10 & 81.07 (8.72) & 80.23 (3.19) & 82.76 (7.99) & 84.78 (7.04) & 89.34 (4.42) & \textbf{91.61} (2.74) & 82.34 (7.84) & 85.9 (4.5) & 46.67 (19.38) & 81.03 (10.58) & 80.19 (10.61) \\
 & 0.20 & 87.58 (4.63) & 80.12 (2.62) & 88.2 (4.11) & 89.09 (3.37) & 90.57 (2.57) & \textbf{91.68} (2.13) & 84.25 (6.29) & 89.39 (2.82) & 64.05 (12.29) & 90.46 (3.6) & 90.31 (3.74) \\
 & 0.30 & 90.12 (3.0) & 79.69 (2.11) & 90.21 (2.97) & 90.36 (2.42) & 91.11 (2.25) & 92.06 (2.28) & 86.19 (5.72) & 91.26 (2.54) & 58.99 (15.82) & 92.1 (3.21) & \textbf{92.29} (3.26) \\
 & 0.50 & 91.37 (3.23) & 79.44 (2.37) & 91.36 (3.01) & 91.41 (3.03) & 91.4 (3.23) & 92.07 (3.09) & 88.52 (4.53) & 92.29 (2.53) & 56.39 (16.53) & 93.4 (3.22) & \textbf{93.88} (3.11) \\
 & 0.80 & 92.0 (4.42) & 78.17 (4.4) & 91.67 (4.3) & 91.97 (4.16) & 92.07 (4.36) & 92.5 (4.14) & 88.97 (4.89) & 92.97 (4.37) & 89.7 (4.88) & 94.5 (3.86) & \textbf{95.37} (3.51) \\
\cline{1-13}
\multirow[t]{7}{*}{2} & 0.02 & 61.86 (13.18) & 82.17 (8.39) & 70.41 (6.98) & 71.54 (7.29) & 82.59 (8.64) & \textbf{86.56} (8.18) & 77.13 (9.09) & 76.24 (8.18) & 32.05 (6.24) & 67.74 (12.14) & 67.59 (12.34) \\
 & 0.05 & 71.06 (7.31) & 81.55 (5.92) & 72.73 (8.16) & 74.31 (7.75) & 84.42 (7.85) & \textbf{88.54} (6.52) & 78.38 (9.14) & 79.48 (7.88) & 34.97 (7.18) & 72.77 (9.79) & 72.53 (9.82) \\
 & 0.10 & 81.07 (8.72) & 80.23 (3.19) & 82.76 (7.99) & 84.78 (7.04) & 89.34 (4.42) & \textbf{91.61} (2.74) & 82.34 (7.84) & 85.9 (4.5) & 46.67 (19.38) & 83.77 (10.48) & 83.66 (10.64) \\
 & 0.20 & 87.58 (4.63) & 80.12 (2.62) & 88.2 (4.11) & 89.09 (3.37) & 90.57 (2.57) & \textbf{91.68} (2.13) & 84.25 (6.29) & 89.39 (2.82) & 64.05 (12.29) & 91.52 (3.07) & 91.57 (3.1) \\
 & 0.30 & 90.12 (3.0) & 79.69 (2.11) & 90.21 (2.97) & 90.36 (2.42) & 91.11 (2.25) & 92.06 (2.28) & 86.19 (5.72) & 91.26 (2.54) & 58.99 (15.82) & 92.82 (3.06) & \textbf{92.93} (3.06) \\
 & 0.50 & 91.37 (3.23) & 79.44 (2.37) & 91.36 (3.01) & 91.41 (3.03) & 91.4 (3.23) & 92.07 (3.09) & 88.52 (4.53) & 92.29 (2.53) & 56.39 (16.53) & 93.91 (2.99) & \textbf{94.24} (2.92) \\
 & 0.80 & 92.0 (4.42) & 78.17 (4.4) & 91.67 (4.3) & 91.97 (4.16) & 92.07 (4.36) & 92.5 (4.14) & 88.97 (4.89) & 92.97 (4.37) & 89.7 (4.88) & 95.47 (3.75) & \textbf{95.8} (3.44) \\
\cline{1-13}
\multirow[t]{7}{*}{3} & 0.02 & 61.86 (13.18) & 82.17 (8.39) & 70.41 (6.98) & 71.54 (7.29) & 82.59 (8.64) & \textbf{86.56} (8.18) & 77.13 (9.09) & 76.24 (8.18) & 32.05 (6.24) & 70.62 (9.65) & 70.09 (10.95) \\
 & 0.05 & 71.06 (7.31) & 81.55 (5.92) & 72.73 (8.16) & 74.31 (7.75) & 84.42 (7.85) & \textbf{88.54} (6.52) & 78.38 (9.14) & 79.48 (7.88) & 34.97 (7.18) & 73.8 (10.51) & 73.49 (10.33) \\
 & 0.10 & 81.07 (8.72) & 80.23 (3.19) & 82.76 (7.99) & 84.78 (7.04) & 89.34 (4.42) & \textbf{91.61} (2.74) & 82.34 (7.84) & 85.9 (4.5) & 46.67 (19.38) & 85.27 (9.58) & 85.34 (9.78) \\
 & 0.20 & 87.58 (4.63) & 80.12 (2.62) & 88.2 (4.11) & 89.09 (3.37) & 90.57 (2.57) & 91.68 (2.13) & 84.25 (6.29) & 89.39 (2.82) & 64.05 (12.29) & 91.9 (2.9) & \textbf{92.04} (2.96) \\
 & 0.30 & 90.12 (3.0) & 79.69 (2.11) & 90.21 (2.97) & 90.36 (2.42) & 91.11 (2.25) & 92.06 (2.28) & 86.19 (5.72) & 91.26 (2.54) & 58.99 (15.82) & 92.94 (3.0) & \textbf{93.15} (3.05) \\
 & 0.50 & 91.37 (3.23) & 79.44 (2.37) & 91.36 (3.01) & 91.41 (3.03) & 91.4 (3.23) & 92.07 (3.09) & 88.52 (4.53) & 92.29 (2.53) & 56.39 (16.53) & 94.03 (2.99) & \textbf{94.31} (2.85) \\
 & 0.80 & 92.0 (4.42) & 78.17 (4.4) & 91.67 (4.3) & 91.97 (4.16) & 92.07 (4.36) & 92.5 (4.14) & 88.97 (4.89) & 92.97 (4.37) & 89.7 (4.88) & 95.5 (3.71) & \textbf{95.73} (3.45) \\
\cline{1-13}
\multirow[t]{7}{*}{4} & 0.02 & 61.86 (13.18) & 82.17 (8.39) & 70.41 (6.98) & 71.54 (7.29) & 82.59 (8.64) & \textbf{86.56} (8.18) & 77.13 (9.09) & 76.24 (8.18) & 32.05 (6.24) & 70.98 (9.86) & 70.96 (9.8) \\
 & 0.05 & 71.06 (7.31) & 81.55 (5.92) & 72.73 (8.16) & 74.31 (7.75) & 84.42 (7.85) & \textbf{88.54} (6.52) & 78.38 (9.14) & 79.48 (7.88) & 34.97 (7.18) & 74.24 (10.64) & 74.03 (10.44) \\
 & 0.10 & 81.07 (8.72) & 80.23 (3.19) & 82.76 (7.99) & 84.78 (7.04) & 89.34 (4.42) & \textbf{91.61} (2.74) & 82.34 (7.84) & 85.9 (4.5) & 46.67 (19.38) & 86.04 (9.08) & 86.18 (9.05) \\
 & 0.20 & 87.58 (4.63) & 80.12 (2.62) & 88.2 (4.11) & 89.09 (3.37) & 90.57 (2.57) & 91.68 (2.13) & 84.25 (6.29) & 89.39 (2.82) & 64.05 (12.29) & 92.09 (2.92) & \textbf{92.18} (2.91) \\
 & 0.30 & 90.12 (3.0) & 79.69 (2.11) & 90.21 (2.97) & 90.36 (2.42) & 91.11 (2.25) & 92.06 (2.28) & 86.19 (5.72) & 91.26 (2.54) & 58.99 (15.82) & 93.04 (3.01) & \textbf{93.19} (3.04) \\
 & 0.50 & 91.37 (3.23) & 79.44 (2.37) & 91.36 (3.01) & 91.41 (3.03) & 91.4 (3.23) & 92.07 (3.09) & 88.52 (4.53) & 92.29 (2.53) & 56.39 (16.53) & 94.08 (2.97) & \textbf{94.31} (2.9) \\
 & 0.80 & 92.0 (4.42) & 78.17 (4.4) & 91.67 (4.3) & 91.97 (4.16) & 92.07 (4.36) & 92.5 (4.14) & 88.97 (4.89) & 92.97 (4.37) & 89.7 (4.88) & 95.53 (3.74) & \textbf{95.8} (3.47) \\
\cline{1-13}
\end{tabular}
}
\end{table*}

\begin{table*}[!h]
\vspace{-3mm}
\caption{Accuracy of various SSL methods on the digits dataset. We pick $\eps^{(k)} = 100^{2-k}$ for $1 \leq k \leq 5$. Proposed methods are in bold.}
\vspace{-3mm}
\label{tab:digitsQ:full}
\vskip 0.15in
\small
\resizebox{0.975\textwidth}{!}{
\begin{tabular}{llllllllll}
\toprule
$q$ &   rate   &   Laplace  &   Poisson  &  \textbf{IP-QC}   &  \textbf{CP-QC}  &    \textbf{IP-SC} &   \textbf{CP-SC}  &   \textbf{IP-CC} &    \textbf{CP-CC}\\
\midrule
\multirow[t]{7}{*}{2} & 0.02 & 11.96 (4.03) & \textbf{78.81} (2.98) & 24.19 (8.92) & 15.81 (5.5) & 21.91 (8.44) & 14.88 (5.48) & 19.55 (7.77) & 13.44 (5.08) \\
 & 0.05 & 19.35 (6.62) & \textbf{84.87} (1.63) & 62.35 (7.28) & 34.88 (8.75) & 59.01 (7.52) & 29.09 (9.18) & 53.38 (7.77) & 23.86 (7.5) \\
 & 0.10 & 42.87 (7.4) & \textbf{87.13} (1.12) & 81.84 (3.6) & 58.25 (7.34) & 80.96 (3.81) & 53.24 (6.98) & 79.07 (4.3) & 49.71 (6.62) \\
 & 0.20 & 68.58 (4.38) & 87.61 (0.94) & \textbf{89.21} (1.5) & 84.77 (2.24) & 89.11 (1.48) & 81.91 (2.7) & 88.86 (1.44) & 78.27 (3.3) \\
 & 0.30 & 82.1 (2.02) & 87.58 (0.74) & \textbf{91.78} (0.86) & 90.13 (1.08) & \textbf{91.78} (0.88) & 88.85 (1.2) & 91.74 (0.89) & 87.13 (1.3) \\
 & 0.50 & 88.3 (1.11) & 87.85 (0.78) & 93.87 (0.72) & 92.78 (0.87) & 93.87 (0.72) & 92.01 (0.87) & \textbf{93.91} (0.7) & 91.08 (0.93) \\
 & 0.80 & 89.73 (1.43) & 87.88 (1.42) & \textbf{94.96} (0.98) & 93.64 (1.16) & 94.94 (0.97) & 92.86 (1.22) & 94.9 (0.96) & 91.89 (1.22) \\
\cline{1-10}
\multirow[t]{7}{*}{3} & 0.02 & 11.96 (4.03) & \textbf{78.81} (2.98) & 22.57 (9.14) & 15.02 (5.8) & 20.91 (8.57) & 15.46 (5.4) & 18.96 (7.82) & 13.79 (5.43) \\
 & 0.05 & 19.35 (6.62) & \textbf{84.87} (1.63) & 61.81 (7.17) & 37.24 (7.55) & 58.56 (7.5) & 31.54 (9.11) & 52.93 (7.74) & 24.84 (7.85) \\
 & 0.10 & 42.87 (7.4) & \textbf{87.13} (1.12) & 81.57 (3.51) & 60.04 (7.23) & 80.78 (3.71) & 54.66 (7.07) & 78.93 (4.26) & 50.4 (6.7) \\
 & 0.20 & 68.58 (4.38) & 87.61 (0.94) & \textbf{89.12} (1.5) & 85.79 (2.17) & 89.06 (1.5) & 82.83 (2.57) & 88.82 (1.47) & 79.01 (3.19) \\
 & 0.30 & 82.1 (2.02) & 87.58 (0.74) & 91.74 (0.87) & 90.98 (1.02) & \textbf{91.75} (0.87) & 89.44 (1.15) & 91.73 (0.88) & 87.57 (1.28) \\
 & 0.50 & 88.3 (1.11) & 87.85 (0.78) & 93.87 (0.71) & 93.39 (0.81) & \textbf{93.89} (0.7) & 92.45 (0.86) & \textbf{93.89} (0.71) & 91.37 (0.92) \\
 & 0.80 & 89.73 (1.43) & 87.88 (1.42) & \textbf{94.98} (0.99) & 94.33 (1.13) & 94.96 (0.98) & 93.3 (1.2) & 94.91 (0.96) & 92.18 (1.21) \\
\cline{1-10}
\multirow[t]{7}{*}{4} & 0.02 & 11.96 (4.03) & \textbf{78.81} (2.98) & 22.57 (9.14) & 15.03 (5.82) & 20.91 (8.57) & 15.46 (5.41) & 18.96 (7.82) & 13.79 (5.43) \\
 & 0.05 & 19.35 (6.62) & \textbf{84.87} (1.63) & 61.81 (7.17) & 37.29 (7.55) & 58.56 (7.5) & 31.56 (9.11) & 52.93 (7.74) & 24.85 (7.85) \\
 & 0.10 & 42.87 (7.4) & \textbf{87.13} (1.12) & 81.57 (3.51) & 60.09 (7.24) & 80.78 (3.71) & 54.68 (7.06) & 78.93 (4.26) & 50.41 (6.71) \\
 & 0.20 & 68.58 (4.38) & 87.61 (0.94) & \textbf{89.12} (1.5) & 85.83 (2.18) & 89.06 (1.5) & 82.84 (2.57) & 88.82 (1.47) & 79.01 (3.19) \\
 & 0.30 & 82.1 (2.02) & 87.58 (0.74) & 91.74 (0.87) & 91.0 (1.02) & \textbf{91.75} (0.87) & 89.45 (1.15) & 91.73 (0.88) & 87.57 (1.29) \\
 & 0.50 & 88.3 (1.11) & 87.85 (0.78) & 93.87 (0.71) & 93.4 (0.81) & \textbf{93.89} (0.7) & 92.45 (0.86) & \textbf{93.89} (0.71) & 91.38 (0.92) \\
 & 0.80 & 89.73 (1.43) & 87.88 (1.42) & \textbf{94.98} (0.99) & 94.34 (1.13) & 94.96 (0.98) & 93.3 (1.2) & 94.91 (0.96) & 92.18 (1.21) \\
\cline{1-10}
\multirow[t]{7}{*}{5} & 0.02 & 11.96 (4.03) & \textbf{78.81} (2.98) & 22.57 (9.14) & 15.03 (5.82) & 20.91 (8.57) & 15.46 (5.41) & 18.96 (7.82) & 13.79 (5.43) \\
 & 0.05 & 19.35 (6.62) & \textbf{84.87} (1.63) & 61.81 (7.17) & 37.29 (7.55) & 58.56 (7.5) & 31.56 (9.11) & 52.93 (7.74) & 24.85 (7.85) \\
 & 0.10 & 42.87 (7.4) & \textbf{87.13} (1.12) & 81.57 (3.51) & 60.09 (7.24) & 80.78 (3.71) & 54.68 (7.06) & 78.93 (4.26) & 50.41 (6.71) \\
 & 0.20 & 68.58 (4.38) & 87.61 (0.94) & \textbf{89.12} (1.5) & 85.83 (2.18) & 89.06 (1.5) & 82.84 (2.57) & 88.82 (1.47) & 79.01 (3.19) \\
 & 0.30 & 82.1 (2.02) & 87.58 (0.74) & 91.74 (0.87) & 91.0 (1.02) & \textbf{91.75} (0.87) & 89.45 (1.15) & 91.73 (0.88) & 87.57 (1.28) \\
 & 0.50 & 88.3 (1.11) & 87.85 (0.78) & 93.87 (0.71) & 93.4 (0.81) & \textbf{93.89} (0.7) & 92.45 (0.86) & \textbf{93.89} (0.71) & 91.38 (0.92) \\
 & 0.80 & 89.73 (1.43) & 87.88 (1.42) & \textbf{94.98} (0.99) & 94.34 (1.13) & 94.96 (0.98) & 93.3 (1.2) & 94.91 (0.96) & 92.18 (1.21) \\
\cline{1-10}
\end{tabular}
}
\end{table*}

\begin{table*}[!h]
\vspace{-3mm}
\caption{Accuracy of various SSL methods on the digits dataset. We pick $\eps^{(\ell)} = 100^{2-\ell}$ for $1 \leq \ell \leq 5$. Proposed methods are in bold.}
\vspace{-3mm}
\label{tab:digitsJ:full}
\vskip 0.15in
\small
\resizebox{0.975\textwidth}{!}{
\begin{tabular}{lllllllllllll}
\toprule
$j$ & rate     &    Laplace &  Poisson   &   WNLL   &    Properly &     $p$-Lap &    RW  &     CK &     \textbf{IP-VQC (2)} &     \textbf{IP-VQC (3)}\\
\midrule
\multirow[t]{7}{*}{1} & 0.02 & 12.2 (4.75) & \textbf{79.0} (2.75) & 67.07 (6.07) & 78.29 (3.14) & 77.83 (3.23) & 30.17 (11.33) & 60.0 (4.17) & 20.58 (8.29) & 19.66 (8.71)  \\
 & 0.05 & 20.42 (7.03) & \textbf{84.61} (1.72) & 69.2 (4.38) & 83.11 (2.08) & 82.5 (2.19) & 32.0 (5.96) & 66.19 (3.73) & 53.07 (7.79) & 50.55 (8.44) \\
 & 0.10 & 41.62 (6.59) & 86.73 (1.36) & 80.73 (3.07) & \textbf{87.67} (1.45) & 87.45 (1.51) & 31.95 (5.56) & 71.98 (2.73) & 78.63 (4.42) & 77.94 (4.46) \\
 & 0.20 & 68.47 (4.79) & 87.61 (0.99) & 86.21 (1.53) & 89.04 (0.97) & 88.93 (1.0) & 40.94 (4.75) & 78.25 (1.53) & \textbf{89.19} (1.11) & 88.97 (1.1) \\
 & 0.30 & 82.17 (2.32) & 87.62 (0.8) & 88.0 (1.2) & 89.81 (0.87) & 89.74 (0.89) & 44.89 (5.34) & 82.11 (0.81) & \textbf{91.75} (0.84) & 91.67 (0.84)  \\
 & 0.50 & 88.18 (1.0) & 87.84 (0.96) & 89.04 (1.0) & 89.98 (1.0) & 89.94 (0.99) & 37.33 (2.51) & 85.67 (0.98) & \textbf{93.8} (0.87) & 93.77 (0.86) \\
 & 0.80 & 89.65 (1.49) & 87.88 (1.4) & 89.68 (1.45) & 89.97 (1.42) & 89.97 (1.41) & 33.93 (1.16) & 88.34 (1.39) & 94.86 (1.0) & \textbf{94.89} (1.02) \\
\cline{1-11}
\multirow[t]{7}{*}{2} & 0.02 & 12.2 (4.75) & \textbf{79.0} (2.75) & 67.07 (6.07) & 78.29 (3.14) & 77.83 (3.23) & 30.17 (11.33) & 60.0 (4.17) & 25.16 (9.35) & 24.25 (9.65)\\
 & 0.05 & 20.42 (7.03) & \textbf{84.61} (1.72) & 69.2 (4.38) & 83.11 (2.08) & 82.5 (2.19) & 32.0 (5.96) & 66.19 (3.73) & 62.69 (6.84) & 61.96 (6.85) \\
 & 0.10 & 41.62 (6.59) & 86.73 (1.36) & 80.73 (3.07) & \textbf{87.67} (1.45) & 87.45 (1.51) & 31.95 (5.56) & 71.98 (2.73) & 81.51 (3.66) & 81.25 (3.61)\\
 & 0.20 & 68.47 (4.79) & 87.61 (0.99) & 86.21 (1.53) & 89.04 (0.97) & 88.93 (1.0) & 40.94 (4.75) & 78.25 (1.53) & \textbf{89.49} (1.09) & 89.41 (1.1) \\
 & 0.30 & 82.17 (2.32) & 87.62 (0.8) & 88.0 (1.2) & 89.81 (0.87) & 89.74 (0.89) & 44.89 (5.34) & 82.11 (0.81) & \textbf{91.83} (0.86) & 91.79 (0.83) \\
 & 0.50 & 88.18 (1.0) & 87.84 (0.96) & 89.04 (1.0) & 89.98 (1.0) & 89.94 (0.99) & 37.33 (2.51) & 85.67 (0.98) & \textbf{93.79} (0.91) & 93.77 (0.9) \\
 & 0.80 & 89.65 (1.49) & 87.88 (1.4) & 89.68 (1.45) & 89.97 (1.42) & 89.97 (1.41) & 33.93 (1.16) & 88.34 (1.39) & 94.91 (1.01) & \textbf{94.93} (1.0) \\
\cline{1-11}
\multirow[t]{7}{*}{3} & 0.02 & 12.2 (4.75) & \textbf{79.0} (2.75) & 67.07 (6.07) & 78.29 (3.14) & 77.83 (3.23) & 30.17 (11.33) & 60.0 (4.17) & 26.92 (9.6) & 26.12 (9.88) \\
 & 0.05 & 20.42 (7.03) & \textbf{84.61} (1.72) & 69.2 (4.38) & 83.11 (2.08) & 82.5 (2.19) & 32.0 (5.96) & 66.19 (3.73) & 64.75 (6.55) & 64.28 (6.54) \\
 & 0.10 & 41.62 (6.59) & 86.73 (1.36) & 80.73 (3.07) & \textbf{87.67} (1.45) & 87.45 (1.51) & 31.95 (5.56) & 71.98 (2.73) & 81.97 (3.53) & 81.79 (3.5)  \\
 & 0.20 & 68.47 (4.79) & 87.61 (0.99) & 86.21 (1.53) & 89.04 (0.97) & 88.93 (1.0) & 40.94 (4.75) & 78.25 (1.53) & \textbf{89.52} (1.07) & 89.44 (1.1) \\
 & 0.30 & 82.17 (2.32) & 87.62 (0.8) & 88.0 (1.2) & 89.81 (0.87) & 89.74 (0.89) & 44.89 (5.34) & 82.11 (0.81) & \textbf{91.8} (0.86) & 91.78 (0.84) \\
 & 0.50 & 88.18 (1.0) & 87.84 (0.96) & 89.04 (1.0) & 89.98 (1.0) & 89.94 (0.99) & 37.33 (2.51) & 85.67 (0.98) & \textbf{93.78} (0.91) & 93.77 (0.92)  \\
 & 0.80 & 89.65 (1.49) & 87.88 (1.4) & 89.68 (1.45) & 89.97 (1.42) & 89.97 (1.41) & 33.93 (1.16) & 88.34 (1.39) & 94.91 (1.0) & \textbf{94.94} (0.98) \\
\cline{1-11}
\multirow[t]{7}{*}{4} & 0.02 & 12.2 (4.75) & \textbf{79.0} (2.75) & 67.07 (6.07) & 78.29 (3.14) & 77.83 (3.23) & 30.17 (11.33) & 60.0 (4.17) & 27.69 (9.72) & 26.92 (9.95)  \\
 & 0.05 & 20.42 (7.03) & \textbf{84.61} (1.72) & 69.2 (4.38) & 83.11 (2.08) & 82.5 (2.19) & 32.0 (5.96) & 66.19 (3.73) & 65.51 (6.41) & 65.15 (6.38) \\
 & 0.10 & 41.62 (6.59) & 86.73 (1.36) & 80.73 (3.07) & \textbf{87.67} (1.45) & 87.45 (1.51) & 31.95 (5.56) & 71.98 (2.73) & 82.13 (3.47) & 81.95 (3.45) \\
 & 0.20 & 68.47 (4.79) & 87.61 (0.99) & 86.21 (1.53) & 89.04 (0.97) & 88.93 (1.0) & 40.94 (4.75) & 78.25 (1.53) & \textbf{89.52} (1.08) & 89.46 (1.1)  \\
 & 0.30 & 82.17 (2.32) & 87.62 (0.8) & 88.0 (1.2) & 89.81 (0.87) & 89.74 (0.89) & 44.89 (5.34) & 82.11 (0.81) & \textbf{91.79} (0.86) & 91.77 (0.83) \\
 & 0.50 & 88.18 (1.0) & 87.84 (0.96) & 89.04 (1.0) & 89.98 (1.0) & 89.94 (0.99) & 37.33 (2.51) & 85.67 (0.98) & \textbf{93.78} (0.92) & \textbf{93.78} (0.92) \\
 & 0.80 & 89.65 (1.49) & 87.88 (1.4) & 89.68 (1.45) & 89.97 (1.42) & 89.97 (1.41) & 33.93 (1.16) & 88.34 (1.39) & 94.92 (1.0) & \textbf{94.94} (1.0) \\
\cline{1-11}
\bottomrule
\end{tabular}
}
\end{table*}

\begin{table*}[!h]
\vspace{-3mm}
\caption{Accuracy of various SSL methods on the Salinas A dataset. We pick $k^{(1)} = 50$, $k^{(2)} = 30$, $k^{(3)} = 20$ and $k^{(4)} = 10$. Proposed methods are in bold.}
\vspace{-3mm}
\label{tab:salinasaQ:full}
\vskip 0.15in
\small
\resizebox{0.975\textwidth}{!}{
\begin{tabular}{llllllllll}
\toprule
$q$ &   rate   &   Laplace  &   Poisson  &  \textbf{IP-QC}   &  \textbf{CP-QC}  &    \textbf{IP-SC} &   \textbf{CP-SC}  &   \textbf{IP-CC} &    \textbf{CP-CC}\\
\midrule
\multirow[t]{7}{*}{2} & 1 & 58.08 (8.37) & 57.12 (7.32) & \textbf{60.63} (7.49) & 58.79 (7.96) & 59.53 (7.92) & 58.65 (8.04) & 58.82 (8.11) & 58.5 (8.14) \\
 & 2 & 66.85 (5.49) & 57.32 (6.44) & \textbf{67.72} (5.37) & 67.26 (5.44) & 67.23 (5.57) & 67.21 (5.47) & 66.95 (5.65) & 67.11 (5.47) \\
 & 5 & 73.46 (2.31) & 56.83 (5.31) & \textbf{73.7} (2.34) & 73.62 (2.32) & 73.66 (2.4) & 73.59 (2.28) & 73.62 (2.45) & 73.55 (2.24) \\
 & 10 & 75.86 (1.82) & 56.08 (5.31) & 76.2 (1.82) & 76.1 (1.82) & \textbf{76.24} (1.82) & 76.05 (1.83) & 76.18 (1.81) & 75.98 (1.82) \\
 & 20 & 77.61 (1.15) & 56.2 (4.25) & \textbf{78.55} (1.38) & 77.95 (1.15) & 78.35 (1.33) & 77.85 (1.15) & 78.2 (1.19) & 77.77 (1.14) \\
 & 50 & 79.6 (0.88) & 56.44 (3.93) & \textbf{80.91} (0.89) & 80.08 (0.91) & 80.72 (0.92) & 79.96 (0.89) & 80.48 (0.93) & 79.85 (0.9) \\
 & 100 & 80.86 (0.57) & 56.06 (2.98) & \textbf{82.35} (0.63) & 81.46 (0.54) & 82.14 (0.61) & 81.32 (0.54) & 81.9 (0.61) & 81.17 (0.55) \\
\cline{1-10}
\multirow[t]{7}{*}{3} & 1 & 58.08 (8.37) & 57.12 (7.32) & \textbf{60.98} (7.28) & 59.25 (7.54) & 59.73 (7.89) & 59.0 (7.85) & 58.81 (8.09) & 58.67 (8.07) \\
 & 2 & 66.85 (5.49) & 57.32 (6.44) & \textbf{67.75} (5.42) & 67.45 (5.44) & 67.26 (5.6) & 67.32 (5.46) & 66.85 (5.77) & 67.22 (5.46) \\
 & 5 & 73.46 (2.31) & 56.83 (5.31) & 73.59 (2.36) & \textbf{73.65} (2.35) & 73.61 (2.42) & 73.63 (2.34) & 73.58 (2.48) & 73.59 (2.27) \\
 & 10 & 75.86 (1.82) & 56.08 (5.31) & 76.09 (1.88) & \textbf{76.21} (1.81) & 76.15 (1.84) & 76.14 (1.83) & 76.15 (1.84) & 76.06 (1.83) \\
 & 20 & 77.61 (1.15) & 56.2 (4.25) & \textbf{78.52} (1.51) & 78.14 (1.18) & 78.42 (1.43) & 78.02 (1.17) & 78.26 (1.31) & 77.87 (1.15) \\
 & 50 & 79.6 (0.88) & 56.44 (3.93) & \textbf{80.95} (0.91) & 80.37 (0.89) & 80.83 (0.94) & 80.18 (0.9) & 80.64 (0.93) & 80.0 (0.9) \\
 & 100 & 80.86 (0.57) & 56.06 (2.98) & \textbf{82.47} (0.7) & 81.82 (0.56) & 82.33 (0.62) & 81.61 (0.56) & 82.1 (0.61) & 81.35 (0.55) \\
\cline{1-10}
\multirow[t]{7}{*}{4} & 1 & 58.08 (8.37) & 57.12 (7.32) & \textbf{60.88} (7.37) & 59.7 (7.13) & 59.7 (7.93) & 59.12 (7.71) & 58.78 (8.12) & 58.7 (8.03) \\
 & 2 & 66.85 (5.49) & 57.32 (6.44) & \textbf{67.7} (5.43) & 67.55 (5.47) & 67.21 (5.65) & 67.35 (5.49) & 66.79 (5.8) & 67.24 (5.49) \\
 & 5 & 73.46 (2.31) & 56.83 (5.31) & 73.58 (2.36) & \textbf{73.65} (2.37) & 73.6 (2.42) & 73.62 (2.36) & 73.57 (2.48) & 73.6 (2.26) \\
 & 10 & 75.86 (1.82) & 56.08 (5.31) & 76.07 (1.89) & \textbf{76.31} (1.8) & 76.15 (1.85) & 76.18 (1.81) & 76.14 (1.85) & 76.1 (1.83) \\
 & 20 & 77.61 (1.15) & 56.2 (4.25) & \textbf{78.52} (1.51) & 78.23 (1.18) & 78.42 (1.44) & 78.09 (1.18) & 78.26 (1.33) & 77.91 (1.15) \\
 & 50 & 79.6 (0.88) & 56.44 (3.93) & \textbf{80.95} (0.93) & 80.53 (0.87) & 80.84 (0.93) & 80.31 (0.91) & 80.65 (0.94) & 80.06 (0.91) \\
 & 100 & 80.86 (0.57) & 56.06 (2.98) & \textbf{82.45} (0.7) & 82.03 (0.59) & 82.34 (0.63) & 81.77 (0.58) & 82.11 (0.61) & 81.45 (0.54) \\
\cline{1-10}
\end{tabular}
}
\end{table*}

\begin{table*}[h]
\vspace{-3mm}
\caption{Accuracy of various SSL methods on the Salinas A dataset. We pick $k^{(1)} = 50$, $k^{(2)} = 30$, $k^{(3)} = 20$ and $k^{(4)} = 10$. Proposed methods are in bold.}
\vspace{-3mm}
\label{tab:salinasaJ:full}
\vskip 0.15in
\small
\resizebox{0.975\textwidth}{!}{
\begin{tabular}{lllllllllllll}
\toprule
$j$ & rate     &    Laplace &  Poisson   &   WNLL   &    Properly &     $p$-Lap &    RW  &     CK &     \textbf{IP-VQC (2)} &     \textbf{IP-VQC (3)}\\
\midrule
\multirow[t]{7}{*}{1} & 1 & 59.28 (8.54) & 58.31 (6.46) & \textbf{64.13} (6.05) & 64.1 (6.04) & 60.26 (5.44) & 63.1 (5.14) & 28.5 (5.98) & 59.93 (8.28) & 60.38 (7.94) \\
 & 2 & 66.82 (5.35) & 56.76 (7.03) & \textbf{67.54} (5.04) & 67.42 (5.1) & 64.65 (5.13) & 66.94 (4.76) & 33.05 (6.65) & 67.04 (5.19) & 67.19 (5.18) \\
 & 5 & 73.74 (2.71) & 55.56 (5.89) & 73.42 (3.07) & 73.14 (3.15) & 72.26 (3.07) & 73.7 (2.6) & 46.37 (5.32) & \textbf{73.84} (2.81) & 73.78 (2.85) \\
 & 10 & 75.88 (1.67) & 56.49 (5.18) & 75.81 (1.73) & 75.32 (1.81) & 74.8 (1.85) & 75.98 (1.73) & 55.54 (4.27) & \textbf{76.15} (1.72) & 76.12 (1.87) \\
 & 20 & 77.44 (1.37) & 55.99 (4.62) & \textbf{78.23} (1.4) & 77.56 (1.58) & 77.51 (1.61) & 77.99 (1.22) & 66.04 (3.1) & 78.1 (1.38) & 78.18 (1.35) \\
 & 50 & 79.58 (0.94) & 56.69 (4.19) & \textbf{80.87} (0.9) & 80.21 (0.93) & 80.36 (0.88) & 79.1 (0.85) & 75.21 (1.76) & 80.45 (0.96) & 80.77 (0.98) \\
 & 100 & 80.96 (0.73) & 55.83 (2.75) & 82.1 (0.63) & 81.88 (0.66) & 82.12 (0.61) & 79.27 (0.7) & 79.82 (1.02) & 81.94 (0.71) & \textbf{82.34} (0.73) \\
\cline{1-11}
\multirow[t]{7}{*}{2} & 1 & 59.28 (8.54) & 58.31 (6.46) & \textbf{64.13} (6.05) & 64.1 (6.04) & 60.26 (5.44) & 63.1 (5.14) & 28.5 (5.98) & 61.88 (7.11) & 62.23 (6.78) \\
 & 2 & 66.82 (5.35) & 56.76 (7.03) & 67.54 (5.04) & 67.42 (5.1) & 64.65 (5.13) & 66.94 (4.76) & 33.05 (6.65) & 67.53 (5.07) & \textbf{67.68} (5.12) \\
 & 5 & 73.74 (2.71) & 55.56 (5.89) & 73.42 (3.07) & 73.14 (3.15) & 72.26 (3.07) & 73.7 (2.6) & 46.37 (5.32) & \textbf{73.94} (2.84) & 73.86 (2.85) \\
 & 10 & 75.88 (1.67) & 56.49 (5.18) & 75.81 (1.73) & 75.32 (1.81) & 74.8 (1.85) & 75.98 (1.73) & 55.54 (4.27) & \textbf{76.23} (1.76) & 76.14 (1.81) \\
 & 20 & 77.44 (1.37) & 55.99 (4.62) & 78.23 (1.4) & 77.56 (1.58) & 77.51 (1.61) & 77.99 (1.22) & 66.04 (3.1) & 78.34 (1.31) & \textbf{78.4} (1.37) \\
 & 50 & 79.58 (0.94) & 56.69 (4.19) & 80.87 (0.9) & 80.21 (0.93) & 80.36 (0.88) & 79.1 (0.85) & 75.21 (1.76) & 80.87 (0.98) & \textbf{80.98} (1.01) \\
 & 100 & 80.96 (0.73) & 55.83 (2.75) & 82.1 (0.63) & 81.88 (0.66) & 82.12 (0.61) & 79.27 (0.7) & 79.82 (1.02) & 82.41 (0.72) & \textbf{82.53} (0.75) \\
\cline{1-11}
\end{tabular}
}
\end{table*}

\begin{table*}[h]
\vspace{-3mm}
\caption{Accuracy of various SSL methods on the MNIST dataset. We pick $k^{(\ell)} = 30 - (\ell-1)\cdot 10$ for $1 \leq \ell \leq 3$. Proposed methods are in bold.}
\vspace{-3mm}
\label{tab:mnistQ:full}
\vskip 0.15in
\small
\resizebox{0.975\textwidth}{!}{
\begin{tabular}{llllllllll}
\toprule
$q$ &   rate   &   Laplace  &   Poisson  &  \textbf{IP-QC}   &  \textbf{CP-QC}  &    \textbf{IP-SC} &   \textbf{CP-SC}  &   \textbf{IP-CC} &    \textbf{CP-CC}\\
\midrule
\multirow[t]{7}{*}{2} & 0.02 & 97.07 (0.07) & 96.8 (0.06) & \textbf{97.33} (0.07) & 97.18 (0.07) & 97.24 (0.07) & 97.16 (0.07) & 97.18 (0.07) & 97.14 (0.07) \\
 & 0.05 & 97.37 (0.05) & 96.85 (0.04) & \textbf{97.62} (0.06) & 97.49 (0.05) & 97.55 (0.05) & 97.46 (0.05) & 97.49 (0.05) & 97.44 (0.05) \\
 & 0.10 & 97.58 (0.04) & 96.85 (0.04) & \textbf{97.8} (0.04) & 97.69 (0.04) & 97.74 (0.04) & 97.67 (0.04) & 97.69 (0.04) & 97.64 (0.04) \\
 & 0.20 & 97.81 (0.04) & 96.87 (0.04) & \textbf{98.0} (0.04) & 97.91 (0.04) & 97.95 (0.04) & 97.89 (0.04) & 97.91 (0.04) & 97.87 (0.04) \\
 & 0.30 & 97.92 (0.04) & 96.87 (0.05) & \textbf{98.1} (0.04) & 98.01 (0.05) & 98.06 (0.04) & 98.0 (0.04) & 98.02 (0.04) & 97.98 (0.04) \\
 & 0.50 & 98.08 (0.06) & 96.87 (0.08) & \textbf{98.24} (0.06) & 98.16 (0.06) & 98.21 (0.06) & 98.15 (0.06) & 98.17 (0.06) & 98.13 (0.06) \\
 & 0.80 & 98.25 (0.09) & 96.9 (0.12) & \textbf{98.39} (0.09) & 98.32 (0.09) & 98.36 (0.09) & 98.31 (0.09) & 98.33 (0.09) & 98.29 (0.09) \\
\cline{1-10}
\multirow[t]{7}{*}{3} & 0.02 & 97.07 (0.07) & 96.8 (0.06) & \textbf{97.36} (0.07) & 97.29 (0.08) & 97.26 (0.07) & 97.24 (0.07) & 97.19 (0.07) & 97.19 (0.07) \\
 & 0.05 & 97.37 (0.05) & 96.85 (0.04) & \textbf{97.64} (0.06) & 97.59 (0.06) & 97.56 (0.05) & 97.54 (0.06) & 97.5 (0.05) & 97.48 (0.06) \\
 & 0.10 & 97.58 (0.04) & 96.85 (0.04) & \textbf{97.82} (0.04) & 97.77 (0.04) & 97.76 (0.04) & 97.74 (0.04) & 97.7 (0.04) & 97.69 (0.04) \\
 & 0.20 & 97.81 (0.04) & 96.87 (0.04) & \textbf{98.01} (0.04) & 97.98 (0.04) & 97.97 (0.04) & 97.95 (0.04) & 97.92 (0.04) & 97.91 (0.04) \\
 & 0.30 & 97.92 (0.04) & 96.87 (0.05) & \textbf{98.1} (0.04) & 98.07 (0.04) & 98.07 (0.04) & 98.05 (0.04) & 98.02 (0.04) & 98.02 (0.05) \\
 & 0.50 & 98.08 (0.06) & 96.87 (0.08) & \textbf{98.24} (0.06) & 98.21 (0.06) & 98.21 (0.06) & 98.19 (0.06) & 98.18 (0.06) & 98.17 (0.06) \\
 & 0.80 & 98.25 (0.09) & 96.9 (0.12) & \textbf{98.38} (0.09) & 98.36 (0.1) & 98.37 (0.09) & 98.34 (0.09) & 98.34 (0.09) & 98.32 (0.09) \\
\cline{1-10}
\end{tabular}
}
\end{table*}

\begin{table*}[h]
\vspace{-3mm}
\caption{Accuracy of various SSL methods on the MNIST dataset. We pick $k^{(\ell)} = 30 - (\ell-1)\cdot 10$ for $1 \leq \ell \leq 3$. Proposed methods are in bold.}
\vspace{-3mm}
\label{tab:MNISTJ:full}
\vskip 0.15in
\small
\resizebox{0.975\textwidth}{!}{
\begin{tabular}{lllllllllllll}
\toprule
$j$ & rate     &    Laplace &  Poisson   &   WNLL   &    Properly &     $p$-Lap &    RW  &     CK &     \textbf{IP-VQC (2)} &     \textbf{IP-VQC} (3)\\
\midrule
\multirow[t]{7}{*}{1} & 0.02 & 97.06 (0.09) & 96.79 (0.07) & 96.55 (0.09) & 94.76 (0.17) & 94.48 (0.17) & 97.15 (0.1) & 95.34 (0.16) & 97.17 (0.09) & \textbf{97.2} (0.09) \\
 & 0.05 & 97.37 (0.06) & 96.85 (0.05) & 97.2 (0.05) & 94.49 (0.12) & 95.49 (0.1) & 97.37 (0.07) & 96.46 (0.08) & 97.49 (0.05) & \textbf{97.52} (0.05) \\
 & 0.10 & 97.59 (0.04) & 96.86 (0.04) & 97.58 (0.05) & 95.59 (0.08) & 96.88 (0.06) & 97.45 (0.05) & 97.18 (0.06) & 97.69 (0.04) & \textbf{97.73} (0.04) \\
 & 0.20 & 97.8 (0.04) & 96.87 (0.04) & 97.86 (0.04) & 97.08 (0.05) & 97.71 (0.04) & 97.5 (0.05) & 97.68 (0.04) & 97.9 (0.04) & \textbf{97.93} (0.04) \\
 & 0.30 & 97.92 (0.05) & 96.87 (0.05) & 97.98 (0.05) & 97.61 (0.06) & 97.88 (0.05) & 97.51 (0.05) & 97.88 (0.05) & 98.02 (0.05) & \textbf{98.04} (0.05) \\
 & 0.50 & 98.08 (0.06) & 96.86 (0.06) & 98.11 (0.06) & 98.01 (0.06) & 98.07 (0.06) & 97.51 (0.06) & 98.09 (0.06) & 98.17 (0.06) & \textbf{98.19} (0.06) \\
 & 0.80 & 98.22 (0.1) & 96.87 (0.14) & 98.23 (0.1) & 98.22 (0.11) & 98.23 (0.1) & 97.52 (0.13) & 98.24 (0.11) & 98.31 (0.11) & \textbf{98.33} (0.11) \\
\cline{1-11}
\multirow[t]{7}{*}{2} & 0.02 & 97.06 (0.09) & 96.79 (0.07) & 96.55 (0.09) & 94.76 (0.17) & 94.48 (0.17) & 97.15 (0.1) & 95.34 (0.16) & 97.31 (0.09) & \textbf{97.34} (0.09) \\
 & 0.05 & 97.37 (0.06) & 96.85 (0.05) & 97.2 (0.05) & 94.49 (0.12) & 95.49 (0.1) & 97.37 (0.07) & 96.46 (0.08) & 97.62 (0.05) & \textbf{97.64} (0.05) \\
 & 0.10 & 97.59 (0.04) & 96.86 (0.04) & 97.58 (0.05) & 95.59 (0.08) & 96.88 (0.06) & 97.45 (0.05) & 97.18 (0.06) & 97.8 (0.04) & \textbf{97.82} (0.04) \\
 & 0.20 & 97.8 (0.04) & 96.87 (0.04) & 97.86 (0.04) & 97.08 (0.05) & 97.71 (0.04) & 97.5 (0.05) & 97.68 (0.04) & 97.99 (0.04) & \textbf{98.0} (0.04) \\
 & 0.30 & 97.92 (0.05) & 96.87 (0.05) & 97.98 (0.05) & 97.61 (0.06) & 97.88 (0.05) & 97.51 (0.05) & 97.88 (0.05) & \textbf{98.1} (0.05) & \textbf{98.1} (0.05) \\
 & 0.50 & 98.08 (0.06) & 96.86 (0.06) & 98.11 (0.06) & 98.01 (0.06) & 98.07 (0.06) & 97.51 (0.06) & 98.09 (0.06) & \textbf{98.24} (0.06) & \textbf{98.24} (0.05) \\
 & 0.80 & 98.22 (0.1) & 96.87 (0.14) & 98.23 (0.1) & 98.22 (0.11) & 98.23 (0.1) & 97.52 (0.13) & 98.24 (0.11) & \textbf{98.37} (0.11) & \textbf{98.37} (0.11) \\
\cline{1-11}
\end{tabular}
}
\end{table*}

\end{document}